\documentclass[twoside]{article}

\usepackage[accepted]{aistats2021}

\usepackage[font=small,labelfont=bf]{caption}

\usepackage{url}
\usepackage[T1]{fontenc}
\usepackage{lmodern}
\usepackage{xr-hyper}
\usepackage[colorlinks]{hyperref}        
\usepackage{url}            
\usepackage{booktabs}       
\usepackage{multirow}
\usepackage{amsfonts}       
\usepackage{nicefrac}       
\usepackage{microtype}      
\usepackage{amssymb}
\usepackage{amsmath}
\usepackage{amsthm}
\usepackage{natbib}
\usepackage{color}
\usepackage{graphicx}
\usepackage{colordvi}
\usepackage{mathtools}
\usepackage{subcaption}
\usepackage{xcolor}
\usepackage{todonotes}

\usepackage{natbib}
\usepackage{bm}

\hypersetup{citecolor=blue}

\newtheorem{definition}{Definition}
\newtheorem{theorem}{Theorem}
\newtheorem{lemma}{Lemma}
\newtheorem{corollary}{Corollary}
\newtheorem{prop}{Proposition}

\newtheorem{manthmin}{Theorem}
\newenvironment{manualthm}[1]{%
  \renewcommand\themanthmin{#1}%
  \manthmin
}{\endmanthmin}

\newtheorem{manpropin}{Proposition}
\newenvironment{manualprop}[1]{%
  \renewcommand\themanpropin{#1}%
  \manpropin
}{\endmanpropin}

\newtheorem{manlemmin}{Lemma}
\newenvironment{manuallemma}[1]{%
  \renewcommand\themanlemmin{#1}%
  \manlemmin
}{\endmanlemmin}

\newtheorem{mancorin}{Corollary}
\newenvironment{manualcor}[1]{%
  \renewcommand\themancorin{#1}%
  \mancorin
}{\endmancorin}

\newtheorem{mandefin}{Definition}
\newenvironment{manualdefinition}[1]{%
  \renewcommand\themandefin{#1}%
  \mandefin
}{\endmandefin}

\DeclareMathOperator{\Tr}{Tr}

\newcommand{\al}{\alpha}
\newcommand{\At}{\mathcal{A}_t}

\newcommand{\be}{\beta}

\newcommand{\dd}{\mathrm{d}}
\newcommand{\E}{\mathbb{E}}
\newcommand{\f}{\varphi}
\newcommand{\Gt}{\mathcal{G}_t}
\newcommand{\Hh}{\mathcal{H}}

\newcommand{\R}{\mathbb{R}}
\newcommand{\sbq}{\sigma_b^2}
\newcommand{\Sd}{\mathbb{S}^{d-1}}
\newcommand{\Span}{\mathrm{Span}}
\newcommand{\swq}{\sigma_w^2}

\makeatletter
\newcommand*{\addFileDependency}[1]{
  \typeout{(#1)}
  \@addtofilelist{#1}
  \IfFileExists{#1}{}{\typeout{No file #1.}}
}
\makeatother

\newcommand*{\myexternaldocument}[1]{%
    \externaldocument{#1}%
    \addFileDependency{#1.tex}%
    \addFileDependency{#1.aux}%
}

\myexternaldocument{supplement}

\makeatletter
\renewcommand*{\@fnsymbol}[1]{\ifcase#1\or*\else\@arabic{#1}\fi}
\makeatother

\begin{document}

\twocolumn[

\aistatstitle{Stable ResNet}

\aistatsauthor{ Soufiane Hayou*$^{1}$ \And Eugenio Clerico*$^{1}$ \And  Bobby He*$^{1}$ \And George Deligiannidis$^{1}$}
\aistatsauthor{}
\aistatsauthor{Arnaud Doucet$^{1}$ \And Judith Rousseau$^{1}$}
\aistatsauthor{} ]

\begin{abstract}
Deep ResNet architectures have achieved state of the art performance on many tasks. While they solve the problem of gradient vanishing, they might suffer from gradient exploding as the depth becomes large. Moreover, recent results have shown that ResNet might lose expressivity as the depth goes to infinity \citep{yang2017meanfield, hayou}. To resolve these issues, we introduce a new class of ResNet architectures, called Stable ResNet, that have the property of stabilizing the gradient while ensuring expressivity in the infinite depth limit. 
\end{abstract}

\section{INTRODUCTION}
\vspace{-0.5em}
The limit of infinite width has been the focus of many theoretical studies on Neural Networks (NNs) \citep{neal,poole, samuel,yang2017meanfield, hayou, lee_wide_nn_ntk}. Although unachievable in practice, it features many interesting properties which can help grasp the complex behaviour of large networks. \\
Infinitely wide 1-layer random NNs behave like Gaussian Processes (GPs) at initialization \citep{neal}. This was recently extended to multilayer NNs, where each layer can be associated to its own GP \citep{matthews, lee_gaussian_process, yang2019scaling}. From a theoretical point of view, GPs have the advantage that their behaviour is fully captured by the mean function and the covariance kernel. Moreover, when dealing with GPs that are equivalent to infinite width NNs, these processes are usually centered, and hence fully determined by their covariance kernel. For multilayer networks, these kernels can be computed recursively, layer by layer \citep{lee_gaussian_process}. Interestingly, in apparent contradiction with the naive idea ``the deeper, the more expressive'', it was shown in \citep{samuel} that the GP becomes trivial as the number of layers goes to infinity, that is the output completely forgets about the input and hence lacks expressive power. This loss of input information during the forward propagation through the network might be exponential in depth and could lead to trainability issues for extremely deep nets \citep{samuel,hayou}.\\
One natural way to prevent this last issue is the introduction of skip connections, commonly known as the ResNet architecture. However, in the regime of large width and depth, the output of standard ResNets becomes inexpressive and the network may suffer from gradient exploding \citep{yang2017meanfield}.\\
In the present work, we propose a new class of residual neural networks, the Stable ResNet, which, in the limit of infinite width and depth, is shown to stabilize the gradient (no gradient vanishing or exploding) and to preserve expressivity in the limit of large depth. The main idea is the introduction of layer/depth dependent scaling factors to the ResNet blocks. \\
For ReLU networks, we provide a comprehensive analysis of two different scalings: a uniform one, where the scaling factor is the same for all the layers, and a decreasing one, where the scaling factor decreases as we go deeper inside the network. 
We also show that Stable ResNet solve the problem of Neural Tangent kernel (NTK) degeneracy in the limit of large depth \citep{hayou_ntk}; indeed, with our scalings, the NTK is universal in the limit of infinite depth, which ensures that any continuous function can be approximated to an arbitrary precision by the features of the infinite depth NTK on a compact set.

All theoretical results are substantiated with numerical experiments in Section \ref{section:experiments}, where we demonstrate the benefits of Stable ResNet scalings both for the corresponding infinite width GP kernels as well as trained ResNets, over a range of moderate and large-scale image classification tasks: MNIST, CIFAR-10, CIFAR-100 and TinyImageNet.

\section{RESNET}
\subsection{Setup and Notations}
Consider a standard ResNet architecture with $L+1$ layers, labelled with $l\in[0:L]$\footnote{Notation: $[m:n]=\{m,m+1\dots n\}$ for integers $n\geq m$.}, of dimensions $\{N_l\}_{l\in[0:L]}$.
\begin{equation}\label{equation:Resnet_dynamics}
\begin{aligned}
y_0(x) &= W_0\,x + B_0\,; \\
y_l(x) &= y_{l-1}(x) + \mathcal{F}((W_l,B_l), y_{l-1}(x)) \quad \mbox{for } l \in [1:L]\,,
\end{aligned}
\end{equation}
where $x \in \mathbb{R}^d$ is an input, $y_l(x)=\{y_l^i(x)\}_{i\in[1:N_l]}$ is the vector of pre-activations, $W_l$ and $B_l$ are respectively the weights and bias of the $l^{th}$ layer, and $\mathcal{F}$ is a mapping that defines the nature of the layer. In general, the mapping $\mathcal{F}$ consists of successive applications of simple linear maps (including convolutional layers), normalization layers \citep{ioffe2015batch} and activation functions. In this work, for the sake of simplicity, we consider Fully Connected blocks with ReLU activation function:
\begin{align*}
  \mathcal{F}((W,B), x) = W\phi(x) + B\,,
\end{align*}
where $\phi$ is the activation function. The weights and bias are initialized with $W_l\overset{\text{iid}}{\sim}\mathcal{N}(0,\sigma_w^2/N_{l-1})$, and $B_l\overset{\text{iid}}{\sim}\mathcal{N}(0,\sigma_b^2)$, where $\sigma_w>0$, $\sigma_b\geq0$, $N_{-1}=d$, and $\mathcal{N}(\mu, \sigma^{2})$ is the normal law of mean $\mu$ and variance $\sigma^{2}$.

Recent results by \citep{hayou_pruning} suggest that scaling the residual blocks with $L^{-1/2}$ might have some beneficial properties on model pruning at initialization. This results from the stabilization effect on the gradient due to the scaling.

More generally, we introduce the residual architecture:
\begin{align}\label{equation:scaled_Resnet}
\begin{split}
&y_0(x) = W_0\,x+B_0\,; \\
&y_l(x) = y_{l-1}(x) + \lambda_{l,L}\, \mathcal{F}((W_l,B_l), y_{l-1})\,, \quad l \in [1:L]\,,
\end{split}
\end{align}
where $\{\lambda_{l,L}\}_{l\in[1:L]}$ is a sequence of scaling factors. We assume hereafter that there exists $\lambda_{\max} \in (0,\infty)$ such that $\lambda_{l,L} \in (0,\lambda_{\max}]$ for all $L\geq 1$ and $l \in [1:L]$.\\
In the next proposition, we give a necessary and sufficient condition for the gradient to remain bounded as the depth $L$ goes to infinity. 

\begin{prop}[Stable Gradient]\label{proposition:stable_gradient}
Consider a ResNet of type \eqref{equation:scaled_Resnet}, and let $\mathcal{L}_y(x):= \ell(y_L^1(x), y)$
for some $(x,y) \in \mathbb{R}^d \times \mathbb{R}$, where $\ell : (z,y) \mapsto \ell(z,y)$ is a loss function satisfying $\sup_{K_1\times K_2}\left|\frac{\partial \ell(z,y)}{\partial z}\right| < \infty$, for all compacts $K_1, K_2 \subset \mathbb{R}$. Then, in the limit  of infinite width, for any compacts $K \subset \mathbb{R}^d$, $K' \subset \mathbb{R}$, there exists a constant $C>0$ such that for all $(x,y)\in K\times K'$
$$\sup\limits_{l \in [0: L]}\mathbb{E}\!\left[\left|\tfrac{\partial \mathcal{L}_y(x)}{\partial W_l^{11}}\right|^2\right] \leq C \exp{\left(\tfrac{\sigma_w^2}{2}\textstyle \sum_{l=1}^L \lambda_{l,L}^2\right)}\,.$$
Moreover, if there exists $\lambda_{\min}>0$ such that for all $L\geq 1$ and $l\in [1:L]$ we have $\lambda_{l,L}\geq\lambda_{\min}$, then, for all $(x,y) \in (\mathbb{R}^d\setminus\{0\}) \times \mathbb{R}$ such that $\left|\frac{\partial \ell(z,y)}{\partial z}\right| \neq 0$, there exists $\kappa>0$ such that for all $l \in [1:L]$
\begin{align*}
\mathbb{E}\left[\left|\tfrac{\partial \mathcal{L}_y(x)}{\partial W_l^{11}}\right|^2\right] \geq \kappa \left(1 + \tfrac{\lambda^2_{\min} \sigma_w^2}{2}\right)^L\,. 
\end{align*}
\end{prop}
Proposition \ref{proposition:stable_gradient} shows that in order to stabilize the gradient, we have to scale the blocks of the ResNet with scalars $\{\lambda_{l,L}\}_{l\in[1:L]}$ such that $\sum_{l=1}^L \lambda_{l,L}^2$ remains bounded as the depth $L$ goes to infinity. Taking $\lambda_{\min}=1$, Proposition \ref{proposition:stable_gradient} shows that the standard ResNet architecture  \eqref{equation:Resnet_dynamics} suffers from gradient exploding at initialization,\footnote{In \citep{yang2017meanfield}, authors show a similar result with a slightly different ResNet architecture.} which may cause instability during the first step of gradient based optimization algorithms such as Stochastic Gradient Descent (SGD). This motivates the following definition of Stable ResNet.
\begin{definition}[Stable ResNet]\label{def:stable_resnet}
A ResNet of type \eqref{equation:scaled_Resnet} is called a Stable ResNet if and only if $\lim\limits_{L \rightarrow \infty}\sum\limits_{l=1}^L \lambda_{l,L}^2 < \infty$.
\end{definition}
The condition on the scaling factors is satisfied by a wide range of sequences $\{\lambda_{l,L}\}_{l\in[1:L], L\geq1}$. However, it is natural to consider the two categories:\\
\textbf{Uniform scaling.} The scaling factors have similar magnitude and tend to zero at the same time. A simple example is the uniform scaling $\lambda_{l,L} = 1/\sqrt{L}$.\\
\textbf{Decreasing scaling.} The sequence is decreasing and tends to zero. To be clearer, we consider a general sequence $\{\lambda_{l}\}_{l\in[1:L]}$
such that $\sum_{l\geq 1} \lambda_{l}^2 < \infty$, and let $\lambda_{l,L} = \lambda_l$ for all $L\geq 1$, all $l\in[1:L]$.

Note that our theoretical analyses will hold for any decreasing scaling $\{\lambda_l\}_{l\geq 1}$ that is square summable, but for simplicity in all empirical results we consider the decreasing scaling: 
$$\lambda_l^{-1} = l^{1/2} \times \text{log}(l+1)\,.$$
We study theoretical properties of both ResNets with uniform and decreasing scaling. We show that, in addition to stabilizing the gradient, both scalings ensure that the ResNet is expressive in the infinite depth limit. For this purpose, we use a tool known as Neural Network Gaussian Process (NNGP) \citep{lee_gaussian_process} which is the equivalent Gaussian Process of a Neural Network in limit of infinite width. 
\subsection{On Gaussian Process approximation of Neural Networks}
Consider a ResNet of type \eqref{equation:scaled_Resnet}. Neurons $\{y_0^i(x)\}_{i\in[1:N_1]}$ are iid since the weights with which they are connected to the inputs are iid. Using the Central Limit Theorem, as $N_{0} \rightarrow \infty$, $y^i_{1}(x)$ is a Gaussian variable for any input $x$ and index $i \in [1:N_1]$. Moreover, the variables $\{y^i_1(x)\}_{i \in [1:N_1]}$ are iid. Therefore, the processes $y^i_{1} (.) $ can be seen as independent (across $i$) centred Gaussian processes with covariance kernel $Q_1$. This is an idealized version of the true process corresponding to letting width $N_{0}\to \infty$. Doing this recursively over $l$ leads to similar approximations for $y_l^i(.)$ where $l \in [1:L]$, and we write accordingly $y_l^i \stackrel{ind}{\sim} \mathcal{GP}(0, Q_{l})$. The approximation of $y_l^i(.)$ by a Gaussian process was first proposed by \citep{neal} in the single layer case and was extended to  multiple feedforward layers by \citep{lee_wide_nn_ntk} and \citep{matthews}. More recently, a powerful framework, known as Tensor Programs, was proposed by \citep{yang2019tensor_i}, confirming the large-width NNGP association for nearly all NN architectures. 

For any input $x \in \mathbb R^d$,  we have  $\mathbb E[y^i_l(x)] = 0$, so that the covariance $Q_l(x,x') = \mathbb{E}[y_l^1(x)y_l^1(x')]$ satisfies for all $x,x'\in \mathbb R^d$ (see Appendix \ref{app:section:ResNets_and_GPs})
\begin{align*}
Q_l(x,x') &= Q_{l-1}(x,x') + \lambda_{l,L}^2 \Psi_{l-1}(x,x')\,,
\end{align*}
where $\Psi_{l-1}(x,x') = \sigma^2_b + \sigma^2_w \mathbb{E}[\phi(y_{l-1}^1(x))\phi(y_{l-1}^1(x'))]$.

For the ReLU activation function $\phi:x\mapsto\max(0,x)$, the recurrence relation can be written more explicitly as in \citep{daniely2016deeper}.
Let $C_l$ be the correlation kernel, defined as 
\begin{align}\label{defc}
C_l(x,x') = \tfrac{Q_l(x,x')}{\sqrt{Q_l(x,x)Q_l(x',x')}}
\end{align}
and let $f:[-1,1]\to\mathbb{R}$ be given by 
\begin{align}\label{deff}
f:\gamma\mapsto \tfrac{1}{\pi}(\sqrt{1-\gamma^2}-\gamma\arccos \gamma)\,.
\end{align}
The recurrence relation reads (see Appendix \ref{app:section:ResNets_and_GPs})
\begin{align}\label{rec}
\begin{split}
&Q_{l} = Q_{l-1} + \lambda_{l,L}^2\left[{\sigma_b}^2 +\tfrac{{\sigma_w}^2}{2}\left(1+\tfrac{f(C_{l-1})}{C_{l-1}}\right)Q_{l-1}\right]\,,\\
&Q_0(x,x') =  \sbq+\swq\,\tfrac{x\cdot x'}{d}\,.
\end{split}
\end{align}
This recursion leads to divergent diagonal terms $Q_L(x,x)$. This was proven in \citep{yang2017meanfield} for a slightly different ResNet architecture. In the next Lemma, we extend this result to the ResNet defined by \eqref{equation:Resnet_dynamics}.
\begin{lemma}[Exploding kernel with standard ResNet]\label{lemma:exploding}
Consider a ResNet of type \eqref{equation:Resnet_dynamics}. Then, for all $x\in \mathbb{R}^d$, 
$$Q_L(x,x) \geq \left(1+\tfrac{\swq}{2}\right)^L \left(\sigma_b^2 \left(1 + \tfrac{2}{\swq}\right) + \tfrac{\swq}{d} \|x\|^2\right).$$
\end{lemma}

\begin{figure}[ht]
  \centering
  \vspace{-1.em}
      \includegraphics[width=0.8\linewidth]{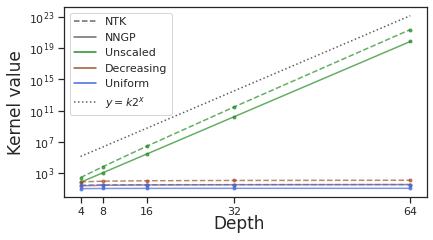}
      \vspace{-1.em}
  \caption{NNGP/NTK for unscaled ResNets explode exponentially (with base 2 if $\swq=2$) in depth, unlike (both uniform and decreasing scaled) Stable ResNets.}\label{fig: explosion}
\end{figure}
Figure \ref{fig: explosion} plots the diagonal NNGP and NTK (introduced in Section \ref{sec:ntk}) values for a point on the sphere, highlighting the exploding kernel problem for standard ResNets. Stable ResNets do not suffer from this problem. 

We now introduce further notation and definitions. Hereafter, unless specified otherwise, $K$ will denote a compact set in $\R^d$ $(d\geq 1)$ and $x,x'$ denote two arbitrary elements of $K$. \\
Let us start with a formal definition of a kernel.\footnote{Our definition is not the standard definition of a kernel, which is more general and does not require the continuity,  \citep{paulsen2016RKHS}.}
\begin{definition}[Kernel]\label{def:kernel}
A kernel $Q$ on $K$ is a symmetric continuous function $K^2\to\mathbb R$ such that, for all $n\in\mathbb N$, for any finite subset $\{x_1\dots x_n\}\subset K$, the matrix $\{Q(x_i,x_j)\}_{i,j}$ is non-negative definite. 
\end{definition}
The symmetry in the above definition has to be understood as $Q(x,x') = Q(x',x)$ for all $x,x'\in K$.

Kernels induce non-negative integral operators \citep{paulsen2016RKHS}.
\begin{lemma}\label{lemma:T(Q)}
Given a continuous and symmetric function $Q:K^2\to\R$, we can define the induced integral operator $T(Q)$ on $L^2(K)$ via its action
$T(Q)\f(x) = \int_K Q(x,y)\f(y)\,\dd y$, for $\f\in L^2(K)$.\footnote{Naturally, we should write $L^2(K,\mu)$, specifying a measure $\mu$ on $K$. In the present work, unless otherwise specified, the notation $L^2(K)$ will imply the choice of any arbitrary finite Borel measure on $K$ (cf Appendix \ref{app:setup_notations}).} Moreover, $T(Q)$ is a bounded, compact, non-negative definite self-adjoint operator. 
\end{lemma}

Each kernel induces a centred Gaussian Process on $K$ \citep{dudley_2002}, that is a random function $F$ on $K$ such that, for any finite $\hat K\subset K$, $\{F(x)\}_{x\in\hat K}$ is a centred Gaussian vector. We recall that the law of a centred GP is fully determined by its covariance function $(x,x')\mapsto \E[F(x)F(x')]$, defined on $K^2$.
\begin{definition}[Induced GP]
Given a kernel $Q$ on $K$, the Gaussian Process induced by $Q$ is a centred GP on $K$ whose covariance function is $Q$.
\end{definition}
We will sometimes use the notation $\mathcal{GP}(0,Q)$ for the law of the GP induced by a kernel $Q$. With our definition of a kernel, the samples from the induced GP lies in $L^2(K)$ with probability $1$ \citep{steinwart2017convergence}. 

From now on we will assume that $0\notin K$ if $\sigma_b=0$.\footnote{We exclude $0$ since for $\sigma_b=0$ $C_0$ is discontinuous in $0$ and can't be a kernel on $K$ as in Definition \ref{def:kernel}, if $0\in K$.} For all ResNets, it is straightforward to check that $Q_L$ is a kernel, in the sense of Definition \ref{def:kernel} (see Appendix \ref{app:section:ResNets_and_GPs} or \citep{daniely2016deeper}). The induced Gaussian Process is what we refer to as NNGP.

We denote by $\mathcal{H}_Q(K)$ the Reproducing Kernel Hilbert Space (RKHS)\footnote{\label{foot_RKHS}See Appendix \ref{app:setup_notations} for a definition.} induced by the kernel $Q$ on the set $K$. The following hierarchical result holds.
\begin{prop}\label{prop:rkhs_hierarchy}
For all $L\geq 1$, $l\in [0,L-1]$, 
$\mathcal{H}_{Q_l}(K) \subseteq \mathcal{H}_{Q_{l+1}}(K)$.
\end{prop}
Proposition \ref{prop:rkhs_hierarchy} shows that, as we go deeper, the RKHS cannot become poorer. However, increasing $L$ might introduce stability issues as illustrated in Proposition~\ref{proposition:stable_gradient}. We show in Sections~\ref{section:uniform_scaling} and \ref{section:decreasing_scaling} that Stable ResNets resolve this problem.

By Lemma \ref{lemma:T(Q)}, $T(Q_L)$ is a bounded, compact, self-adjoint operator and hence can be written as the sum of the projections on its eigenspaces \citep{lang2012real}. By Mercer's Theorem \citep{paulsen2016RKHS}, all the eigenfunctions of $T(Q_L)$ are continuous. 
Finally, it is possible to link the eigen-decomposition of $T(Q_L)$ with the distribution of the GP induced by $Q_L$. Denoting respectively by $\mu_k$ and $\psi_k$ the eigenvalues and eigenfunctions of the operator $T(Q_L)$, we have the equivalence in law:
\begin{align}\label{gproc}
    y^1_L \sim \sum_{k\in\mathbb{N}}\sqrt{\mu_k} Z_k \psi_k\sim \mathcal{GP}(0,Q_L)\,,
\end{align}
where $\{Z_k\}_{k\geq 0}$ are i.i.d.\ standard Gaussian random variables \citep{grenander1950}. The expressivity, that is the capacity to approximate a large class of function, of the network at initialization is then closely linked to the eigendecomposition of $Q_L$ \citep{yang2019finegrained}. 

\subsection{Universal kernels and expressive GPs}
In this section, we provide a comprehensive study of the kernel $Q_L$. We start with a formal definition of universality ($c$-universality in \citep{JMLR:v12:sriperumbudur11a}). Again, unless otherwise stated, let $K$ be a compact in $\R^d$.
\begin{definition}[Universal Kernel]\label{univ_kern}
Let $Q$ be a kernel on $K$, and $\mathcal{H}_Q(K)$ its RKHS \footnote{\label{foot_RKHS}See Appendix \ref{app:setup_notations}.}. We say that $Q$ is universal on $K$ if for any $\varepsilon>0$ and any continuous function $g$ on $K$, there exists $h \in \mathcal{H}_Q(K)$ such that $\|h-g\|_{\infty} <\varepsilon$.
\end{definition}
The universality of a kernel $Q$ on a compact set implies that the kernel is strictly positive definite, i.e. for all non-zero $\f \in L^2(K), \langle T(Q) \f, \f \rangle > 0$ \citep{JMLR:v12:sriperumbudur11a}. Moreover, universality also implies the full expressivity of the induced GP, as expressed in the following.
\begin{definition}[Expressive GP]\label{def:exprGP}
A Gaussian Process on $K$ is said to be expressive on $L^2(K)$ if, denoting by $\psi$ a random realisation  $\psi$ of the process, for all $\f\in L^2(K)$, for all $\varepsilon>0$,
\begin{align*}
     \mathbb{P}(\|\psi-\f\|_2\leq \varepsilon)>0\,.
 \end{align*}
\end{definition}
\begin{lemma}\label{univ=>expr}
A universal kernel $Q$ on $K$ induces an expressive GP on $L^2(K)$.
\end{lemma}
By definition, universal kernels are characterized by the property that their associated RKHS is dense (w.r.t the uniform norm $\|.\|_{\infty}$) in the space of continuous functions on $K$. This is crucial for Kernel regression and Gaussian Process inference \citep{kanagawa2018gaussian}.\footnote{The closure of the set of functions described by the mean function of the posterior of a GP regression is exactly the RKHS of the kernel of the GP prior.} By Proposition \ref{prop:rkhs_hierarchy}, it suffices to prove that $Q_{L_0}$ is universal for some $L_0$ in order to conclude for all $L\geq L_0$. It turns out this is true for $L_0 =2$.
\begin{prop}\label{prop:universality_on_a_compact}
If $\sigma_b > 0$, then $Q_2$ is universal on $K$. From Proposition \ref{prop:rkhs_hierarchy}, $Q_L$ is universal for all $L\geq 2$.
\end{prop}
Note that the presence of biases is essential to achieve universality in the case of a general $K$, since the output of a ReLU ResNet with no bias is always a positive homogeneous function of its input, i.e., a map $F$ such that $F(\al x) =\al F(x)$ for all $\al\geq 0$. However, in the particular case of $K=\Sd$, the unit sphere in $\R^d$, the kernel $Q_L$ is universal (for $L\geq 2$), even when $\sigma_b=0$.
\begin{prop}\label{prop:universality_on_sphere}
Assume $\sigma_b = 0$. Then for all $L \geq 2$, $Q_L$ is universal on $\Sd$ for $d\geq 2$.
\end{prop}

Another interesting fact of the case $K = \Sd$ is that the eigendecomposition of the kernel $Q_L$ has a simple structure. Indeed, on $\Sd$, $Q_L(x,x')$ depends only on the scalar product $x\cdot x'$. These kernels (zonal kernel) admit Spherical Harmonics as an eigenbasis \citep{yang2019finegrained}.
\begin{prop}[Spectral decomposition on $\Sd$]\label{proposition:spectral_decomposition_on_sphere}
Let $Q$ be a zonal kernel on $\Sd$, that is $Q(x,x') = p(x\cdot x')$ for a continuous function $p:[-1,1]\to\mathbb{R}$. Then, there is a sequence $\{\mu_k\geq 0\}_{k\in\mathbb N}$ such that for all $x,x'\in\Sd$
\begin{align*}
Q(x,x') = \sum_{k\geq 0 } \mu_{k} \sum_{j=1}^{N(d,k)} Y_{k,j}(x) Y_{k,j}(x')\,,
\end{align*}
where $\{Y_{k,j}\}_{k\geq0, j\in [1:N(d,k)]}$ are spherical harmonics of  $\mathbb{S}^{d-1}$ and $N(d,k)$ is the number of harmonics of order $k$. With respect to the standard spherical measure, the spherical harmonics form an orthonormal basis of $L^2(\Sd)$ and $T(Q)$ is diagonal on this basis.
\end{prop}
Although the kernel is universal for fixed depth $L$, it is not guaranteed that as $L \rightarrow \infty$, $Q_L$ remains universal. Indeed, for the standard ResNet architecture, the variance $Q_{L}(x,x)$ grows exponentially with $L$ \citep{yang2017meanfield}, and therefore, the kernel diverges. In order to analyse the expressivity of the kernel of a standard ResNet in the limit of large depth, we can study the correlation kernel $C_L$, defined in \eqref{defc}, instead. We show in the following Lemma that, as $L$ goes to infinity, the kernel $C_L$ converges to a constant (which has a 1D RKHS).
\begin{lemma}
\label{lemma:infnite_depth_standard_resnet}
Consider a standard ResNet of type \eqref{equation:Resnet_dynamics} and let $K \subset \mathbb{R}^d \setminus \{0\}$ be a compact set. We have that
$$\lim_{L \rightarrow \infty}\sup_{x,x' \in K}\left| 1 - C_L(x,x')\right| = 0\,.$$
Moreover, if $\sigma_b=0$, then,
$$\sup_{x,x' \in K}\left| 1 - C_L(x,x')\right| = \mathcal{O}(L^{-2})\,.$$
Therefore, $\mathcal{H}_{C_\infty}(K)$ is the space of constant functions.
\end{lemma}
Lemma \ref{lemma:infnite_depth_standard_resnet} shows that in the limit of infinite depth $L$, the RKHS of the correlation kernel is trivial, meaning that the NNGP cannot be expressive. On the contrary, we will show in the next sections that Stable ResNets achieve a universal kernel for infinite depth $L$.
\section{UNIFORM SCALING}\label{section:uniform_scaling}
Consider a Stable ResNet with layers $[0:L]$. Under uniform scaling, the recurrence relation in \eqref{rec} reads:
\begin{align}\label{recunif}
Q_{l} &= Q_{l-1} + \tfrac{1}{L}\left({\sigma_b}^2 +\tfrac{{\sigma_w}^2}{2}\left(1+\tfrac{f(C_{l-1})}{C_{l-1}}\right)Q_{l-1}\right).
\end{align}
In the limit as $L\to\infty$, \eqref{recunif} converges uniformly to a continuous ODE. Studying the solution of the corresponding Cauchy problem, we show that the covariance kernel remains universal in the limit of infinite depth.
\subsection{Continuous formulation}
The layer index $l$ in \eqref{recunif} can be rescaled as $l\mapsto t(l)= l/L$. Clearly $t(0) = 0$ and $t(L)=1$, so the image of $t$ is contained in $[0,1]$. In the limit $L\to\infty$ it is natural to consider $t$ as a continuous variable spanning the interval $[0,1]$. With this in mind, it makes sense to look at the continuous version of \eqref{recunif}.\\
Let $K\subset\R^d$ be a compact set and $x,x'\in K$. If $\sigma_b=0$ assume that $0\notin K$.
\begin{align}
\label{recc}
\begin{split}
&\dot q_t(x,x') = {\sigma_b}^2 + \tfrac{{\sigma_w}^2}{2}\left(1+\tfrac{f(c_t(x,x'))}{c_t(x,x')}\right)q_t(x,x')\,,\\
&q_0(x,x') = \sbq + \swq\,\tfrac{x\cdot x'}{d}\,,\\
&c_t(x,x')=\tfrac{q_t(x,x')}{\sqrt{q_t(x,x)q_t(x',x')}}\,.
\end{split}
\end{align}
As discussed in Section \ref{app:section:uniform_scaling} of the Appendix, for any $x,x'$, the solution of the above Cauchy problem exists and is unique. Moreover, the solutions $q_t$ and $c_t$ are kernels on $K$, in the sense of Definition \ref{def:kernel}.

Clearly, for finite $L$, the continuous ODE \eqref{recc} is an approximation. However, the following result holds.

\begin{lemma}[Convergence to the continuous limit]\label{convcont}
Let $Q_{l|L}$ be the covariance kernel of the layer $l$ in a net of $L+1$ layers $[0:L]$, and $q_t$ be the solution of \eqref{recc}, then
\begin{align*}
    \lim_{L\to\infty}\sup_{l\in[0:L]}\sup_{(x,x')\in K^2}|Q_{l|L}(x,x')-q_{t=l/L}(x,x')| = 0\,.
\end{align*}
\end{lemma}
\subsection{Universality of the covariance kernel}
When $\sigma_b>0$, the kernel $q_t$ is universal for $t>0$.
\begin{theorem}[Universality of $q_t$]\label{thm:expr}
Let $K \subset \R^d$ be compact and assume $\sigma_b>0$. For any $t\in (0,1]$, the solution  $q_t$ of \eqref{recc} is a universal kernel on $K$.
\end{theorem}
The proof of the above statement is detailed in Appendix \ref{app:section:uniform_scaling}. The main idea is to show that the integral operator $T(q_t)$ is strictly positive definite and then use a characterization of universal kernels, due to \citep{JMLR:v12:sriperumbudur11a}, which connects the universality of Definition \ref{univ_kern} with the strict positivity of the induced integral operator.\footnote{The details are more involved as we need to show that the kernel induces a strictly positive definite operator on $L^2(K,\mu)$ for any finite Borel measure $\mu$ on $K$.}

As mentioned previously, the presence of the bias is essential to achieve full expressivity on a generic compact $K\subset \R^d$. However, we can still have universality when no bias is present, limiting ourselves to the case of the unit sphere $K=\Sd$.
\begin{prop}[Universality on $\Sd$]\label{prop:uniform_Sd}
For any $t\in (0,1]$, the covariance kernel $q_t$, solution of \eqref{recc} with $\sigma_b=0$, is universal on $\Sd$, with $d\geq 2$.
\end{prop}
\section{DECREASING SCALING}\label{section:decreasing_scaling}
Consider a Stable ResNet with decreasing scaling, that is a sequence of scaling factors $(\lambda_k)_{k\geq1}$ such that $\sum_{k\geq 1} \lambda_k^2 < \infty$. In this setting, each additional layer can be seen as a correction to the network output with decreasing magnitude. As for the uniform scaling, we show in the next proposition that the kernel $Q_L$ converges to a limiting kernel $Q_{\infty}$, and the convergence is uniform over any compact set of $\mathbb{R}^d$. The notation $g(x) = \Theta(m(x))$ means there exist two constants $A,B > 0$ such that $ A m(x) \leq g(x) \leq B m(x)$.
\begin{prop}[Uniform Convergence of the Kernel]\label{prop:uniform_convergence_decreasing_scaling}
Consider a Stable ResNet with a decreasing scaling, i.e. the sequence $\{\lambda_l\}_{l\geq1}$ is such that $\sum_l \lambda_l^2 < \infty$. Then for all $(\sigma_b,\sigma_w)\in \mathbb{R}^+ \times (\mathbb{R}^+)^*$, there exists a kernel $Q_{\infty}$ on $\mathbb{R}^d$ such that for any compact set $K \subset \mathbb{R}^d$, 
\begin{align*}
    \sup_{x,x' \in K} |Q_{L}(x,x') - Q_{\infty}(x,x')| = \Theta\big(\textstyle\sum_{k \geq L} \lambda_k^2\big)\,.
\end{align*}
\end{prop}
\vspace{-1.em}
The convergence of the kernel $Q_L$ to the limiting kernel $Q_\infty$ is governed by the convergence rate of the series of scaling factors. Moreover, leveraging the RKHS hierarchy from Proposition \ref{prop:rkhs_hierarchy}, we find that $Q_\infty$ is universal.
\begin{corollary}[Universality of $Q_\infty$]\label{cor:universality_decreasing_scaling}
The following statements hold\\
$\bullet$~Let $K$ be a compact set of $\mathbb{R}^d$ and assume $\sigma_b>0$. Then, $Q_\infty$ is universal on $K$.\\
$\bullet$~Assume $\sigma_b=0$. Then $Q_\infty$ is universal on $\Sd$.
\end{corollary}
As in the uniform scaling case, the limiting kernel exists and is universal unlike the standard ResNet architecture that yields a divergent kernel $Q_L$ as $L\to\infty$. 

To validate our universality and expressivity results, Figure \ref{fig: toy_expressivity} plots the leading eigenvalues of the NNGP (\& NTK, introduced in Section \ref{sec:ntk}) kernels on a set of 1000 points sampled uniformly at random from the circle, normalized so that the largest eigenvalue is 1. We use the recursion formulas for NNGP correlation (Lemma \ref{lemma:corr_formula}) and normalized NTK (Lemma \ref{lemma:stable_ntk_recursion}) to avoid the exploding variance/gradient problem. We see that the unscaled ResNet NNGP becomes inexpressive with depth because all non-leading eigenvalues converge to 0, whereas our Stable ResNets (decreasing and uniform scaling) are expressive even in the large depth limit. 

\begin{figure}[h]
  \centering
      \includegraphics[width=0.9\linewidth]{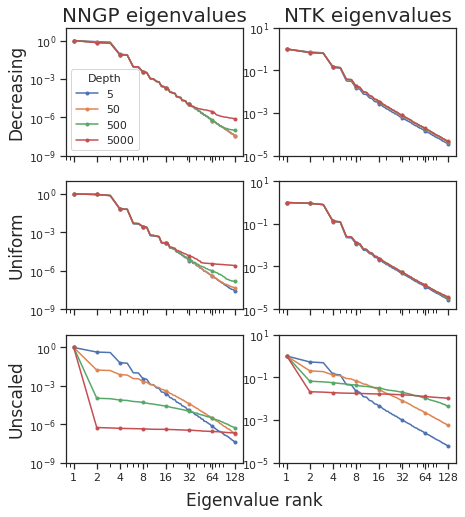}
      \vspace{-1em}
  \caption{(Normalized) NNGP \& NTK matrix  eigenvalues of Stable (decreasing \& uniform) \&
 unscaled (i.e. standard) ResNets.}\label{fig: toy_expressivity}
 \vspace{-1.5em}
\end{figure}

\section{NEURAL TANGENT KERNEL}\label{sec:ntk}
In the so-called lazy training regime \citep{chizat}, the training dynamics of an infinitely wide network can be described via the Neural Tangent Kernel (NTK) \citep{lee_wide_nn_ntk}, introduced in \citep{jacot} and defined as $$\tilde\Theta^{ij}_L(x,x') = \nabla_{\text{par}}\, y^i_L(x)\cdot\nabla_{\text{par}}\, y^j_L(x')\,,$$
with $\nabla_{\text{par}}$ the gradient wrt the parameters of the NN.\footnote{All network considered in this section are assumed to have NTK parametrization, cf Appendix \ref{app:section:NTK} for details.}\\
To simplify our presentation we will assume that the output dimension of the network is 1.\footnote{This does not affect our final conclusion of universality for the NTK, which is diagonal in the output space, that is $\tilde\Theta^{ij}=\Theta\delta^{ij}$, \citep{jacot, hayou_ntk}.}\\
Let $F_\tau$ be the output function of the ResNet at training time $\tau$. In the NTK regime (infinite width), the gradient flow is equivalent to a simple linear model \citep{lee_wide_nn_ntk}, that gives 
$$
F_\tau(x) - F_0(x) = \Theta_L(x, \mathcal{X}) \hat{\Theta}_L^{-1}(I - e^{-\eta \hat{\Theta}_L \tau}) (\mathcal{Y}- F_0(\mathcal{X}))\,,
$$
where $\mathcal{X}$ and $\mathcal{Y}$ are respectively the input and output datasets, $\Theta_L(x, \mathcal{X}) = \{\Theta_L(x, x')\}_{x' \in \mathcal{X}}$ and $\hat\Theta_L$ is the matrix $\{\Theta_l(x,x')\}_{x,x'\in\mathcal X}$. The universality of the NTK is crucial for the ResNet to learn beyond initialization, since the residual $F_\tau - F_0$ lies in the RKHS generated by $\Theta_L$. For unscaled ResNet, \citep{hayou_ntk} showed that the limiting NTK is trivial in the sense of Lemma \ref{lemma:infnite_depth_standard_resnet}. However, this is not the case for Stable ResNet.

Consider a ResNet of type \eqref{equation:scaled_Resnet}. We have \footnote{This is true under the technical assumption that the parameters appearing in the back-propagation can be considered independent from the ones of the forward pass (Gradient Independent Assumption) \citep{yang2019scaling}}  
\begin{align}\label{recurrence_NTK}
\Theta_0=Q_0,\hspace{0.15cm}\Theta_{l+1} = \Theta_l + \lambda_{l,L}^2\left(\Psi_l + \Psi'_l\,\Theta_l\right)\,,
\end{align}
where $\Psi_l(x,x') = {\sbq} +  {\swq} \E[\phi(y^l_1(x))\phi(y^l_1(x'))]$ and $\Psi'_l(x,x') =\swq \E[\phi'(y^l_1(x))\phi'(y^l_1(x'))]$ (see Appendix \ref{app:section:NTK}).
\begin{prop}\label{prop:decreasing_ntk_universal}
Fix a compact $K\subset\R^d$ ($0\notin K$ if $\sigma_b=0$) and consider a Stable ResNet with decreasing scaling. Then $\Theta_L$ converges uniformly over $K^2$ to a kernel $\Theta_\infty$. Moreover $\Theta_\infty$ is universal on $K$ if $\sigma_b>0$. If $K=\Sd$, then the universality holds for $\sigma_b=0$.
\end{prop}
An analogous result can be stated for the uniform scaling, after noticing that a continuous formulation ($\Theta_l\mapsto \theta_{t(l)}$) can be obtained in analogy with what has been done for the covariance kernel (cf Appendix \ref{app:section:NTK}). 
\begin{prop}\label{prop:uniform_ntk_universal}
Let $K\subset\R^d$ and fix $t\in(0,1]$. If $\sigma_b>0$, then $\theta_t$ is universal on $K$. The same holds true if $\sigma_b=0$ and $K=\Sd$.
\end{prop}
Figure \ref{fig: toy_expressivity} shows that the non-leading NTK eigenvalues do not decay to 0 with depth for Stable ResNets, unlike for unscaled ResNets. This is in line with findings of Propositions \ref{prop:decreasing_ntk_universal} and \ref{prop:uniform_ntk_universal}.


\section{A PAC-BAYES RESULT}
Consider a dataset $S$ with $N$ iid training examples $(x_i,y_i)_{1\leq i \leq N} \in X\times Y$, and a hypothesis space $\mathcal{P}$ from which we want to learn an optimal hypothesis according to some bounded loss function $\ell : Y\times Y \mapsto [0,1]$. The empirical/generalization loss of a hypothesis $h \in \mathcal{U}$ are
$$
r_S(h) = \tfrac{1}{N} \textstyle\sum_{i=1}^N \ell(h(x_i), y_i), \quad r(h) = \mathbb{E}_{\nu}[\ell(h(x),y)]\,,
$$
where $\nu$ is a probability distribution on $X \times Y$. For some randomized learning algorithm $\mathcal{A}$, the empirical and generalization loss are given by:
$$
r_S(\mathcal{A}) = \mathbb{E}_{h \sim \mathcal{A}}[r_S(h)], \qquad r(\mathcal{A}) = \mathbb{E}_{h \sim \mathcal{A}}[r(h)]\,.
$$
The PAC-Bayes theorem gives a probabilistic upper bound on the generalization loss $r(\mathcal{A})$ of a randomized learning algorithm $\mathcal{A}$ in terms of the empirical loss $r_S(\mathcal{A})$. Fix a prior distribution $\mathcal{P}$ on the hypothesis set $\mathcal{U}$. The Kullback-Leibler divergence between $\mathcal{A}$ and $\mathcal{P}$ is defined as $\text{KL}(\mathcal{A} \| \mathcal{P}) = \int \mathcal{A}(h) \log \frac{\mathcal{A}(h)}{\mathcal{P}(h)} \textrm{d}h \in [0, \infty]$. The Bernoulli KL-divergence is given by $\text{kl}(a||p) = a \log \frac{a}{p} + (1-a) \log \frac{1-a}{1-p} $  for $a,p \in [0,1]$. We define the inverse Bernoulli KL-divergence $\text{kl}^{-1}$ by 
$$
\text{kl}^{-1}(a,\varepsilon) = \sup\{ p \in [0,1] : \text{kl}(a||p) \leq \varepsilon \}\,.
$$
\begin{theorem}[PAC-Bayes bound Theorem \citep{seeger_pacbayes}]\label{thm:pac-bayes_theorem}
For any loss function $\ell$ that is $[0,1]$ valued, any distribution $\nu$,  any $N \in \mathbb{N}$, any prior $\mathcal{P}$, and any $\delta \in (0,1]$, with probability at least $1-\delta$ over the sample $S$, we have 
\begin{align*}
\forall \mathcal{A}, \hspace{0.25cm} r(\mathcal{A}) \leq  \textup{kl}^{-1}\left(r_S(\mathcal{A}), \tfrac{\textup{KL}(\mathcal{A}\|\mathcal{P}) + \log (2 \sqrt{N}/\delta)}{N}\right)\,.
\end{align*}
\end{theorem}
\vspace{-1.em}
The KL-divergence term $\text{KL}(\mathcal{A}\|\mathcal{P})$ plays a major role as it controls the generalization gap, i.e. the difference (in terms of Bernoulli KL-divergence) between the empirical loss and the generalization loss. In our setting, we consider an ordinary GP regression with prior $\mathcal{P}(f) = \mathcal{GP}(f | 0, Q(x,x'))$. Under the standard assumption that the outputs $y_N = (y_i)_{i \in [1:N]}$ are noisy versions of $f_N = (f(x_i))_{ i \in [1:N]}$ with $y_N | f_N \sim \mathcal{N}(y_N | f_N, \sigma^2 I)$, the Bayesian posterior $\mathcal{A}$ is also a GP and is given by 
\begin{equation}\label{eqn:posterior}
\begin{split}
\mathcal{A}(f) = \mathcal{GP}&(f | Q_N(x)(Q_{NN} + \sigma^2 I)^{-1} y_N, Q(x,x')\\
& - Q_N(x)(Q_{NN} + \sigma^2 I)^{-1})Q_N(x')^T)\,.
\end{split}
\end{equation}
$Q_N(x) = (Q(x,x_i))_{i \in [1:N]}$,  $Q_{NN} = (Q(x_i, x_j))_{1\leq i,j \leq N}$. In this setting, we have the following result
\begin{prop}[Curse of Depth]\label{prop:pac-bayes_bound}
Let $Q_L$ be the kernel of a ResNet. Let $P_L$ be a GP with kernel $Q_L$ and $\mathcal{A}_L$ be the corresponding Bayesian posterior for some fixed noise level $\sigma^2>0$. Then, in a fixed setting (fixed sample size N), the following results hold:\\
$\bullet$~With a standard ResNet, $\textup{KL}(\mathcal{A}_L \| P_L) \gtrsim L$.\\
$\bullet$~With a Stable ResNet, $\textup{KL}(\mathcal{A}_L \| P_L) = \mathcal{O}_L(1)$.
\end{prop}
The KL-divergence bound diverges for a standard ResNet while it remains bounded for Stable ResNet. Although PAC-Bayes bounds only give an upper bound on the generalization error, Proposition \ref{prop:pac-bayes_bound} shows that Stable ResNet does not suffer from the ``curse of depth'', i.e. the KL-divergence does not explode as the depth becomes large.

\section{EXPERIMENTS}\label{section:experiments}
In line with our theory, we now present results demonstrating empirical advantages of Stable ResNets (both uniform and decreasing scaling) compared to their unscaled counterparts on a toy regression task and standard image classification tasks, both for infinite-width NNGP kernels as well as trained finite-width NNs in the latter case. In the interests of space, all experimental details not described in this section can be found in Appendix \ref{app: experimental details}.
All error bars in this section correspond to 3 independent runs.
\paragraph{Stable NNGP regression experiment} We first present a toy regression posterior regression experiment with NNGP kernel. We compare across different depths and scalings, with target test function $y=x\text{sin}(x)$ and a small amount of observation noise $\sigma=0.1$ ($\sigma$ as defined in Eq. \ref{eqn:posterior}).
\begin{figure*}\centering
      \includegraphics[width=0.8\linewidth]{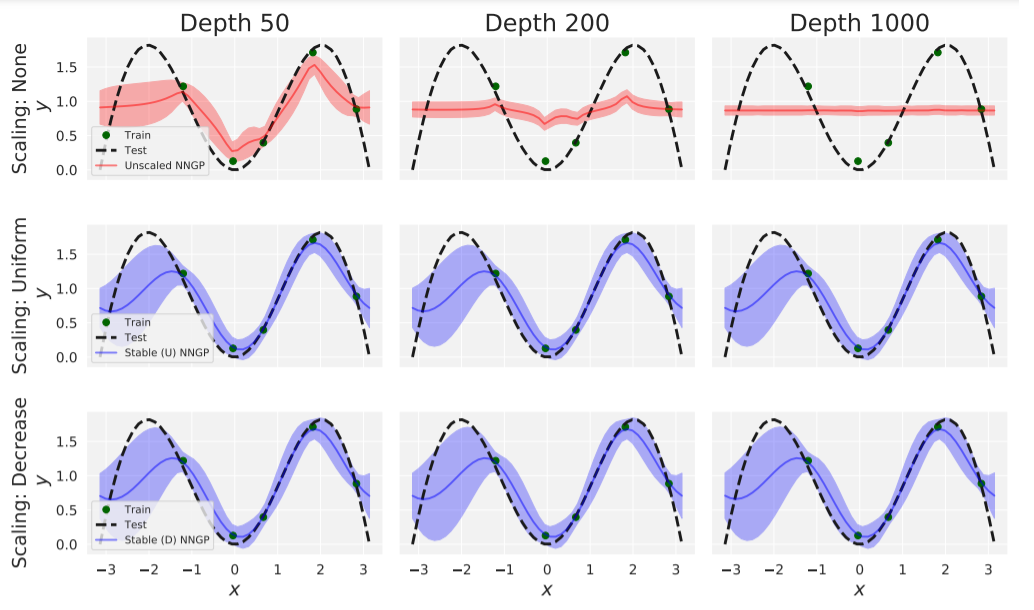}
      \vspace{-1.em}
  \caption{NNGP toy regression experiment.}\label{fig: reg}
\end{figure*}
 We use 5 training points (dark green dots).

We map our 1D inputs $x$ onto the circle $(\text{cos}(x), \text{sin}(x))$ before performing GP regression. This is so that all inputs have unit norm and we can use the NNGP correlation kernel (Eq. \ref{defc}) for the vanilla ResNet (ResNet with fully connected blocks), in order to avoid the exploding variance problem.\\ As expected from our theory, in Figure \ref{fig: reg}, for depth 1000 the NNGP correlation kernel without stable scaling (top row, red) is unable to learn anything beyond a constant function due to inexpressivity, whereas our Stable ResNets (bottom two rows, blue) are still expressive in the large depth limit. We plot mean and 95\% posterior predictive credible interval for NNGP posteriors.

\paragraph{Stable NNGP classification results}

 We first compare the performance of Stable and standard ResNets of varying depths through their infinite-width NNGP kernels, on MNIST \& CIFAR-10. For each considered NNGP kernel $Q$ and training set $(x_i, y_i)_{i\in[1:N]}$, we report test accuracy using the mean of the posterior predictive (Eq. \ref{eqn:posterior}):      $Q_N(\cdot)(Q_{NN}+\sigma^2 I)^{-1}y_N$, which is also the kernel ridge regression predictor \citep{kanagawa2018gaussian}. We treat classification labels $y$ as one-hot regression targets, similar to recent works \citep{arora2019exact, lee_wide_nn_ntk, shankar2020neural}, and tune the noise $\sigma^2$ using prediction accuracy on a held-out validation set.
 
\begin{table}[ht]
\centering
\caption{CIFAR-10 test accuracies (\%) using posterior predictive mean of NNGP kernels for deep Wide-ResNets \citep{zagoruyko2016wide} with different training set sizes $N$. Scaled (D) \& Scaled (U) refer to decreasing and uniform scaling respectively.}\label{table: wrn_krr}
\setlength\tabcolsep{4.5pt}
\begin{tabular}{llccc}
\toprule
$N$ & Depth &   Scaled (D) & Scaled (U) &  Unscaled \\
\midrule

1K  & 112 &  ${36.84}_{\pm 0.53}$ &  ${36.43}_{\pm 0.49}$ &  ${37.71}_{\pm 0.50}$ \\
      & 202 &  ${36.89}_{\pm 0.55}$ &  ${36.47}_{\pm 0.49}$ &  \textemdash \\
10K & 112 &  ${53.81}_{\pm 0.11}$ &  ${53.55}_{\pm 0.41}$ &  ${53.34}_{\pm 0.07}$ \\
      & 202 &  ${53.80}_{\pm 0.10}$ &  ${53.57}_{\pm 0.40}$ &  \textemdash \\
\bottomrule

\end{tabular}

\end{table}
\begin{table*}[h]
\centering
\caption{MNIST and CIFAR-10 test accuracies (\%) using posterior predictive mean of NNGP kernels for deep vanilla ResNets (ResNet with fully connected blocks) with different size training sets $N$.}\label{tab:vrn_krr}
\begin{tabular}{llllllll}
\toprule
      &      Dataset    & \multicolumn{3}{l}{MNIST} & \multicolumn{3}{l}{CIFAR-10} \\
      \cmidrule(lr){2-2} \cmidrule(lr){3-5} \cmidrule(lr){6-8}
      $N$ &    Depth      &            Scaled (D) &            Scaled (U) &              Unscaled &            Scaled (D) &            Scaled (U) &              Unscaled \\

\midrule
1K  & 50 &  ${92.88}_{\pm 0.35}$ &  ${92.39}_{\pm 0.33}$ &  ${92.44}_{\pm 0.21}$ &  ${35.83}_{\pm 0.14}$ &  ${34.73}_{\pm 0.14}$ &  $\bm{37.16}_{\pm 0.25}$ \\
      & 200 &  ${92.91}_{\pm 0.35}$ &  ${92.39}_{\pm 0.32}$ &  ${89.56}_{\pm 0.56}$ &  $\bm{35.86}_{\pm 0.14}$ &  ${34.76}_{\pm 0.11}$ &  ${34.85}_{\pm 0.17}$ \\
      & 1000 &  ${92.92}_{\pm 0.34}$ &  ${92.39}_{\pm 0.32}$ &  ${55.13}_{\pm 5.31}$ &  $\bm{35.89}_{\pm 0.14}$ &  ${34.76}_{\pm 0.11}$ &  ${12.43}_{\pm 3.97}$ \\
      \midrule
10K & 50 &  ${97.57}_{\pm 0.12}$ &  ${97.55}_{\pm 0.12}$ &  ${97.06}_{\pm 0.10}$ &  ${48.71}_{\pm 0.31}$ &  ${48.12}_{\pm 0.27}$ &  $\bm{50.11}_{\pm 0.37}$ \\
      & 200 &  ${97.57}_{\pm 0.11}$ &  ${97.55}_{\pm 0.12}$ &  ${95.55}_{\pm 0.13}$ &  $\bm{48.77}_{\pm 0.30}$ &  ${47.15}_{\pm 0.18}$ &  ${47.00}_{\pm 0.30}$ \\
      & 1000 &  ${97.57}_{\pm 0.10}$ &  ${97.54}_{\pm 0.12}$ &  ${67.53}_{\pm 2.96}$ &  $\bm{48.76}_{\pm 0.30}$ &  ${47.16}_{\pm 0.17}$ &  ${17.86}_{\pm 2.32}$ \\
\bottomrule
\end{tabular}

\end{table*}
 First, in Table \ref{table: wrn_krr}, we demonstrate the exploding NNGP variance problem for unscaled Wide-ResNets (WRN) \citep{zagoruyko2016wide}. For an unscaled WRN of depth 202, the NNGP kernel values explode resulting in numerical errors, whereas Stable ResNets achieve 54\% test accuracy with 10K training points (out of full size 50K). Note that any numerical errors from exploding NNGP  also afflict the NTK, as the difference between the NTK and NNGP is positive semi-definite \citep{lee_wide_nn_ntk, he2020bayesian} (which is why the NTK lines always lie above their corresponding NNGP in Figure \ref{fig: explosion}).
 
 To isolate the disadvantages of inexpressivity in unscaled Resnets NNGPs compared to our Stable ResNets, we need to avoid the exploding variance problem and ensuing numerical errors. In order to do so, we use the NNGP correlation kernel $C$ instead of the NNGP covariance kernel $Q$, noting that these two kernels are equal up to multiplicative constant on the sphere, and that the posterior predictive mean is invariant to the scale of $Q$ (with $\sigma^2$ also tuned relative to the scale of $Q$). Moreover, the formula in Lemma \ref{lemma:corr_formula} for NNGP correlation recursion for vanilla ResNets without bias  can be recast as a ResNet with a modified scaling (see Appendix \ref{app:corr_as_nngp}), allowing us to use existing optimised libraries \citep{neuraltangents2020}. In order to use the vanilla ResNet correlation recursion, we standardise all MNIST \& CIFAR-10 images to lie on the 784 \& 3072-dimension sphere respectively.
 
Our expressivity results, as well as Proposition \ref{prop:pac-bayes_bound}, suggest that we expect Stable ResNets to outperform standard ResNets for large depths even when exploding variance numerical errors are alleviated for standard ResNets. In Table \ref{tab:vrn_krr}, we see that unscaled ResNets suffer from a degradation in test accuracy with depth, due to inexpressivity, whereas our Stable ResNets (both decreasing and uniform) do not suffer from a drop in performance. For example, the posterior predictive mean using the NNGP of an unscaled vanilla ResNet with depth 1000 attains only 17.86\% accuracy on CIFAR-10 with 10K training points, compared to 48.76\% for Stable ResNet (decreasing scale).   
 
We focus on the NNGP rather than the NTK as recent works \citep{lee2020finite, shankar2020neural} have demonstrated that there is no advantage to the state-of-the-art NTK over the NNGP as infinite-width kernel predictors. Moreover, we do not aim for near state-of-the-art kernel results due to computational resources, and instead aim to empirically validate the theoretical advantages of Stable ResNets.
\paragraph{Trained Stable ResNet results}
Finally, we consider the benefits of trained Stable ResNets on the large-scale CIFAR-10, CIFAR-100 and TinyImageNet\footnote{Available at \url{http://cs231n.stanford.edu/tiny-imagenet-200.zip}} datasets. We compare trained convolutional ResNets \citep{he} of depths 32, 50 \& 104 in terms of test accuracy. In the main text we present results for ResNets trained with Batch Normalization \citep{ioffe2015batch} (BatchNorm), while results for trained ResNets without BatchNorm can be found in Appendix \ref{app: experimental details}. Stable ResNet scalings are applied to the residual connection after all convolution, ReLU and BatchNorm layers.

We use initial learning rate $0.1$ which is decayed by $0.1$ at $50\%$ and $75\%$ of the way through training. This learning rate schedule has been used previously \citep{he} for unscaled ResNets and we found it to work well for all ResNets trained with BatchNorm. We train for 160 epochs on CIFAR-10/100 and 250 epochs on TinyImageNet. Test accuracy results are displayed in Table \ref{table: NN results}. As we can see, Stable ResNets consistently outperform standard ResNets across datasets and depths. Moreover, the performance gap is larger for larger depths: for example on CIFAR-100 our Stable ResNet (decreasing) outperforms its standard counterpart by 1.05\% (75.06 vs 74.01) on average for depth 32 whereas for depth 104 the test accuracy gap is 2.36\% (77.44 vs 75.08) on average. A similar trend can also be observed for the more challenging TinyImageNet dataset. Interestingly, we see that among the Stable ResNets, decreasing scaling also consistently outperforms uniform scaling.

\begin{table}[ht]
\caption{Test accurracies (\%) of trained deep ResNets of various scalings and depths on CIFAR-10 (C-10), CIFAR-100 (C-100) \& TinyImageNet (Tiny-I).}\label{table: NN results}
\setlength\tabcolsep{3.5pt}
\begin{tabular}{lllll}
\toprule
Dataset & Depth &               Scaled (D) &              Scaled (U) &              Unscaled \\
\midrule
C-10 & 32  &     ${94.84}_{\pm 0.08}$ &     ${94.78}_{\pm 0.17}$ &  ${94.66}_{\pm 0.07}$ \\
& 50  &  $\bm{95.07}_{\pm 0.06}$ &     ${94.99}_{\pm 0.03}$ &  ${94.85}_{\pm 0.06}$ \\
& 104 &     ${95.14}_{\pm 0.19}$ &  $\bm{95.31}_{\pm 0.07}$ &  ${95.10}_{\pm 0.21}$ \\

\midrule

C-100 & 32  &  $\bm{75.06}_{\pm 0.05}$ &  ${74.79}_{\pm 0.28}$ &  ${74.01}_{\pm 0.14}$ \\
& 50  &  $\bm{76.20}_{\pm 0.22}$ &  ${75.81}_{\pm 0.20}$ &  ${74.66}_{\pm 0.33}$ \\
& 104 &  $\bm{77.44}_{\pm 0.09}$ &  ${76.88}_{\pm 0.39}$ &  ${75.08}_{\pm 0.42}$ \\
\midrule

Tiny-I & 32      &  ${63.01}_{\pm 0.22}$ &  $\bm{63.06}_{\pm 0.04}$ &  ${62.79}_{\pm 0.08}$ \\
& 50      &  ${64.78}_{\pm 0.24}$ &     ${64.74}_{\pm 0.10}$ &  ${63.96}_{\pm 0.39}$ \\
& 104     &  ${66.57}_{\pm 0.39}$ &     ${66.67}_{\pm 0.12}$ &  ${65.27}_{\pm 0.52}$ \\
\bottomrule
\end{tabular}

\end{table}

\section{CONCLUSION}
Stable ResNets have the benefit of stabilizing the gradient and ensuring expressivity in the limit of infinite depth. We have demonstrated theoretically and empirically that this type of scaling makes NNGP inference robust and  improves test accuracy with SGD on modern ResNet architectures. However, while Stable ResNets with both uniform and decreasing scalings outperform standard ResNet, the selection of an optimal scaling remains an open question; we leave this topic for future work.

\section*{ACKNOWLEDGMENTS}
 This material is based upon work supported in part by the U.S. Army Research Laboratory and the U. S. Army Research Office, and by the U.K. Ministry of Defence (MoD) and the U.K. Engineering and Physical Research Council (EPSRC) under grant number EP/R013616/1. AD is also partially supported by EPSRC EP/R034710/1. BH is supported by the EPSRC and MRC through the OxWaSP CDT programme (EP/L016710/1). The project leading to this work has received funding from the European Research Council
(ERC) under the European Union’s Horizon 2020 research and innovation programme (grant agreement
No 834175).






\bibliographystyle{plain}
\bibliography{sample}

\newpage
\onecolumn
\section*{Appendix}

\renewcommand\theequation{A\arabic{equation}}
\renewcommand{\thelemma}{A\arabic{lemma}}
\renewcommand{\theprop}{A\arabic{prop}}
\renewcommand{\thecorollary}{A\arabic{corollary}}
\renewcommand{\thesection}{A\arabic{section}}
\renewcommand{\thedefinition}{A\arabic{definition}}
\renewcommand{\thetheorem}{A\arabic{theorem}}

\setcounter{equation}{0}
\setcounter{lemma}{0}
\setcounter{prop}{0}
\setcounter{theorem}{0}
\setcounter{corollary}{0}
\setcounter{definition}{0}
\setcounter{section}{-1}

\section{Mathematical preliminaries}\label{app:setup_notations}
We will make use of functional analysis results on the theory of Hilbert space. We refer to \citep{lang2012real} for a comprehensive introduction to the topic. We precise here that, even when not explicitly stated, all Hilbert spaces considered in the present work are real, and all linear operator are bounded. \\
We will make use of the spectral theory for compact self-adjoint operators. We refer again to \citep{lang2012real} for a detailed discussion.

We will now introduce some concepts from the theory of kernels and RKHSs. \\
Consider a compact $K\subset\R^d$. A function $Q:K\to \R$ is said to be symmetric if for all $x,x'\in K$ we have $Q(x,x')=Q(x',x)$. Let us restate the definition of kernel.
\begin{manualdefinition}{\ref{def:kernel}}[Kernel]
A kernel $Q$ on $K$ is a symmetric continuous function $K^2\to\mathbb R$ such that, for all $n\in\mathbb N$, for any finite subset $\{x_1\dots x_n\}\subset K$, the matrix $\{Q(x_i,x_j)\}_{i,j}$ is non-negative definite.
\end{manualdefinition}
We state here a characterisation of kernels, which is an extension of Lemma
\ref{lemma:T(Q)}. Despite being a classical result (see the discussion about Mercer kernels in \citep{paulsen2016RKHS}), we will give a proof, for the sake of completeness.
\begin{lemma}\label{T_mu(Q)}[Extension of Lemma \ref{lemma:T(Q)}]
Let $Q:K^2\to\R$ be a continuous symmetric function. Then, given any finite Borel measure $\mu$ on $K$, we can define the integral operator $T_\mu(Q)$ on $L^2(K,\mu)$, via
\begin{align*}
    T_\mu(Q)\,\f(x) = \int_K T(x,x')\,\f(x')\,\dd\mu(x')\,,
\end{align*}
for any $\f\in L^2(K,\mu)$. The operator $T_\mu(Q)$ is a bounded compact self-adjoint definite operator.\\
Moreover, $Q$ is a kernel if and only if $T_\mu(Q)$ is non-negative definite for all finite Borel measures $\mu$ on $K$.
\end{lemma}
\begin{proof}
Let $Q:K^2\to\R$ be a continuous symmetric function. Then $T_\mu(Q)$ is a well defined bounded compact self-adjoint operator \citep{lang2012real}.\\
Let us assume that $Q$ is a kernel. By Mercer's theorem \citep{paulsen2016RKHS}, we can find continuous functions $\{Y_k\}_{k\in\mathbb N}$ such that for all $x,x'\in K$
\begin{align*}
    Q(x,x') = \sum_{k=0}^\infty Y_k(x)Y_k(x')
\end{align*}
and the convergence is uniform on $K^2$.\\
The continuity of the $Y_k$'s implies that they can be seen as elements of $L^2(K,\mu)$. Moreover, the uniform convergence, along with the fact that $\mu(K)<\infty$, implies the convergence of the sum wrt the $L^2(K,\mu)$ operator norm. In particular $T_\mu(K)$ is a limit of non-negative definite operators and hence non-negative definite.\\
Now, assume that, for all finite Borel $\mu$, $T_\mu(Q)$ is non-negative definite. Chosen a finite set $\{x_1\dots x_n\}\subset K$, in particular we have that $\mu =\sum_{i=1}^n\delta_{x_i}$ is a finite Borel measure (where $\delta_x$ is the Dirac measure on $x\in K$). Hence $T_\mu(Q)$ is the matrix $\{Q(x_i,x_j)\}_{i,j}$. We conclude that $Q$ is a kernel.
\end{proof}
We will now give a definition of the Reproducing Kernel Hilbert Space associated to a kernel. We refer to \citep{paulsen2016RKHS} for a general and comprehensive introduction to the topic.
\begin{definition}[RKHS]
Given a kernel $Q$ on $K$, we can associate to it a real Hilbert space $\mathcal H_Q$, with the following properties:
\begin{itemize}
\item The elements of $\Hh_Q$ are functions $K\to\mathbb R$.
\item Denoting as $\langle\cdot,\cdot\rangle_Q$ the inner product of $\Hh_Q$, for each $x\in K$, there exists a element $k_x\in \Hh_Q$ such that $h(x)=\langle h,k_x\rangle_Q$, for all $h\in\Hh_Q$.
\item For all $x,x'\in K$, $\langle k_x,k_{x'}\rangle_Q = Q(x,x')$.
\end{itemize}
Such a Hilbert space exists for each kernel $Q$ and it is unique up to isomorphism, \citep{paulsen2016RKHS}. $\Hh_Q$ is called the Reproducing Kernel Hilbert Space (RKHS) of $Q$.
\end{definition}
In general, it is not easy to give an explicit form for the RKHS associated to a kernel $Q$. However, we can say that it contains the linear span of $\{x\mapsto Q(x,x')\}_{x'\in K}$. Actually, this linear span is a dense subset of $\Hh_Q$, wrt the norm of $\Hh_Q$ \citep{paulsen2016RKHS}. 

A kernel on $K$ is said to be universal if its RKHS is dense in the space of continuous functions $C(K)$, wrt the uniform norm.
\begin{manualdefinition}{\ref{univ_kern}}[Universal Kernel]
Let $Q$ be a kernel on $K$, and $\mathcal{H}_Q(K)$ its RKHS. We say that $Q$ is universal on $K$ if for any $\varepsilon>0$ and any continuous function $g$ on $K$, there exists $h \in \mathcal{H}_Q(K)$ such that $\|h-g\|_{\infty} <\varepsilon$.
\end{manualdefinition}
We can now state a characterization of universal kernels, from \citep{JMLR:v12:sriperumbudur11a}.
\begin{lemma}\label{univ=Lp_spd}
Let $Q:K^2\to\mathbb\R$ be a kernel, where $K\subset\R^d$ is compact. $Q$ is a universal kernel if and only if $T_\mu(Q)$ is strictly positive definite for all finite Borel measures $\mu$ on $K$, i.e., $\langle T_\mu(Q)\,\f,\f\rangle > 0$ for all non-zero $\f\in L^2(K,\mu)$.
\end{lemma}

As a final note, hereafter we often omit the explicit reference to the measure $\mu$, that is we will speak of the operator $T(Q)$ on $L^2(K)$. Unless otherwise stated, this notation implies the choice of an arbitrary finite Borel measure $\mu$ on the compact $K$.
\section{Residual Neural Networks and Gaussian processes}\label{app:section:ResNets_and_GPs}
Consider a standard ResNet architecture with $L+1$ layers, labelled with $l\in[0:L]$, of dimensions $\{N_l\}_{l\in[0:L]}$.
\begin{equation}\tag{\ref{equation:Resnet_dynamics}}
\begin{aligned}
y_0(x) &= W_0\,x + B_0\,; \\
y_l(x) &= y_{l-1}(x) + \mathcal{F}((W_l,B_l), y_{l-1}) \quad \mbox{for } l \in [1:L]\,,
\end{aligned}
\end{equation}
where $x \in \mathbb{R}^d$ is an input, $y^l(x)$ is the vector of pre-activations, $W^l$ and $B^l$ are respectively the weights and bias of the $l^{th}$ layer, and $\mathcal{F}$ is a mapping that defines the nature of the layer. In general, the mapping $\mathcal{F}$ consists of successive applications of simple activation functions. In this work, for the sake of simplicity, we consider Fully Connected blocks with ReLU activation function $\phi:x\mapsto \max(0,x)$
$$
\mathcal{F}((W,B), x) = W\phi(x) + B\,.
$$
Hereafter, $N_l$ denotes the number of neurons in the $l^{th}$ layer, $\phi$ the activation function and $[m:n]:=\{m,m+1, ..., n\}$ for  $m\leq n$. The components of weights and bias are respectively initialized with $W_l^{ij}\overset{\text{iid}}{\sim}\mathcal{N}(0,\sigma_w^2/N_{l-1})$, and $B_l^i\overset{\text{iid}}{\sim}\mathcal{N}(0,\sigma_b^2)$ where $\mathcal{N}(\mu, \sigma^{2})$ denotes the normal distribution of mean $\mu$ and variance $\sigma^{2}$. 

In \citep{yang2017meanfield}, authors showed that wide deep ResNets might suffer from gradient exploding during backpropagation.

Recent results by \citep{hayou_pruning} suggest that scaling the residual blocks with $L^{-1/2}$ might have some beneficial properties on model pruning at initialization. This is a result of the stabilization effect of scaling on the gradient.

More generally, we introduce the residual architecture:
\begin{align}\tag{\ref{equation:scaled_Resnet}}
\begin{split}
&y_0(x) = W_0\,x+B_0\,; \\
&y_l(x) = y_{l-1}(x) + \lambda_{l,L}\, \mathcal{F}((W_l,B_l), y_{l-1})\,, \quad l \in [1:L]\,,
\end{split}
\end{align}
where $(\lambda_{k,L})_{k\in[1:L]}$ is a sequence of scaling factors. We assume hereafter that there exists $\lambda_{\max} \in (0,\infty)$ such that for all $L\geq 1$ and $k \in [1:L]$, we have that $\lambda_{k,L} \in (0,\lambda_{\max}]$.

\subsection{Recurrence for the covariance kernel}
Recall that in the limit of infinite width, each layer of a ResNet can be seen a centred Gaussian Process. For the layer $l$ we define the covariance kernel $Q_l$ as
$Q_l(x,x') = \E[y_l^1(x)y_l^1(x')]$
for $x,x'\in \mathbb R^d$.

By a standard approach, introduced by \citep{samuel} for feedforward neural networks, and easily generalizable for ResNets \citep{yang2019tensor_i, hayou_ntk}, it is possible to evaluate the covariance kernels layer by layer, recursively. More precisely, consider a ResNet of form \eqref{equation:scaled_Resnet}. Assume that $y_{l-1}^i$ is a Gaussian process for all $i$. Let $x,x' \in \mathbb{R}^d$. We have that
\begin{equation*}
\begin{aligned}
Q_l(x,x') &= \mathbb{E}[y_l^1(x)y_l^1(x')]\\
&= \mathbb{E}[y_{l-1}^1(x)y_{l-1}^1(x')] + \sum_{j=1}^{N_{l-1}} \mathbb{E}[(W_l^{1j})^2 \phi(y^j_{l-1}(x))\phi(y^j_{l-1}(x'))] + \mathbb{E}[(B_{l}^1)^2] +  \mathbb{E}[B_{l}^1 (y_{l-1}^1(x) + y_{l-1}^1(x'))]\\
&+ \mathbb{E}\left[\sum_{j=1}^{N_{l-1}} W_l^{1j} (y_{l-1}^1(x) \phi(y_{l-1}^1(x')) + y_{l-1}^1(x') \phi(y_{l-1}^1(x)))\right].
\end{aligned}
\end{equation*}
Some terms vanish because $\mathbb{E}[W_l^{1j}] = \mathbb{E}[B_l^{j}] = 0$. Let $Z_j = \frac{\sqrt{N_{l-1}}}{\sigma_w} W_l^{1j}$. The second term can be written as 
$$
\mathbb{E}\left[\frac{\sigma_w^2}{N_{l-1}} \sum_j (Z_j)^2 \phi(y^j_{l-1}(x))\phi(y^j_{l-1}(x'))\right] \rightarrow \sigma_w^2 \,\mathbb{E}[\phi(y^1_{l-1}(x))\phi(y^1_{l-1}(x'))]\,,
$$
where we have used the Central Limit Theorem. Therefore, we have 
\begin{align}\label{eq:recurrence_resnet}
Q_l(x,x') = Q_{l-1}(x,x') + \lambda_{l,L}^2 \Psi_{l-1}(x,x')\,,
\end{align}
where $\Psi_{l-1}(x,x') = \sigma^2_b + \sigma^2_w \mathbb{E}[\phi(y_{l-1}^1(x))\phi(y_{l-1}^1(x'))]$.

For the ReLU activation function $\phi(x) = \max(0,x)$, the recurrence relation can be written more explicitly, since we can give a simple expression for the expectation $\E[\phi(y_{l-1}^1(x))\phi(y_{l-1}^1(x'))]$, \citep{daniely2016deeper}.
Let $C_l$ be the correlation kernel, defined as 
\begin{align*}
C_l(x,x') = \frac{Q_l(x,x')}{\sqrt{Q_l(x,x)Q_l(x,x')}}
\end{align*}
and let $f:[-1,1]\to\mathbb{R}$ be given by 
\begin{align}\tag{\ref{deff}}
f:\gamma\mapsto \frac{1}{\pi}(\sqrt{1-\gamma^2}-\gamma\arccos \gamma)\,.
\end{align}
Then we have $\E[\phi(y_{l-1}^1(x))\phi(y_{l-1}^1(x'))]= \tfrac{1}{2}\left(1+\tfrac{f(C_{l-1})}{C_{l-1}}\right)Q_{l-1}$
and so we find the recurrence relation \eqref{rec}
\begin{align*}\tag{\ref{rec}}
\begin{split}
&Q_{l} = Q_{l-1} + \lambda_{l,L}^2\left[{\sigma_b}^2 +\tfrac{{\sigma_w}^2}{2}\left(1+\tfrac{f(C_{l-1})}{C_{l-1}}\right)Q_{l-1}\right]\,;\\
&Q_0(x,x') =  \sbq+\swq\,\tfrac{x\cdot x'}{d}\,.
\end{split}
\end{align*}
For the remainder of this appendix, we define the function 
\begin{equation}\label{equation:f_hat}
\hat{f}(\gamma) =\gamma +  f(\gamma) = \frac{1}{\pi} \left(\gamma \arcsin(\gamma) + \sqrt{1-\gamma^2}\right) + \frac{1}{2}\gamma\,.
\end{equation}
For all $l$, the diagonal terms of $Q_l$ have closed-form expressions. We show this in the next lemma.
\begin{lemma}[Diagonal elements of the covariance]\label{lemma:diagonal_elements}
Consider a ResNet of the form \eqref{equation:scaled_Resnet} and let $x \in \mathbb{R}^d$. We have that for all $l \in [1:L]$,
$$
Q_{l}(x,x) = - \frac{2 \sigma_b^2}{\sigma_w^2}+\prod_{k=1}^{l}\left(1 + \frac{\sigma_w^2 \lambda_{k,L}^2}{2}\right)\left(Q_0(x,x) + \frac{2 \sigma_b^2}{\sigma_w^2}\right) \,.
$$
\end{lemma}
\begin{proof}
We know that 
$$
Q_{l}(x,x) = Q_{l-1}(x,x) + \lambda_{l,L}^2 \left(\sigma_b^2 + \frac{\sigma_w^2}{2} \hat{f}(1)\right)\,,
$$
where $\hat{f}$ is given by \eqref{equation:f_hat}. It is straightforward that $\hat{f}(1) = 1$. This yields
$$
Q_{l}(x,x) + \frac{2 \sigma_b^2}{ \sigma_w^2} =  \left(1+ \lambda_{l,L}^2 \frac{\sigma_w^2}{2}\right)\left(Q_{l-1}(x,x) + \frac{2 \sigma_b^2}{ \sigma_w^2}\right)\,.
$$
we conclude by telescopic product.
\end{proof}
As a corollary of the previous result, it is easy to show that for a Standard ResNet the diagonal terms explode with depth, which is Lemma \ref{lemma:exploding} in the main paper.
\begin{manuallemma}{\ref{lemma:exploding}}[Exploding kernel with standard ResNet]
Consider a ResNet of type \eqref{equation:Resnet_dynamics}. Then, for all $x\in \mathbb{R}^d$, 
$$Q_L(x,x) \geq \left(1+\tfrac{\swq}{2}\right)^L \left(\sigma_b^2 \left(1 + \tfrac{2}{\swq}\right) + \tfrac{\swq}{d} \|x\|^2\right).$$
\end{manuallemma}
\begin{proof}
The statement trivially follows from Lemma \ref{lemma:diagonal_elements}, using that $Q_0(x,x) = \sbq +\tfrac{\swq}{d}\|x\|^2$ and the fact that for a Standard ResNet \eqref{equation:Resnet_dynamics}, all the coefficients $\lambda_{l,L}$'s are equal to $1$.
\end{proof}
In the case of a ResNet with no bias, the correlation kernel follows a simple recursive formula described in the next lemma.
\begin{lemma}[Correlation formula with zero bias]\label{lemma:corr_formula}
For a ResNet of the form \eqref{equation:scaled_Resnet} with $\sigma_b=0$, we have that for all $x,x' \in \mathbb{R}^d$ and $l\leq L$:
$$
C_l(x,x') = \frac{1}{1+\alpha_{l,L}} C_{l-1}(x,x') + \frac{\alpha_{l,L}}{1 + \alpha_{l,L}} \hat{f}(C_{l-1}(x,x') )\,,
$$
where $\alpha_{l,L} = \frac{\lambda_{l,L}^2 \sigma_w^2}{2}$.
\end{lemma}
\begin{proof}
This is direct result of the covariance recursion formula \eqref{rec}.
\end{proof}

\subsection{Proof of Proposition \ref{proposition:stable_gradient}}
We use the following result from \citep{yang_tensor3_2020} in order to derive closed form expressions for the second moment of the gradients.
\begin{lemma}[Corollary of Theorem D.1. in \citep{yang_tensor3_2020}]\label{lemma:gradient_independence}
Consider a ResNet of the form \eqref{equation:scaled_Resnet} with weights $W$. In the limit of infinite width, we can assume that $W^T$ used in back-propagation is independent from $W$ used for forward propagation, for the calculation of Gradient Covariance and NTK.
\end{lemma}

Next we re-state and prove Proposition \ref{proposition:stable_gradient}.
\begin{manualprop}{\ref{proposition:stable_gradient}}[Stable Gradient]
Consider a ResNet of type \eqref{equation:scaled_Resnet}, and let $\mathcal{L}_y(x):= \ell(y_L^1(x), y)$
for some $(x,y) \in \mathbb{R}^d \times \mathbb{R}$, where $\ell : (z,y) \mapsto \ell(z,y)$ is a loss function satisfying $\sup_{K_1\times K_2}\left|\frac{\partial \ell(z,y)}{\partial z}\right| < \infty$, for all compacts $K_1, K_2 \subset \mathbb{R}$. Then, in the limit  of infinite width, for any compacts $K \subset \mathbb{R}^d$, $K' \subset \mathbb{R}$, there exists a constant $C>0$ such that for all $(x,y)\in K\times K'$
$$\sup\limits_{l \in [0: L]}\mathbb{E}\!\left[\left|\tfrac{\partial \mathcal{L}_y(x)}{\partial W_l^{11}}\right|^2\right] \leq C \exp{\left(\tfrac{\sigma_w^2}{2}\textstyle \sum_{l=1}^L \lambda_{l,L}^2\right)}\,.$$
Moreover, if there exists $\lambda_{\min}>0$ such that for all $L\geq 1$ and $l\in [1:L]$ we have $\lambda_{l,L}\geq\lambda_{\min}$, then, for all $(x,y) \in (\mathbb{R}^d\setminus\{0\}) \times \mathbb{R}$ such that $\left|\frac{\partial \ell(z,y)}{\partial z}\right| \neq 0$, there exists $\kappa>0$ such that for all $l \in [1:L]$
\begin{align*}
\mathbb{E}\left[\left|\tfrac{\partial \mathcal{L}_y(x)}{\partial W_l^{11}}\right|^2\right] \geq \kappa \left(1 + \tfrac{\lambda^2_{\min} \sigma_w^2}{2}\right)^L\,. 
\end{align*}
\end{manualprop}

\begin{proof}
Let $(x,y) \in \mathbb{R}^d \times \mathbb{R}$ and $\Bar{q}^l(x,y) = \mathbb{E}\left[\left|\frac{\partial \mathcal{L}_y(x)}{\partial y_l^{1}}\right|^2\right]$. Using Lemma \ref{lemma:gradient_backprop}, we have that 
$$
\Bar{q}^l(x,y) = \left(1 + \frac{\sigma_w^2 \lambda_{l+1,L}^2}{2}\right) \Bar{q}^{l+1}(x,y)\,.
$$
This yields
$$
\Bar{q}^l(x,y) = \prod_{k=l+1}^L \left(1 + \frac{\sigma_w^2 \lambda_{k,L}^2}{2}\right)
 \Bar{q}^l(x,y)\,.$$  
 
Moreover, using Lemma \ref{lemma:gradient_independence}, we have that $\mathbb{E}\left[\left|\frac{\partial \mathcal{L}_y(x)}{\partial W_l^{11}}\right|^2\right] = \lambda_{l,L}^2 \Bar{q}^l(x,y) \mathbb{E}[\phi(y_{l-1}^1(x))^2]$. We have $\mathbb{E}[\phi(y_{l-1}^1(x))^2] = \frac{1}{2} Q_{l-1}(x,x)$. From Lemma~\ref{lemma:diagonal_elements} we know that 
\begin{equation*}
Q_{l-1}(x,x) =  - \frac{2 \sigma_b^2}{\sigma_w^2}+\prod_{k=1}^{l-1}\left(1 + \frac{\sigma_w^2 \lambda_{k,L}^2}{2}\right) \left(Q_0(x,x) + \frac{2 \sigma_b^2}{\sigma_w^2}\right)
\leq \prod_{k=1}^{l-1}\left(1 + \frac{\sigma_w^2 \lambda_{k,L}^2}{2}\right)\left(Q_0(x,x) + \frac{2 \sigma_b^2}{\sigma_w^2}\right)\,,
\end{equation*}
This yields
$$
\mathbb{E}\left[\left|\frac{\partial \mathcal{L}_y(x)}{\partial W^l_{11}}\right|^2\right] \leq \frac{2}{\sigma_w^2} \prod_{k=1}^L\left(1 + \frac{\sigma_w^2 \lambda_{k,L}^2}{2}\right) \left(\frac{1}{2}Q_0(x,x) + \frac{ \sigma_b^2}{\sigma_w^2}\right) \Bar{q}^l(x,y)\,.
$$
It is straightforward that $\prod_{k=1}^L\left(1 + \frac{\sigma_w^2 \lambda_{k,L}^2}{2}\right) \leq  \exp{\left(\frac{\sigma_w^2}{2} \sum_{k=1}^L \lambda_{k,L}^2\right)}$. 
Let $K \subset \mathbb{R}^d$, $K' \subset \mathbb{R}$ be two compact subsets. Using the condition on the loss function $\ell$, we have that
$$
\mathbb{E}\left[\left|\frac{\partial \mathcal{L}_y(x)}{\partial W^l_{11}}\right|^2\right] \leq C \exp{\left(\frac{\sigma_w^2}{2} \sum_{k=1}^L \lambda_{k,L}^2\right)},
$$
where $C = \frac{2}{\sigma_w^2} \left(\sup_{(x,y) \in K\times K'}\Bar{q}^l(x,y)\right) \left(\sup_{x\in K} Q_0(x,x) + \frac{2 \sigma_b^2}{\sigma_w^2}\right)$. We conclude by taking the supremum over $l$ and $x,y$.\\

Let $(x,y) \in (\mathbb{R}^d\setminus\{0\}) \times \mathbb{R}$ such that $\left|\frac{\partial \ell(z,y)}{\partial z}\right| \neq 0$. We have that 

\begin{equation*}
\begin{split}
\mathbb{E}\left[\left|\frac{\partial \mathcal{L}_y(x)}{\partial W^l_{11}}\right|^2\right] &\geq \frac{1}{2} \frac{\lambda_{l,L}^2}{1 + \frac{\sigma_w^2}{2}\lambda_{l,L}^2}   \prod_{k=2}^L\left(1 + \frac{\sigma_w^2 \lambda_{k,L}^2}{2}\right) Q_1(x,x) \Bar{q}^l(x,y)\\
& \geq \kappa \left(1 + \frac{\sigma_w^2 \lambda_{\min}^2}{2}\right)^L\,,
\end{split}
\end{equation*}
where $\kappa = \frac{1}{2} \frac{\lambda_{\min}^2}{\left(1 + \frac{\sigma_w^2}{2}\lambda_{\max}^2\right)\left(1 + \frac{\sigma_w^2}{2}\lambda_{\min}^2\right)} Q_1(x,x)\, \Bar{q}^l(x,y) > 0$.
\end{proof}

Using Lemma~\ref{lemma:gradient_independence}, we can derive simple recursive formulas for the second moment of the gradient as well as for the Neural Tangent Kernel (NTK). This was previously done in \citep{samuel} for feedforward neural networks, we prove a similar result for ResNet in the next lemma.
\begin{lemma}[Gradient Second moment]\label{lemma:gradient_backprop}
In the limit of infinite width, using the same notation as in proposition \ref{proposition:stable_gradient}, we have that 
$$
\Bar{q}^l(x,y) = \left(1 + \frac{\sigma_w^2 \lambda_{l+1,L}^2}{2}\right)\, \Bar{q}^{l+1}(x,y)\,.
$$
\end{lemma}
\begin{proof}
It is straighforward that 
$$
\frac{\partial \mathcal{L}_y(x)}{\partial y_l^i} = \frac{\partial \mathcal{L}_y(x)}{\partial y_{l+1}^i} + \lambda_{l+1,L} \sum_{j} \frac{\partial \mathcal{L}_y(x)}{\partial y_{l+1}^j} W_{l+1}^{ji} \phi'(y_l^i)\,.
$$
Using lemma \ref{lemma:gradient_independence} and the Central Limit Theorem, we have that 
$$
\Bar{q}^l(x,y) = \Bar{q}^{l+1}(x,y) + \lambda_{l+1,L}^2 \Bar{q}^{l+1}(x,y) \sigma_w^2 \mathbb{E}[\phi'(y_l^i(x))^2]\,.
$$
We conclude using $\mathbb{E}[\phi'(y^l_i(x))^2] = \mathbb{P}(\mathcal{N}(0,1) > 0) = \frac{1}{2}$.
\end{proof}

Before moving to the next proofs, recall the definition of Stable ResNet.

\begin{manualdefinition}{\ref{def:stable_resnet}}[Stable ResNet]
A ResNet of type \eqref{equation:scaled_Resnet} is called a Stable ResNet if and only if $\lim\limits_{L \rightarrow \infty}\sum\limits_{k=1}^L \lambda_{k,L}^2 < \infty$.
\end{manualdefinition}

\subsection{Some general results: $Q_l$ and $C_l$ are kernels}
Fix a compact $K\subset \R^d$. If $\sigma_b=0$, then assume that $0\notin K$. We will now show that, for all layers $l$, the covariance function $Q_l$ is a kernel in the sense of Definition \ref{def:kernel}.\\
The symmetric property of $Q_l$ is clear by definition as the covariance of a Gaussian Process. Let us now discuss the regularity of $Q_l$ as a function on $K^2$.

The next result shows that any function $ F(\phi):\gamma\mapsto\mathbb{E}[\phi(X)\phi(Y), (X,Y) \sim \mathcal{N}(0, \begin{psmallmatrix}
1 & \gamma \\
\gamma & 1
\end{psmallmatrix})]$ is analytic on the segment $[-1,1]$.
\begin{lemma}[O'Donnell (2014)]\label{odonnell}
Let $F(\phi)(\gamma) = \mathbb{E}[\phi(X)\phi(Y), (X,Y) \sim \mathcal{N}(0, \begin{psmallmatrix}
1 & \gamma \\
\gamma & 1
\end{psmallmatrix})]$. Then for all $\phi \in L^2(\mathcal{N}(0,1))$, there exists a non negative sequence $\{a_n\}_{n\in\mathbb N}$ such that $F(\phi)(\gamma) = \sum_{i\in\mathbb N} a_i \gamma^i$ for all $\gamma \in [-1,1]$. 
\end{lemma}
Leveraging the previous result, the function $f$ defined in \eqref{deff} is analytic. We clarify this in the next lemma.
\begin{lemma}[Analytic property of $f$]\label{lemma:analytic_decomp_f_hat}
The function $f:[-1,1]\to\mathbb{R}$, defined in \eqref{deff}, is an analytic function on $(-1,1)$, whose expansion $f(\gamma) = \sum_{n\in\mathbb{N}}\al_n\,\gamma^n$
converges absolutely on $[-1,1]$. Moreover, $\al_n>0$ for all even $n\in\mathbb{N}$, $\al_1=-1/2$ and $\al_n=0$ for all odd $n\geq 3$.
\end{lemma}
\begin{proof}
With the notations of Lemma \ref{odonnell}, when $\phi$ is the ReLU activation function we have that $F(\phi) = \hat f$, defined in \eqref{equation:f_hat}. Hence, by Lemma \ref{odonnell}, we know that $\hat f$ is analytic on $(-1,1)$ and its expansion around $0$ converges on $[-1,1]$. In particular this will be true for $f$ as well. \\
For $\gamma\in[-1,1]$, let us write $\hat f(\gamma) = \sum_{n\in\mathbb{N}}a_n\gamma^n$.
Recalling the explicit form of $\hat f$, that is
$$
\hat{f}(\gamma) = \frac{1}{\pi} \gamma \,\text{arcsin}(\gamma) + \frac{1}{\pi} \sqrt{1-\gamma^2} + \frac{1}{2}\gamma\,,
$$
we get $a_0 = \frac{1}{\pi}$. Moreover, we have that for all $\gamma \in (-1,1)$
$$
\hat{f}'(\gamma) = \frac{1}{\pi} \arcsin{\gamma} + \frac{1}{2}\,.
$$
This yields $a_1 = \hat f'(0) = \frac{1}{2}$. Then, noticing that
$$
\hat{f}^{(3)}(\gamma) = \frac{\gamma}{\pi (1 - \gamma^2)^{3/2}}
$$
is an odd function, we get that for all $i \geq 1, a_{2i+1} = 0$. Now let us prove that for all $k\geq 1$, there exist $b_{k,0},b_{k,1},..., b_{k,k-1} > 0$ such that, for all $\gamma \in (-1,1)$,
$$
\hat{f}^{(2k)}(\gamma) = \frac{1}{\pi} \sum_{m=0}^{k-1} b_{k,m} \gamma^{2m} (1 - \gamma^2)^{-k-m+1/2}\,.
$$
We prove this by induction. For $k=1$, we have that 
$$
\hat{f}^{(2)}(\gamma) = \frac{1}{\pi} (1 - \gamma^2)^{-1/2}\,,
$$
so that our claim holds. Assume now that it is true for some $k\geq 1$, let us prove it for $k+1$. It is easy to see that
\begin{equation}
    b_{k+1,m} =
    \begin{cases*}
      2(2k-1) b_{k,0} + 2b_{k,1} & if $m = 0$; \\
      2(4k^2-1) b_{k,0} + 5(2k+1) b_{k,1} + 12 b_{k,2} & if $m = 1$; \\
      2(m+1)(2m+1) b_{k,m+1} + (4m+1)(2k+2m-1) b_{k,m} \\
      \quad +\, (2k+2m-3)(2k+2m-1) b_{k,m-1} & if $m \in \{2,3,...,k-1\}$; \\
      (4k-3)(4k-1) b_{k,k-1} & if $m = k$. \\
    \end{cases*}
  \end{equation}

The induction is straightforward.
In particular, we have shown that $a_{2i} = \frac{\hat{f}^{(2i)}(0)}{(2i)!} = \frac{b_{i,0}}{(2i)!}> 0$.\\
The conclusion for the coefficients $\al$'s of the expansion of $f$ is then trivial.
\end{proof}

Using Lemma \ref{lemma:analytic_decomp_f_hat}, it will not be hard to show that $Q_l$ is continuous. The non-negativity of $T(Q_l)$ can be seen as a consequence of the definition of $Q_l$ as the covariance of a Gaussian Process. However, we will give a direct proof of it, so that we can state here a general result which we will need later on. 
\begin{lemma}\label{schur}
Let $C$ be a kernel on $K$, such that $|C(z)|\leq 1$ for all $z\in K$. Consider a non-negative real sequence $\{\al_n\}_{n\in\mathbb{N}}$, and assume that
\begin{align*}
g(\gamma) = \sum_{k=0}^\infty \al_k\, \gamma^k
\end{align*} 
converges uniformly on $[-1,1]$. Then, for all finite Borel measure $\mu$ on $K$, $T_\mu(g(C))$ is a non-negative definite compact operator, and in particular $g(C)$ is a kernel.
\end{lemma}
\begin{proof}
Fix a finite Borel measure $\mu$ on $K$ and notice that $g(C)$ is continuous and symmetric (as uniform limit of continuous and symmetric functions). Moreover, since the Taylor expansion of $g$ around $0$ converges uniformly on $[-1,1]$, and since $|C(z)|\leq 1$ for all $z\in K$, we have that $T_\mu(g(C)) = \sum_{k\in\mathbb N}\al_k\,T_\mu(C^k)$, the sum converging wrt the operator norm on $L^2(K,\mu)$.\\
As a consequence of the Schur product theorem\footnote{\label{footnote_schur}Given two matrices $M_1$ and $M_2$, define they're Schur product as the matrix $M = M_1\circ M_2$, whose elements are $M^{ij} = {M_1}^{ij}{M_2}^{ij}$. If $M_1$ and $M_2$ are non-negative definite, then $M$ is non-negative definite.}, the product of two kernels is still a kernel.\\
As a consequence, it is easy to prove by induction that $T_\mu(C^k)$ is non-negative definite for all $k$. Hence $T_\mu(g(C))$ is the converging limit of a sum of compact non-negative definite operator. We conclude by Lemma \ref{T_mu(Q)}.
\end{proof}
\begin{lemma}\label{lemma:Ql_is_kernel}
For both Standard and Stable ResNet architectures, for any layer $l$, the covariance function $Q_l$ and the correlation function $C_l$ are kernels on $K$, in the sense of Definition \ref{def:kernel}.
\end{lemma}
\begin{proof}
It is straightforward to prove that $Q_0$ is a kernel. Now let us show that if $Q_l$ is a kernel for some $l$, then $C_l$ is a kernel. Since $Q_l$ is symmetric and so $C_l$ is. Moreover, the diagonal elements of $Q_l$ are continuous by Lemma \ref{lemma:diagonal_elements} and do not vanish (since if $\sigma_b=0$ we are assuming that $0\notin K$). Hence $C_l$ is continuous. It is then trivial to show that the non-negative definiteness of $T(Q_l)$ implies that $T(C_l)$ is non-negative definite, and so $C_l$ is a kernel if $Q_l$ is.\\
Now we proceed by induction. Suppose that $Q_{l-1}$ and $C_{l-1}$ are kernels and recall the recursion \eqref{rec}, taking the coefficient $\lambda$ to be $1$ in the case of a Standard ResNet. Notice that it can be rewritten as 
\begin{align*}
&Q_{l} = Q_{l-1} + \lambda_{l}^2\left({\sigma_b}^2 +\tfrac{{\sigma_w}^2}{2}\hat f(C_{l-1})\,R_{l-1}\right)\,,
\end{align*}
where we have omitted the dependence on $L$ for $\lambda$, we have defined $R_{l-1}(x,x') = \sqrt{Q_{l-1}(x,x)Q_{l-1}(x',x')}$ and $\hat f$ is defined in \eqref{equation:f_hat}. Clearly $R_{l-1}$ is a kernel. By Lemma \ref{lemma:analytic_decomp_f_hat} and Lemma \ref{schur} we have that $\hat f(C_l)$ is a kernel. Using the property that sums and products of kernels are kernels (the sum is trivial, cf Footnote \ref{footnote_schur} for the product), we conclude that $Q_l$, and so $C_l$, is a kernel on $K$.
\end{proof}

\subsection{Proof of Proposition \ref{prop:rkhs_hierarchy}}
As always, consider an arbitrary compact set $K\in\R^d$. Assume that $0\notin K$ if $\sigma_b=0$. Recall from Appendix \ref{app:setup_notations} that with the notation $\mathcal H_Q(K)$ we refer to the RKHS generated by a kernel $Q$ on $K$. We will now prove Proposition \ref{prop:rkhs_hierarchy}.

\begin{manualprop}{\ref{prop:rkhs_hierarchy}}
$\mathcal{H}_{Q_l}(K) \subseteq \mathcal{H}_{Q_{l+1}}(K)$ for all $l \in [0:L-1]$.\\
\end{manualprop}
\begin{proof}

We have already shown that $T(Q_l)-T(Q_{l-1})$ is non-negative definite in the proof of Lemma \ref{lemma:Ql_is_kernel}. We conclude by using the RKHS hierarchy result (see for instance \citep{paulsen2016RKHS} or page 354 in \citep{aronszajin_rkhs_hierarchy}).
\end{proof}

\subsection{Proof of Lemma \ref{univ=>expr}}
We present here the proof of Lemma \ref{univ=>expr}. We have already recalled the Definition \ref{univ_kern} of universal kernel in Appendix \ref{app:setup_notations}. For convenience of the reader, we restate here the definition of expressive GP.\\
Let $K$ be a compact in $\R^d$. 
\begin{manualdefinition}{\ref{def:exprGP}}[Expressive GP]
A Gaussian Process on $K$ is said to be expressive on $L^2(K)$ if, denoted by $\psi$ a random realisation, for all $\f\in L^2(K)$, for all $\varepsilon>0$,
\begin{align*}
     \mathbb{P}(\|\psi-\f\|_2\leq \varepsilon)>0\,.
 \end{align*}
\end{manualdefinition}
\begin{manuallemma}{\ref{univ=>expr}}
A universal kernel $Q$ on $K$ induces an expressive GP on $L^2(K)$. 
\end{manuallemma}
\begin{proof}
First, notice that if $Q$ is universal then $T(Q)$ is strictly positive definite \citep{JMLR:v12:sriperumbudur11a} and so all its eigenvalues are strictly positive.\\
Recall the spectral theorem for compact self-adjoint operators: there is a orthonormal basis of $L^2(K)$ made of the eigenfunctions $\{\psi_n\}_{n\in\mathbb{N}}$ of $T(Q)$. Denoting by $\mu_n>0$ the eigenvalue of $T(Q)$ relatively to $\psi_n$, since $T(Q)$ is compact we have the equality (Karhunen - Lo\`eve decomposition \citep{grenander1950})
\begin{align*}
\psi = \sum_{k=0}^\infty Z_k\sqrt{\mu_k}\,\psi_k\sim\mathcal{GP}(0,Q)\,,
\end{align*}
where $\{Z_k\}_{k\in\mathbb{N}}$ is a family of iid normal random variables, and the series is convergent uniformly on $K$ and in $L^2$ for the stochastic part \citep{paulsen2016RKHS}, that is $\lim_{N\to\infty}\sup_{x\in K}\E[(\psi(x)-\sum_{k=0}^N Z_k\sqrt{\mu_k}\,\psi_k(x))^2]=0$ uniformly for $x\in K$. In particular, we get that $\lim_{N\to\infty}\E[\|\psi-\sum_{k=0}^NZ_k\sqrt{\mu_k}\psi_k\|_2^2]=0$. As consequence, for all $\f\in L^2(K)$, we have that $\|\sum_{k=0}^N Z_k\sqrt{\mu_k}\,\psi_k - \f\|_2^2$ converges in squared mean to $\|\psi - \f\|_2^2$, for $N\to \infty$.\\
Now, let $\f = \sum_{k=0}^N a_k\,\psi_k$ for some finite $N$ and some real coefficients $\{a_0\dots a_N\}$. We have (with convergence in squared mean)
\begin{align*}
\|\psi-\f\|_2^2 = \sum_{k=0}^N\left(Z_k\sqrt{\mu_k}-a_k\right)^2 + \sum_{k=N+1}^\infty\mu_k\,Z_k^2\,.
\end{align*}
For $k\in[0:N]$, we can define the interval $I_k = \left[\tfrac{a_k}{\sqrt{\mu_k}}-\tfrac{\varepsilon}{\sqrt{2(N+1)\mu_k}},\tfrac{a_k}{\sqrt{\mu_k}}+\tfrac{\varepsilon}{\sqrt{2(N+1)\mu_k}}\right]$, so that, for all $z\in I_k$ we have $(z\sqrt{\mu_k}-a_k)^2\leq\tfrac{\varepsilon^2}{2(N+1)}$. Since all these intervals are non empty, we get
\begin{align*}
\mathbb{P}\left(\sum_{k=0}^N \left(Z_k\sqrt{\mu_k}-a_k\right)^2\leq\frac{ \varepsilon^2}{2}\right)\geq \prod_{k=0}^N \mathbb{P}(Z_k\in I_k) > 0\,.
\end{align*}
On the other hand, we have that 
\begin{align*}
\delta_N = \mathbb{E}\left[\sum_{k=N+1}^\infty\mu_k\,Z_k^2\right] = \sum_{k=N+1}^\infty\mu_k\,.
\end{align*}
By Mercer's theorem \citep{paulsen2016RKHS}, $T(Q)$ is trace class and hence $\delta_N\to 0$ for diverging $N$. By Markov's inequality
\begin{align*}
\mathbb{P}\left(\sum_{k=N+1}^\infty \mu_kZ_k^2\geq \frac{\varepsilon^2}{2} \right)\leq \frac{2\delta_N}{\varepsilon^2}
\end{align*}
and we can conclude that $\mathbb{P}(\|\psi-\f\|_2\leq \varepsilon) > 0$ for $N$ large enough. \\

For a general $\f = \sum_{k=0}^\infty a_k\psi_k$, let $\f_N = \sum_{k=0}^N a_k\psi_k$. Since $\{\psi_k\}_{k\in\mathbb N}$ is a basis of $L^2(K)$, fixed $\varepsilon>0$, it is always possible to find a $N$ such that $\|\f-\f_N\|_2 \leq \varepsilon/2$ and $\mathbb{P}(\|\f_N-\psi\|_2\leq \varepsilon/2)>0$, and so we conclude.
\end{proof}

\subsection{Proof of Proposition \ref{prop:universality_on_a_compact}}
In order to prove Proposition \ref{prop:universality_on_a_compact} we first need a preliminary result, which will be at the core of the proof of Theorem \ref{thm:expr} as well.
\begin{prop}\label{prop:univ_K}
Let $K\subset\R^d$ be compact. Assume $\sigma_b>0$ and let $\tilde f:\gamma\mapsto \tfrac{\gamma}{2}+f(\gamma)$ be defined on $[-1,1]$. Then the kernel $\tilde f(c_0)$, defined point-wise as $\tilde f(c_0)(x,x') = \tilde f(c_0(x,x'))$, is universal on $K$.
\end{prop}
\begin{proof}
First notice that $c_0(x,x') = \tfrac{1+\zeta\,x\cdot x'}{\sqrt{(1+\zeta\,\|x\|^2)(1+\zeta\,\|x'\|^2)}}$, where $\zeta = \swq/\sbq$. For $n\in\mathbb N$, define $p_n: (x,x')\mapsto c_0(x,x')^{2n}$, with the convention that $p_0\equiv 1$. It is easy to verify that $c_0$ is kernel. As a consequence, $p_n$ is a kernel for all $n$, since it is a product of kernels.\footnote{See footnote \ref{footnote_schur}.} From Lemma \ref{lemma:analytic_decomp_f_hat}, we can write 
\begin{align*}
\tilde f(c_0) = \sum_{n\in\mathbb N}\al_n\,p_n\,,
\end{align*}
the sum converging uniformly on $K^2$, with $\al_n>0$ for all $n\in\mathbb{N}$. By Lemma \ref{schur}, $\tilde f(c_0)$ is a kernel.\\
Now, for each $n$, we have
\begin{align*}
p_n(x,x') = \frac{1}{(1+\zeta\,\|x\|^2)^n(1+\zeta\,\|x'\|^2)^n}\sum_{k=0}^{2n}\omega_{k,n}\,(x\cdot x')^k \,,
\end{align*}
where the coefficients $\omega_{k,n}$'s are all strictly positive, explicitly $\omega_{k,n} = \zeta^k\binom{n}{k}$. \\
Expanding the inner product $x\cdot x'$, we can express $p_n$ in the form 
\begin{align*}
p_n(x,x') = \sum_{J\in\mathcal J_n}\be_{J,n}\,A_{J,n}(x)A_{J,n}(x')\,,
\end{align*}
where $\mathcal J_n =\{(j_1\dots j_d)\in\mathbb N^d: \sum_{i=1}^dj_i \in[0:2n]\}$, all the coefficients $\be_{J,n}$'s are strictly positive and the $A_{J,n}$'s are defined as
\begin{align*}
A_{J,n}(x) = \frac{{x_1}^{j_1}\dots {x_d}^{j_d}}{(1+\zeta\,\|x\|^2)^n}\,.
\end{align*}
Hence we can write $\tilde f(c_0)$ as
\begin{align}\label{feature_map}
\tilde f(c_0)(x,x') = \sum_{n\in\mathbb N}\sum_{J\in\mathcal J_n} \al_n\be_{J,n} A_{J,n}(x) A_{J,n}(x')\,.
\end{align}
For any $n,n'\in\mathbb N$, $J\in\mathcal J_n$, $J'\in\mathcal J_{n'}$, it is clear that $A_{J,n}A_{J',n'} = A_{J'',n+n'}$, where $J''$ is some element in $\mathcal J_{n+n'}$. As a consequence, the linear span of the family $\{A_{J,n}\}_{n\in\mathbb N, J\in \mathcal J_n}$ is an algebra $\mathcal A$ (which is actually a subalgebra of $C(K)$ since all the $A_{J,n}$'s are continuous). Moreover $A_{(0\dots0),0}\equiv 1$, so that $\mathcal A$ contains a constant, and it is straightforward to check that $\mathcal A$ separates points, that is for all distinct $x,x'\in K$ there exists $a\in\mathcal A$ such that $a(x)\neq a(x')$. Then, from Stone-Weierstrass theorem \citep{lang2012real}, $\mathcal A$ is dense in $C(K)$ wrt the uniform norm. \\
For all $n\in\mathbb N$, $J\in\mathcal{J}_n$, let $\theta_{n,J} = \sqrt{\al_n\beta_{n,J}}$. Define a bijection $\iota:\mathbb{N}\to\{(n,J):n\in\mathbb N, J\in\mathcal J_n\}$ and let $\Phi_n = \theta_{\iota(n)}A_{\iota(n)}$. For all $x\in K$, we have that $\Phi(x) = \{\Phi_n(x)\}_{n\in\mathbb{N}}\in \ell^2$, since $p_n(x,x)<\infty$. We conclude that $\Phi$ is a feature map for $\tilde f(c_0)$, and the density of the linear span of $\{\Phi_n\}_{n\in\mathbb N}$ allows to claim that the kernel is universal on $K$, in the sense of Definition \ref{univ_kern} (cf Theorem 7 in \citep{micchelli}).
\end{proof}
Let $K\subset\R^d$ be an arbitrary compact set. We are now ready to prove Proposition \ref{prop:universality_on_a_compact}.
\begin{manualprop}{\ref{prop:universality_on_a_compact}}
If $\sigma_b > 0$, then $Q_2$ is universal on $K$. From Proposition \ref{prop:rkhs_hierarchy}, $Q_L$ is universal for all $L\geq 2$.
\end{manualprop}
\begin{proof}
Assume $\sigma_b>0$ and let $K \subset \mathbb{R}^d$ be a compact set. With the notation of Proposition \ref{prop:univ_K}, we have that $$ Q_1 = Q_0 + \lambda_{1,L}^2 \left(\sigma_b^2 + \frac{\sigma_w^2}{2} \left(\frac{1}{2} + \frac{\Tilde{f}(C_0)}{C_0}\right) Q_0 \right). $$
By proposition \ref{prop:univ_K}, we know that the kernel $\Tilde{f}(C_0)$ given by $\Tilde{f}(C_0)(x,x') = \Tilde{f}(C_0(x,x'))$ is universal on $K$. Let us prove that $\frac{\Tilde{f}(C_0)}{C_0} Q_0$ is universal. Let $\varepsilon>0$ and $\f \in C(K)$, the space of continuous functions on $K$. Define $\frac{\f}{Q_0}(x) = \frac{\f(x)}{Q_0(x,x)}$. By the universality of $\Tilde{f}(C_0)$, there exists $g \in \mathcal{H}_{\Tilde{f}(C_0)}(K)$ such that 
$$
\left\|g - \frac{\f}{\sqrt{Q_0}}\right\|_\infty \leq \epsilon\,.
$$
with $g$ can be written as a finite linear combination of the functions $\{\hat{f}(C_0)(x,.)\}_{x \in K}$. This yields
$$
\left\|g \sqrt{Q_0} - \f\right\|_\infty \leq \epsilon \kappa\,,
$$
where $g \sqrt{Q_0} (x) = g(x) \sqrt{Q_0(x,x)}$ and $\kappa = \sup_{x \in K} \sqrt{Q_0(x,x)}$. It is straightforward that $g \sqrt{Q_0} \in \mathcal{H}_{\frac{\Tilde{f}(C_0)}{C_0} Q_0}(K)$,\footnote{This is trivial for a function $g$ that can be written as a finite sum of functions of the form $\alpha_i \tilde{f}(C_0)(x_i, .)$, and this would be enough since these functions are dense in $C(K)$ as shown in the proof of Proposition \ref{prop:univ_K}. More generally, given two kernels $Q$ and $Q'$, if $h\in \Hh_Q$ and $h'\in\Hh_{Q'}$, then $hh'\in \Hh_{QQ'}$, cf Theorem 5.16 in \citep{paulsen2016RKHS}.} Therefore, $\frac{\Tilde{f}(C_0)}{C_0} Q_0$ is universal. Since $Q_0$ is non-negative, we have that $Q_1$ is universal by an RKHS hierarchy argument similar to Proposition \ref{prop:rkhs_hierarchy}. Using Proposition \ref{prop:rkhs_hierarchy}, we conclude that $Q_L$ is universal on $K$.
\end{proof}

\subsection{Proof of Proposition \ref{prop:universality_on_sphere}}
\begin{manualprop}{\ref{prop:universality_on_sphere}}
Assume $\sigma_b = 0$. Then for all $L \geq 2$, $Q_L$ is universal on $\Sd$ for $d\geq 2$.
\end{manualprop}
\begin{proof}
See the proof of Proposition \ref{prop:universality_on_sphere_appendix} in Appendix \ref{app:sphere}.
\end{proof}

\subsection{Proof of Proposition \ref{proposition:spectral_decomposition_on_sphere}}
Proposition \ref{proposition:spectral_decomposition_on_sphere} is a well known classical result (see for instance Appendix H in \citep{yang2019finegrained} and the references therein. For completeness we give a proof in Appendix \ref{app:sphere}.
\begin{manualprop}{\ref{proposition:spectral_decomposition_on_sphere}}[Spectral decomposition on $\Sd$]
Let $Q$ be a zonal kernel on $\Sd$, that is $Q(x,x') = p(x\cdot x')$ for a continuous function $p:[-1,1]\to\mathbb{R}$. Then, there is a sequence $\{\mu_k\geq 0\}_{k\in\mathbb N}$ such that for all $x,x'\in\Sd$
\begin{align*}
Q(x,x') = \sum_{k\geq 0 } \mu_{k} \sum_{j=1}^{N(d,k)} Y_{k,j}(x) Y_{k,j}(x')\,,
\end{align*}
where $\{Y_{k,j}\}_{k\geq0, j\in [1:N(d,k)]}$ are spherical harmonics of  $\mathbb{S}^{d-1}$ and $N(d,k)$ is the number of harmonics of order $k$. With respect to the standard spherical measure, the spherical harmonics form an orthonormal basis of $L^2(\Sd)$ and $T(Q)$ is diagonal on this basis.
\end{manualprop}
\begin{proof}
See the proof of Lemma \ref{lemma:spherical_decomposition} in Appendix \ref{app:sphere}.
\end{proof}
\subsection{Proof of Lemma \ref{lemma:infnite_depth_standard_resnet}}
\begin{manuallemma}{\ref{lemma:infnite_depth_standard_resnet}}
Consider a standard ResNet of type \eqref{equation:Resnet_dynamics} and let $K \subset \mathbb{R}^d \setminus \{0\}$ be a compact set. We have that
$$\lim_{L \rightarrow \infty}\sup_{x,x' \in K}\left| 1 - C_L(x,x')\right| = 0\,.$$
Moreover, if $\sigma_b=0$, then,
$$\sup_{x,x' \in K}\left| 1 - C_L(x,x')\right| = \mathcal{O}(L^{-2})\,.$$
Therefore, $\mathcal{H}_{C_\infty}(K)$ is the space of constant functions.
\end{manuallemma}
\begin{proof}
This result was proven in \citep{hayou} in the case of no bias. It was also proven for a slightly different ResNet architecture in \citep{yang2017meanfield}.\\
Consider a ResNet of type \eqref{equation:Resnet_dynamics} and let $K \subset \mathbb{R}^d \setminus \{0\}$ be a compact set. We have that for all $x,x' \in K$
$$
Q_L(x,x') = Q_{L-1}(x,x') + \sigma_b^2 + \frac{\sigma_w^2}{2} \hat{f}(C_{L-1}(x,x')) \sqrt{Q_{L-1}(x,x) Q_{L-1}(x',x')}\,.
$$
Since $\hat{f}(x) \geq x$, $C_L$ is non-decreasing wrt $L$ and converges to the unique fixed point of $\hat{f}$ which is $1$. This convergence is uniform in $x,x'$, i.e. $\lim_{L \rightarrow \infty}\sup_{x,x' \in K} 1 - C_L(x,x') = 0 $.\\
Re-writing the recursion yields
$$
C_L(x,x') = \delta_L \frac{1}{1+\alpha} C_{L-1}(x,x') + \zeta_L + \delta_L \frac{\alpha}{1 + \alpha} \hat{f}(C_{L-1}(x,x'))\,,
$$
where $\alpha = \frac{\sigma_w^2}{2}$, $\delta_l = \left(1 + \frac{\sigma_b^2}{(1 + \alpha)Q_{L-1}(x,x)} \right)^{-1/2} \left(1 + \frac{\sigma_b^2}{(1 + \alpha)Q_{L-1}(x,x)} \right)^{-1/2}$ and $\zeta_L = \sigma_b^2 (Q_L(x,x) Q_L(x',x'))^{-1/2}$. \\
Using Lemma~ \ref{lemma:diagonal_elements}, and the boundedness of $C_L$, a simple Taylor expansion yields
\begin{align*}
C_L(x,x') &= \frac{1}{1+\alpha} C_{L-1}(x,x') + \frac{\alpha}{1 + \alpha} \hat{f}(C_{L-1}(x,x')) + g_L(x,x')\\
&= C_{L-1}(x,x') + \frac{\alpha}{1 + \alpha} f(C_{L-1}(x,x')) + g_L(x,x')\,,
\end{align*}
 where the expansion is uniform on $x,x' \in K$, and $f(x) = \hat{f}(x) - x$, and $g_L = \mathcal{O}(e^{-\beta L})$ for some $\beta>0$.\\
The previous dynamical system can be decomposed in two parts, a first part without the term $\mathcal{O}(e^{-\beta L})$ which is the homogeneous system, i.e.\! the system without bias, and the term $\mathcal{O}(e^{-\beta L})$ which is the contribution of the bias in the dynamical system. \\
Assume $\sigma_b=0$, then the term $g_L$ vanishes. Moreover, a Taylor expansion of $\hat{f}$ near 1 yields
$$
f(x) = s (1- x)^{3/2} + \mathcal{O}((1-x)^{5/2})\,.
$$
Therefore, uniformly in $x,x' \in K$, we have that
$$
C_L(x,x') = C_{L-1}(x,x') + \frac{s \alpha}{1 + \alpha} (1- C_{L-1}(x,x'))^{3/2} + \mathcal{O}((1-C_{L-1}(x,x'))^{5/2})\,.
$$
Letting $\gamma_L = 1 - C_L$, a simple Taylor expansion leads to 
$$
\gamma_L^{-1/2} = \gamma_{L-1}^{-1/2} + \frac{s \alpha}{2(1 + \alpha) } + \mathcal{O}(\gamma_{L-1})\,.
$$
Therefore, $\gamma_L \sim \kappa L^{-2}$ where $\kappa = \frac{4(1+\alpha)^2}{s^2 \alpha^2}$. This equivalence is uniform in $x,x' \in K$.

It is likely that the rate $\mathcal{O}(L^{-2})$ holds without assuming $\sigma_b=0$. However, the analysis in this requires unnecessarily complicated details. 
\end{proof}
\section{Stable ResNet with uniform scaling}\label{app:section:uniform_scaling}
In this section we detail the proofs for the uniform scaling of a Scaled ResNet, that is $\lambda_{l,L}=1/\sqrt{L}$. \\
When not otherwise specified, $K$ is a generic compact of $\R^d$. We assume that $0\notin K$ if $\sigma_b = 0$.

\subsection{Continuous formulation}
We provide the results of existence, uniqueness and regularity of the solution of \eqref{recc} in Lemma \ref{ex}. Corollary \ref{corex} shows that the differential problem can be restated in the operator space. Eventually we give a proof of Lemma \ref{convcont}, assuring uniform convergence to the continuous limit.

We recall that by continuous formulation we mean a rescaling of the layer index $l$, which becomes a continuous index $t$, spanning the interval $[0,1]$, as the depth diverges, that is $L\to\infty$.\\
More precisely, for all $L\geq 1$ and all $l\in[0:L]$, we can define $t(l,L) = l/L$.\\
Consider a sequence $\{l_n,L_n\}_{n\in\mathbb N}$ (where, for all $n$, $L_n\geq 1$ and $l\in[0:L_n]$), such that $L_n$ diverges but $l_n/L_n$ converges to a finite $t=\lim_{n\to\infty}t(l_n,L_n)$. We will show in this section (Lemma \ref{convcont}) that the kernels $Q_{l_n|L_n}$ (covariance kernel of the layer $l_n$ in a net with $L_n+1$ layers) converge uniformly to a kernel, $q_t$, on $K$. \\
Moreover we can define a differential problem for the mapping $t\mapsto q_t$, with $q\in[0,1]$, that is
\begin{align*}
\tag{\ref{recc}}
\begin{split}
&\dot q_t(x,x') = {\sigma_b}^2 + \tfrac{{\sigma_w}^2}{2}\left(1+\tfrac{f(c_t(x,x'))}{c_t(x,x')}\right)q_t(x,x')\,,\\
&q_0(x,x') = \sbq + \swq\,\tfrac{x\cdot x'}{d}\,,\\
&c_t(x,x')=\tfrac{q_t(x,x')}{\sqrt{q_t(x,x)q_t(x',x')}}\,.
\end{split}
\end{align*}

\begin{lemma}[Existence and uniqueness]\label{ex}
For any $x,x'$ in $K$, the solution of \eqref{recc} is unique and well defined for all $t\in[0,1]$. The maps $(x,x')\mapsto q_t(x,x')$ and $(x,x')\mapsto c_t(x,x')$ are Lipschitz continuous on $K^2$ and $c_t$ takes values in $[-1,1]$. Moreover, both $q_t$ and $c_t$ are kernels in the sense of Definition \ref{def:kernel}.
\end{lemma}
\begin{proof}
First notice that from \eqref{recc} we can find, with few algebraic manipulations, an explicit recurrence relation for the correlation $C_l$, defined in \eqref{defc}. For any $x,x'\in K$ we have
\begin{align}\label{recurrence_C}
    \begin{split}
    &C_{l+1}(x,x') =  A_{l+1}(x,x')\,C_l(x,x')+\tfrac{\swq}{2L}\left(1+\tfrac{\swq}{2L}\right)^{-1}A_{l+1}(x,x')\,f(c_l(x,x'))+\tfrac{1}{L}\tfrac{\sbq}{\sqrt{Q_l(x,x)Q_l(x',x')}}\,;\\
    &A_l(x,x') = \sqrt{\left(1-\tfrac{1}{L}\tfrac{\sbq}{Q_l(x,x)}\right)\left(1-\tfrac{1}{L}\tfrac{\sbq}{Q_l(x',x')}\right)}\,.
    \end{split}
\end{align}
We can find a Cauchy problem for the correlation directly from \eqref{recc} or by noting that $A_l(x,x')=1-\tfrac{\sbq}{2L}\left(\tfrac{1}{Q_l(x,x)}+\tfrac{1}{Q_l(x',x')}\right) + o(1/L)$, for $L\to\infty$. With both approaches, we have
\begin{align}\label{eqc}
\begin{split}
&\dot c_t(x,x')= {\sigma_b}^2\left(\Gt(x,x') - \At(x,x')\,c_t(x,x')\right) + \frac{\swq}{2}\,f(c_t(x,x'))\,,\\
&c_0(x,x') = \frac{\sbq + \swq\,x\cdot x'}{\sqrt{(\sbq + \swq\,\|x\|^2)(\sbq + \swq\,\|x'\|^2)}}\,,
\end{split}
\end{align}
where $f$ is defined in $\eqref{deff}$ and
\begin{align*}
\At(x,x') = \frac{1}{2}\left(\frac{1}{q_t(x,x)}+\frac{1}{q_t(x',x')}\right)\,;\qquad\Gt(x,x') = \sqrt{\frac{1}{q_t(x,x)\,q_t(x',x')}}\,.
\end{align*}
Note that for the diagonal terms $q_t(x,x)$, \eqref{recc} reduces to $\dot q_t ={\sigma_b}^2 + \frac{{\sigma_w}^2}{2}\,q_t$, whose solution is 
\begin{align*}
q_t(x,x) = e^{\frac{{\sigma_w}^2}{2}\,t}\,q_0(x,x)+\frac{2{\sigma_b}^2}{{\sigma_w}^2}\left(e^{\frac{{\sigma_w}^2}{2}\,t}-1\right)=e^{\frac{{\sigma_w}^2}{2}\,t}\,(\sbq + \swq\,\|x\|^2)+\frac{2{\sigma_b}^2}{{\sigma_w}^2}\left(e^{\frac{{\sigma_w}^2}{2}\,t}-1\right)\,.
\end{align*}
Now, fix $z=(x,x')\in K^2$ and let $\gamma_0 = c_0(z)\in[-1,1]$. Consider $\bar f:\mathbb{R}\to \mathbb{R}$, an arbitrary Lipschitz extension of $f$ to the whole $\mathbb{R}$ and define $H:[0,\infty)\times \mathbb{R}\to \mathbb{R}$ as
\begin{align*}
H(t,\gamma) = \sbq(\Gt(z) - \At(z)\,\gamma) + \frac{\swq}{2}\,\bar f(\gamma)\,. 
\end{align*}
$H$ is Lipschitz continuous in $\gamma$ and $C^\infty$ in $t$, so there exists $\tau>0$ such that the Cauchy problem
\begin{align*}
&\dot \gamma (t) = H(t,\gamma(t))\,;\\
&\gamma(0) = \gamma_0
\end{align*}
has a unique $C^1$ solution defined for  $t\in[0,\tau)$. \\
Noticing that 
\begin{align*}
\Gt(x,x') -\At(x,x') = -\frac{1}{2}\left(\frac{1}{q_t(x,x)}-\frac{1}{q_t(x',x')}\right)^2\leq 0\,,
\end{align*}
we get that for all $t_1$ such that $\gamma(t_1)=1$ we have $\dot\gamma(t_1)\leq 0$, since $f(1)=0$, and for all $t_{-1}$ such that $\gamma(t_{-1})=-1$ we have $\dot\gamma(t_{-1}) = \sbq(\Gt(x,x')+\At(x,x'))+\tfrac{\swq}{2}>0$. As a consequence $\gamma(t)\in[-1,1]$ for all $t\in[0,\tau)$ and we can take $\tau=\infty$.\\
In particular we get that \eqref{eqc} has a unique solution $t\mapsto c_t(z)$, defined for $t\in[0,1]$ and bounded in $[-1,1]$. \\
As a consequence, \eqref{recc} has a unique and well defined solution for all $t\geq 0$. \\

Now notice that $z\mapsto c_0(z)$ is Lipschitz on $K^2$. let us denote as $L_0$ a Lipschitz constant for $c_0$.\\
Since both $\Gt$ and $\At$ are $C^1$, we can find real constants $L_G$, $L_A$ and $M_A$ such that for all $z,z'$ elements of  $K^2$
\begin{align*}
&|\Gt(z)-\Gt(z')|\leq L_G\,\|z-z'\|\,;\\
&|\At(z)-\At(z')|\leq L_A\,\|z-z'\|\,;\\
&|\At(z)|\leq M_A\,.
\end{align*}
Let $L_f$ be a Lipschitz constant for $f$. Using the fact that $|c_t|\leq 1$, we can write
\begin{align*}
|\dot c_t(z) - \dot c_t(z')| \leq L_1\,\|z-z'\|+L_2\,|c_t(z)-c_t(z')|\,,
\end{align*}
where $L_1 = \sbq(L_G+L_A)$ and $L_2 = \sbq M_A+\frac{\swq}{2}\,L_f$.\\
Now fix $z$ and $z'$ and consider $\Delta(t) = c_t(z)-c_t(z')$. We have 
\begin{align*}
&|\dot\Delta(t)| \leq L_1\,\|z-z'\| + L_2\,|\Delta(t)|\,;\\
&|\Delta(0)| \leq L_0\,\|z-z'\| \,.
\end{align*} 
So $|\Delta(t)|\leq \left(\frac{L_1}{L_2}\,\left(e^{L_2\,t}-1\right)+L_0\,e^{L_2\,t}\right)\|z-z'\|$, meaning that $c_t$ (and so $q_t$) is Lipschitz on $L^2$.\\

Since the mapping $(x,x')\mapsto q_t(x,x')$ is continuous, it defines a compact integral operator $T(q_t)$ on $L^2(K)$ \citep{lang2012real}. Since $q_t$ is real and symmetric under the swap of $x$ and $x'$, the operator is self-adjoint. The same holds true for $c_t$. 

The fact that  $T(q_t)$ is a non-negative operator can be seen as a corollary of Lemma \ref{convcont}. Indeed all $T(Q_{l|L})$ is a non-negative definite operator, since it is induced by a kernel. Hence, for each $t\in[0,1]$ it is enough to find a sequence $\{l_n,L_n\}_{n\in\mathbb N}$ (where $L_n\geq 1$ is an integer and $l_n\in [0:L_n]$) such that $L_n\to\infty$ and $l_n/L_n\to t$. By Lemma \ref{convcont}, $T(Q_{l_n|L_n})\to T(q_t)$ in the $L^\infty$ norm, and hence in $L^2$, as we are on a compact set. By Lemma \ref{lemma:Ql_is_kernel}, for all $n\in\mathbb N$ we have that $T(Q_{l_n|L_n})$ is non-negative definite. Since the subspace of non-negative definite operators in $L^2$  is closed wrt the $L^2$ operator norm, we conclude.\\
Once we have established that $T(q_t)$ is non-negative definite, it follows immediately that $T(c_t)$ is non-negative as well. Since these results hold for any arbitrary finite Borel measure $\mu$ on $K$, we can thus conclude by Lemma \ref{T_mu(Q)} that both $q_t$ and $c_t$ are kernels, in the sense of Definition \ref{def:kernel}.
\end{proof}
\begin{corollary}\label{corex}
The maps $t\mapsto T(q_t)$ and $t\mapsto T(c_t)$, defined on $[0,1]$, are continuous and twice differentiable with respect to the operator norm in $L^2(K)$. Moreover, $\tfrac{\dd}{\dd t}T(q_t) = T(\dot q_t)$, $\tfrac{\dd}{\dd t}T(c_t) = T(\dot c_t)$, $\tfrac{\dd^2}{\dd t^2}T(q_t) = T(\ddot q_t)$ and $\tfrac{\dd^2}{\dd t^2}T(c_t) = T(\ddot c_t)$.
\end{corollary}
\begin{proof}
Consider the map $(t,z)\mapsto q_t(z)$, defined on $[0,1]\times K^2$, which is continuous wrt $z$ and $C^2$ wrt $t$, as it can be easily checked. Since $K^2$ and $[0,1]$ are compact sets, it follows that for any $t$
\begin{align*}
\lim_{s\to t} \sup_{z\in I^2}\left|\frac{q_{s}(z)-q_t(z)}{s-t} -\dot q_t (z)\right|= \sup_{z\in I^2}\lim_{s\to t}\left|\frac{q_{s}(z)-q_t(z)}{s-t}-\dot q_t(z)\right| = 0\,.
\end{align*}
Hence $\lim_{s\to t}\frac{q_s-q_t}{t-s}=\dot q_t$ uniformly on $K^2$, and hence $\lim_{s\to t}\frac{T(q_s)-T(q_t)}{t-s}=T(\dot q_t)$ in the $L^2(K,\mu)$ norm for operators, since $K$ is compact. \\
The proof for the second derivative works in the same way, using the fact that $(t,z)\mapsto q_t(z)$ is continuous in $z$ and $C^1$ in $t$.\\
As a consequence of the above results, $t\mapsto T(q_t)$ is continuous and twice differentiable, with $\tfrac{\dd}{\dd t}T(q_t) = T(\dot q_t)$ and $\tfrac{\dd^2}{\dd t^2}T(q_t) = T(\ddot q_t)$.\\ The proof for $T(c_t)$ is analogous.
\end{proof}
\begin{manuallemma}{\ref{convcont}}[Convergence to the continuous limit]
Let $Q_{l|L}$ be the covariance kernel of the layer $l$ in a net of $L+1$ layers $[0:L]$, and $q_t$ be the solution of \eqref{recc}, then
\begin{align*}
    \lim_{L\to\infty}\sup_{l\in[0:L]}\sup_{(x,x')\in K^2}|Q_{l|L}(x,x')-q_{t=l/L}(x,x')| = 0\,.
\end{align*}
\end{manuallemma}
\begin{proof}
We will show that the relation holds for $c_t$, and hence for $q_t$.\\
Let $H$, defined on $[0,1]\times K^2$, be such that $\dot c_t(z) = H(z,t,c_t(z))$. Explicitly, with the same notations as in \eqref{eqc}, we have 
\begin{align*}
    H(z,t,\gamma) = \sbq(\Gt(z) - \At(z)\,\gamma) + \frac{\swq}{2}\, f(\gamma)\,.
\end{align*}
Define
\begin{align*}
\tau(h) = \sup_{t,z}\left| \frac{c_{t+h}(z)-c_t(z)}{h}-H(z,t,c_t(z))\right|\,.
\end{align*}
Since $t$ and $z$ takes values on compact sets, by uniform continuity, fixed $h$ we can write, for $h\to 0$
\begin{align*}
\sup_t\sup_{s\in [t,t+h]} \left| H(z,s,c_s(z)) - H(z,t,c_t(z))\right|=o(h)\,.
\end{align*}
Hence, since $\tau$ can be rewritten as $\tau(h) = \frac{1}{h}\sup_{t,z}\left| \int_t^{t+h}(H(z,s,c_s(z)) - H(z,t,c_t(z))\,\dd s\right|$, it is clear that $\tau(h)\to 0$ for $h\to 0$. \\
Now, for any integer $L\geq 1$, let $\tilde H_L:K^2\times[0:L-1]\times[-1,1]$ be given by
\begin{align*}
\tilde {H}_L(z,l,\gamma) = (A_{l+1|L}(x,x')-1)\,L\,\gamma+\tfrac{\swq}{2}\left(1+\tfrac{\swq}{2L}\right)^{-1}A_{l+1|L}(x,x')\,f(c_l(x,x'))+\tfrac{\sbq}{\sqrt{Q_{l|L}(x,x)Q_{l|L}(x',x')}}\,,
\end{align*}
where,
\begin{align*}
    A_{l|L}(x,x') = \sqrt{\left(1-\tfrac{1}{L}\tfrac{\sbq}{Q_{l|L}(x,x)}\right)\left(1-\tfrac{1}{L}\tfrac{\sbq}{Q_{l|L}(x',x')}\right)}\,.
\end{align*}
It is clear from \eqref{recurrence_C} that $\tilde H_L$ has been defined so that  $C_{l+1|L}(z)-C_{l|L}(z)=\tfrac{1}{L}\tilde H_L(z,l,\gamma)$, for all $L\in[0:L-1]$ and all $z\in K^2$. Using the explicit form of the diagonal terms of $Q$ and $q$, it can be easily shown that, for $L\to\infty$, \begin{align*}
    &\sup_{(x,x')\in K^2}\sup_{l\in[0:L-1]}A_{l+1|L}(x,x')=1+\tfrac{\sbq}{L}\mathcal{A}_{t=l/L}(x,x')+O(1/L^2)\,;\\
    &\sup_{(x,x')\in K^2}\sup_{l\in[0:L]}\tfrac{\sbq}{\sqrt{Q_{l|L}(x,x)Q_{l|L}(x',x')}} = \mathcal{G}_{t=l/L}(x,x')+O(1/L^2)\,,
\end{align*}
where $\mathcal{A}_t$ and $\mathcal{G}_t$ are defined as in \eqref{eqc}. As a consequence, we can find a constant $M_1>0$ and an integer $L_\star>0$ such that, for all $\gamma\in[-1,1]$, for all $z\in K^2$, for all $L\geq L^\star$
\begin{align}\label{eq1:proof_conv}
    |\tilde H_L(z,l,\gamma)-H(z,l/L,\gamma)|\leq \frac{M_1}{L}\,.
\end{align}
Moreover, there exists a constant $M_2>0$ such that for all $z\in K^2$, all $t\in[0,1]$ and all pairs $(\gamma,\gamma')\in[-1,1]^2$
\begin{align}\label{eq2:proof_conv}
|H(z,t,\gamma)-H(z,t,\gamma')|\leq M_2\|\gamma-\gamma'\|\,.
\end{align}
Thanks to the two above uniform inequalities, we will now show that, for $L\geq L_\star$,
\begin{align}\label{goal}
\sup_{l\in[0:L]}\sup_{z\in K^2}|C_{l|L}(x,x')-c_{t(l,L)}(x,x')| \leq \tilde\tau(1/L) \frac{e^{M_2}-1}{M_2}\,,
\end{align}
where $\tilde\tau:h\mapsto \tau(h) + M_1h$.\\
To do so, fix $L\geq L_\star$ and define $\Delta_{l}=\sup_{z\in K^2}|C_{l|L}(x,x')-c_{t(l,L)}(x,x')|$. Using the definition of $\tau$, \eqref{eq1:proof_conv} and \eqref{eq2:proof_conv} we get
\begin{align*}
    |\Delta_{l+1}|\leq\left(1+\tfrac{M_2}{L}\right)|\Delta_l|+\tfrac{1}{L}\tau(1/L) + \tfrac{M_1}{L} = \left(1+\tfrac{M_2}{L}\right)|\Delta_l|+\tfrac{1}{L}\tilde\tau(1/L)\,.
\end{align*}
At this point, using the fact that $\Delta_0=0$, it is easy to show by induction that 
\begin{align*}
    \Delta_l\leq \tilde\tau(1/L)\,\frac{\left(1+\tfrac{M_2}{L}\right)^l-1}{M_2}\,,
\end{align*}
and so \eqref{goal} follows.\\
Finally, the uniform convergence of $C$ to $c$ implies the one of $Q$ to $q$ and so we conclude.
\end{proof}
\subsection{Universality of the covariance kernel}
We will now prove the results of universality of Theorem \ref{thm:expr} and Proposition \ref{prop:uniform_Sd}.
\subsection*{Proof of Theorem \ref{thm:expr}}
The idea is to prove that for any finite Borel measure $\mu$ on $K$, the operator $T_\mu(q_t)$ is strictly positive definite if $t>0$, and then use the characterization of universal kernels given in Lemma \ref{univ=Lp_spd}.\\
To prove the strict positive definiteness, we will proceed in two steps. First we show in Proposition \ref{prop:t_phi} that for all non-zero $\f\in L^2(K,\mu)$, $\langle T_\mu(q_t)\,\f,\f\rangle>0$ for $t$ small enough. Then we use Proposition \ref{prop:dotq>=0}, which shows that $\tfrac{\dd}{\dd t}T_\mu(q_t)$ is non-negative definite.
\begin{prop}\label{prop:t_phi}
Fix any finite Borel measure $\mu$ on $K$, and assume that $\sigma_b>0$. Given any non-zero $\f\in L^2(K,\mu)$, there exists a $t_\f \in (0,1]$ such that $\langle T_\mu(q_t)\, \f,\f\rangle > 0$, for all $t\in (0,t_\f)$.
\end{prop}
\begin{proof}
From Corollary \ref{corex}, we can expand $T_\mu(q_t)$ around $t=0$ as 
\begin{align*}
    T_\mu(q_t) = T_\mu(q_0) +  t\, T_\mu(\dot q_0) + o(t) = t \,T_\mu\left(\sbq + \frac{\swq}{2} q_0\right) + T_\mu((c_0 + t f(c_0))R_0) + o(t) \,,
\end{align*}
the $o(t)$ being wrt the operator norm, where we have defined the kernel $R_0$ via $R_0(x,x') =\tfrac{\swq}{2}\sqrt{(1+\zeta \|x\|^2)(1+\zeta \|x'\|^2)}$.\\
Since $T_\mu(q_0)$ is non-negative, for any $\f\in L^2(I)$, we have 
\begin{align*}
    \langle T_\mu(q_t)\,\f,\f\rangle \geq \langle T_\mu((c_0 + t f(c_0))R_0)\,\f,\f\rangle + o(t) =\left(1-\frac{t}{2}\right)\langle T_\mu(c_0)\,\psi,\psi\rangle + t\,\langle T_\mu(f(c_0)))\,\psi,\psi\rangle + o(t)\,,
\end{align*}
where $\psi(x) = \sigma_w\sqrt{(1+\zeta \|x\|^2)/2}\,\f(x)$. We conclude by the strict positivity of $\tilde f(c_0)$ on $L^2(K,\mu)$, thanks to Proposition \ref{prop:univ_K} and Lemma \ref{univ=Lp_spd}. 
\end{proof}
\begin{prop}\label{prop:dotq>=0}
For any finite Borel measure $\mu$ on $K$, for any $t\in[0,1]$, the operator $T_\mu(\dot q_t)$ on $L^2(K,\mu)$ is non-negative definite. In particular, for all $\f\in L^2(K,\mu)$ we have
\begin{align*}
    \tfrac{\dd}{\dd t}\langle T_\mu(q_t)\,\f,\f\rangle \geq 0\,.
\end{align*}
\end{prop}
\begin{proof}

Fix $\mu$ and $\f\in L^2(K,\mu)$. From \eqref{recc} we can write
\begin{align*}
T_\mu(\dot q_t) = T_\mu\left(\sbq +\frac{\swq}{2}\,q_t +\frac{\swq}{2}\,\frac{f(c_t)}{c_t} q_t\right).
\end{align*}

By Lemma \ref{ex}, $T_\mu(q_t)$ is non-negative definite, so we can write
\begin{align*}
\langle T_\mu(\dot q_t)\,\f,\f\rangle &= \sbq |\langle 1,\f\rangle|^2 + \frac{\swq}{2}\left\langle T_\mu\left(\frac{c_t + f(c_t)}{c_t} \,q_t\right)\f,\f\right\rangle\\
&\geq\frac{\swq}{2}\left\langle T_\mu\left(\tilde f(c_t)\,\frac{q_t}{c_t}\right)\f,\f\right\rangle\\
&=\frac{\sbq}{2}\langle T_\mu(\tilde f(c_t))\,\psi,\psi\rangle\,,
\end{align*}
where $\tilde f:\gamma\mapsto \tfrac{\gamma}{2}+f(\gamma)$, for $\gamma\in [-1,1]$, and $\psi(x) = \sqrt{q_t(x,x)}\,\f(x)$. By Lemma \ref{lemma:analytic_decomp_f_hat}, the Taylor expansion of $\tilde f$ around $0$ converges uniformly on $[-1,1]$, and all its coefficients are non-negative. We conclude by Lemma \ref{schur} that $T_\mu(\dot q_t)$ is non-negative definite.\\
Finally, to prove the inequality, it is enough to recall that $\tfrac{\dd}{\dd t}T_\mu(q_t) = T_\mu(\dot q_t)$ by Corollary \ref{corex}, the derivative $\tfrac{\dd}{\dd t}$ being wrt the operator norm on $L^2(K,\mu)$. 
\end{proof}
\begin{manualthm}{\ref{thm:expr}}
[Universality of $q_t$]
Let $K \subset \R^d$ be compact and assume $\sigma_b>0$. For any $t\in (0,1]$, the solution  $q_t$ of \eqref{recc} is a universal kernel on $K$.
\end{manualthm}
\begin{proof}
By Lemma \ref{univ=Lp_spd}, it suffices to show that for any finite Borel measure $\mu$ on $K$, $T_\mu(q_t)$ is strictly positive definite for all $t\in(0,1]$. Fix any nonzero $\f\in L^2(K,\mu)$, define the map $F$ on $[0,1]$ by $F(t) = \langle T_\mu(q_t)\,\f,\f\rangle$. For any fixed $t\in(0,1]$, by Proposition \ref{prop:t_phi} we can find $s\in(0,t)$ such that $F(s)>0$. Since $F$ is non decreasing by Proposition \ref{prop:dotq>=0}, we get that $F_t>0$. Hence $T_\mu(q_t)$ is strictly positive definite.
\end{proof}
\subsection*{Proof of Proposition \ref{prop:uniform_Sd}}
The proof of Proposition \ref{prop:uniform_Sd} is quite similar to the one of Theorem \ref{thm:expr}. \\
Using Lemma \ref{lemma:Sd_spd=>univ} instead of Lemma \ref{univ=Lp_spd}, we will not need to consider a generic finite Borel measure $\mu$ on $\Sd$, but it will be enough to show that $T_{\nu}(q_t)$ is a striclty positive operator on $L^2(\Sd,\nu)$, where $\nu$ is the standard unifrom spherical measure on $\Sd$.\\
Since $\sigma_b=0$, we will not be able to use Proposition \ref{prop:univ_K}. We will hence state some preliminary results. 
\begin{lemma}\label{lem10}
Let $\{A_n\}_{n\in\mathbb{N}}$ be a family of compact non-negative operators on a separable Hilbert space $\mathcal{H}$. Let $R_n$ be the range of $A_n$ and assume that $V = \Span(\bigcup_{n\in\mathbb N} R_n)$ is dense in $\mathcal{H}$. Let $\{\al_n\}_{n\in\mathbb{N}}$ be a strictly positive sequence such that the sum
\begin{align*}
A = \sum_{n\in\mathbb{N}} \alpha_n\,A_n
\end{align*}
converges in the operator norm. Then $A$ is a compact strictly positive definite operator. 
\end{lemma}
\begin{proof}
$A$ is the convergent limit of a sum of compact self-adjoint operators and hence it is compact and self-adjoint. Now, fix an arbitrary nonzero $h\in\mathcal{H}$. To show that $A$ is strictly positive it is enough to prove that $\langle A\,h,h\rangle>0$.\\
Denote by $V_N$ the linear span of $\bigcup_{n\in[0:N]}R_n$. Since $V_{N}\subseteq V_{N+1}$ for all $N$, and  
$\bigcup_{N\in\mathbb{N}}V_N= V$ is dense in $H$, there exists a sequence $\{h_N\}_{N\in\mathbb N}$ converging to $h$ and such that $h_N\in V_N$ for all $N$. \\
Now let us show that there must exist $n^\star\in\mathbb N$ such that $A_{n^\star}\,h\neq 0$. Since $\lim_{N\to\infty}\langle h,h_N\rangle = \langle h,h\rangle>0$, there must be a $N^\star$ such that $\langle h,h_{N^\star}\rangle>0$ and so there exists $n^\star\in[0:N^\star]$ and $h_{n^\star}\in V_{n^\star}$ such that $\langle h, h_{n^\star}\rangle\neq 0$. In particular, $h$ is not orthogonal to $R_{n^\star}$ and can not lie in the nullspace of $A_{n^\star}$, using the fact that $A_{n^\star}$ is compact and self-adjoint and so its range and its nullspace are orthogonal \citep{lang2012real}. \\
Using the spectral decomposition of non-negative compact operators, it is straightforward that $A_{n^\star}\,h \neq 0$ implies that $\langle A_{n^\star}\,h,h\rangle > 0$. Now, since $A_n$ is non-negative and $\al_n>0$ for all $n$, we have
\begin{align*}
    \langle A\,h,h\rangle = \sum_{n\in\mathbb N}\al_n\langle A_n\,h,h\rangle \geq \al_{n^\star} \langle A_{n^\star}\,h,h\rangle > 0\,,
\end{align*}
and so we conclude.
\end{proof}
\begin{lemma}\label{density_sphere}
For all $n\in\mathbb{N}$, consider the kernel $p_n$ on $\Sd$, defined by $p_n(x,x') = (x\cdot x')^n$, and let $T_{\nu}(p_n)$ be the induced integral operator on $L^2(\Sd,\nu)$. Denoting as $R_n$ the range of $T_\nu(p_n)$, the subspace $V = \Span\left(\bigcup_{n\in\mathbb N}R_n\right)$ is dense in $L^2(\Sd,{\nu})$.\\
Moreover, letting $V' = \Span\left(\bigcup_{n\in\mathbb N}R_{2n}\right)$ and $V'' = \Span\left(\bigcup_{n\in\mathbb N}R_{2n+1}\right)$, we have $L^2(\Sd,{\nu}) = \overline{V'}\oplus\overline{V''}$, the overline denoting the closure in $L^2(\Sd,{\nu})$. 
\end{lemma}
\begin{proof}
To prove that $V$ is dense, first notice that for each spherical harmonic $Y$, we can find an operator in the form $T_\nu(P(x\cdot x'))$, for a polynomial $P$, which has $Y$ in its range. Since the range of such an operator is trivially contained in $V$, it follows that $V$ contains all the spherical harmonics, and so it is dense in $L^2(\Sd,\nu)$. \\
Now, note that for any even $n$ and odd $n'$ we have 
\begin{align*}
    \int_{\Sd}(x\cdot z)^n(z\cdot x')^{n'}\dd \nu(z) = 0\,,
\end{align*}
by an elementary symmetry argument, since it is the integral on the sphere of a homogeneous polynomial of odd degree $n+n'$ in the components $z_i$'s of $z$.\\
It follows that $V'$ and $V''$ are orthogonal. Since their union $V$ is dense, we conclude that $L^2(\Sd,\nu)=\overline{V'}\oplus\overline{V''}$.
\end{proof}
\begin{corollary}\label{op_dec}
With the notations of Lemma \ref{density_sphere}, assume that a sequence $\{\al_{n\in\mathbb N}\}$ is such that 
$A = \sum_{n\in\mathbb{N}}\al_n\,T_\nu(p_n)$
converges wrt the operator norm on $L^2(\Sd,\nu)$. Then $A = A'+A''$, where $A':\overline{V'}\to\overline{V'}$ and $A'':\overline{V''}\to\overline{V''}$. Such a decomposition is unique and 
\begin{align*}
    &A' = \sum_{n\in\mathbb{N}}\al_{2n}\,T_\nu(p_{2n})\,;
    &A'' = \sum_{n\in\mathbb N}\al_{2n+1}\,T_\nu(p_{2n+1})\,,
\end{align*}
both sums converging wrt the operator norm.
\end{corollary}
\begin{proof}
It is clear that $A=A'+A''$, when both $A'$ and $A''$ are defind on the whole $L^2(\Sd,\nu)$.\\
Consider any $\f\in L^2(\Sd,\nu)$. We have $A'\f \in \overline{V'}$, since $T_\nu(p_{2n})\f\in\overline{V'}$ for all $n$. Analogously, we can show that $A''\f\in\overline{V''}$. To conclude that we can consider the restrictions of $A'$ and $A''$ to $\overline{V'}$ and $\overline{V''}$ respectively, it is enough to recall that for compact self adjoint operators the nullspace is the orthogonal of the closure of the range \citep{lang2012real}, so that the nullspace of $A'$ contains $\overline{V''}$ and the nullspace of $A''$ contains $\overline{V'}$. 
\end{proof}
\begin{lemma}\label{propf}
The function $f:[-1,1]\to\mathbb{R}$, defined in \eqref{deff}, is an analytic function on $(-1,1)$, whose expansion $f(\gamma) = \sum_{n\in\mathbb{N}}\al_n\,\gamma^n$
converges absolutely on $[-1,1]$. Moreover, $\al_n>0$ for all even $n\in\mathbb{N}$, $\al_1=-1/2$ and $\al_n=0$ for all odd $n\geq 3$.\\
Let $g:[-1,1]\to\mathbb{R}$ be defined as $g(\gamma) = f(\gamma)f'(\gamma)$. $g$ is analytic on $(-1,1)$ and its expansion $g(\gamma) = \sum_{n\in\mathbb{N}}\be_n\,\gamma^n$ converges absolutely on $[-1,1]$. Moreover, for all odd $n\in\mathbb N$ the coefficient $\be_n$ is strictly positive.
\end{lemma}
\begin{proof}
The claims for $f$ have been already proven in Lemma \ref{lemma:analytic_decomp_f_hat}. As for $g$, the analyticity of $f$ implies the one of $f'$, and it is easy to check the convergence on $[-1,1]$. Moreover, all the odd Taylor coefficients of $f'$ are striclty positive, as the even coefficients of $f$ are. It follows that $\beta_n>0$ for all odd $n$.
\end{proof}
\begin{prop}\label{prop:t_phi_Sd}
Given any non-zero $\f\in L^2(\Sd,\nu)$, there exists a $t_\f \in (0,1]$ such that $\langle T_\nu(q_t)\, \f,\f\rangle > 0$, for all $t\in (0,t_\f)$. 
\end{prop}
\begin{proof}
The case $\sigma_b>0$ has been already established in Proposition \ref{prop:t_phi}, hence suppose that $\sigma_b=0$.\\
First recall \eqref{eqc}
\begin{align}\label{cdt}
     \dot c_t = \frac{\swq}{2}\,f(c_t)\,.
\end{align}
Deriving once more we have
\begin{align}\label{cddt}
    \ddot c_t = g(c_t)\,,
\end{align}
where $g=ff'$ as in Lemma \ref{propf}.\\
Define the kernels $p_n$'s, and the subspaces $V'$ and $V''$ of $L^2(\Sd,\nu)$, as in Lemma \ref{density_sphere}. 
By \eqref{cdt} and \eqref{cddt} we can write
\begin{align*}
    c_t = c_0 + t\,\dot c_0 + \tfrac{t^2}{2}\,\ddot c_0 + o(t^2) = c_0 + t\,f(c_0) + \tfrac{t^2}{2}\,g(c_0) + o(t^2)\,.
\end{align*}
Since $\sigma_b = 0$, we have that $c_0(x,x') = x\cdot x'$, so that $c_0=p_1$. \\
From Lemma \ref{propf}, $T_\nu(\dot c_0)=\sum_{n\in\mathbb N}\al_n\,T_\nu(p_n)$ and $T_\nu(\ddot c_0)=\sum_{n\in\mathbb N}\be_n\,T_\nu(p_n)$, both sums converging in the operator norm. Moreover, $\al_n>0$ for all even $n$ and $\al_n=0$ for all odd $n\geq 3$, whilst $\be_n>0$ for all odd $n$.\\
In particular, by Corollary \ref{op_dec} and Lemma \ref{lem10}, we deduce that the restriction of $T_\nu(\dot c_0)|_{\overline{V'}}:\overline{V'}\to\overline{V'}$ is well defined and strictly positive, and the same holds true for the restriction $T_\nu(\ddot c_0)|_{\overline{V''}}:\overline{V''}\to\overline{V''}$. \\
Now fix a non-zero $\f\in L^2(\Sd,\nu)$. By Lemma \ref{density_sphere}, we can write $\f = \f'+\f''$, with $\f'\in \overline{V'}$, $\f''\in \overline{V''}$ uniquely determined. \\
First, suppose that $\f'\neq 0$. Using Corollary \ref{corex} and recalling that $c_0=p_1$, we get
\begin{align*}
    \langle T_\nu(c_t)\,\f,\f\rangle = t\langle T_\nu(\dot c_0)|_{\overline{V'}}\,\f',\f'\rangle+
    \langle (1+t\,\al_1)T_\nu(p_1)\,\f'',\f''\rangle+o(t)>0\,,
\end{align*}
for $t$ small enough.\\
On the other hand, for $\f'=0$, we have $\f=\f''$ and so
\begin{align*}
    \langle T_\nu(c_t)\,\f,\f\rangle = \langle (1+t\,\al_1)T_\nu(p_1)\,\f'',\f''\rangle + \tfrac{t^2}{2}\,\langle T_\nu(\ddot c_0)|_{\overline{V''}}\,\f'',\f''\rangle+o(t^2)>0
\end{align*}
for $t$ small enough.\\
So there is a $t_\f$ such that, for $t\in(0,t_\f)$, $\langle T_\nu(c_t)\,\f,\f\rangle>0$. It follows immediately that the same property is true for $T_\nu(q_t)$.
\end{proof}
\begin{lemma}\label{lemma:Sd_spd=>univ}
Let $Q$ be a kernel on $\Sd$. Then $Q$ is universal on $\Sd$ if and only if $T_\nu(Q)$ is strictly positive definite on $L^2(\Sd,\nu)$.
\end{lemma}
\begin{proof}
If $Q$ is universal, $T_\nu(Q)$ is strictly positive definite by Lemma \ref{univ=Lp_spd}. On the other hand, if $T_\nu(Q)$ is strictly positive definite, by Proposition \ref{proposition:spectral_decomposition_on_sphere} its range contains all the spherical harmonics. Since the RKHS generated by $Q$ contains the range of $T_\nu(Q)$ (Proposition 11.17 in \citep{paulsen2016RKHS}), it contains the linear span of the spherical harmonics, which is dense in $C(\Sd)$ \citep{kounchev2001multivariate}. Hence $Q$ is universal.
\end{proof}
\begin{manualprop}{\ref{prop:uniform_Sd}}
[Universality on $\Sd$]
For any $t\in (0,1]$, the covariance kernel $q_t$, solution of \eqref{recc} with $\sigma_b=0$, is universal on $\Sd$, with $d\geq 2$.
\end{manualprop}
\begin{proof}
Proceeding as in the proof of Theorem \ref{thm:expr}, using Proposition \ref{prop:dotq>=0} and Proposition \ref{prop:t_phi_Sd} we can show that $T_\nu(q_t)$ is strictly posititive definite on $L^2(\Sd,\nu)$ for all $t\in(0,1]$. We conclude by Lemma \ref{lemma:Sd_spd=>univ} that $q_t$ is universal on $\Sd$.
\end{proof}

\section{Stable ResNet with decreasing scaling}\label{app:section:decreasing_scaling}
\subsection{Proof of Proposition \ref{prop:uniform_convergence_decreasing_scaling}}
\begin{manualprop}{\ref{prop:uniform_convergence_decreasing_scaling}}[Uniform Convergence of the Kernel]
Consider a Stable ResNet with a decreasing scaling, i.e. the sequence $\{\lambda_l\}_{l\geq1}$ is such that $\sum_l \lambda_l^2 < \infty$. Then for all $(\sigma_b,\sigma_w)\in \mathbb{R}^+ \times (\mathbb{R}^+)^*$, there exists a kernel $Q_{\infty}$ on $\mathbb{R}^d$ such that for any compact set $K \subset \mathbb{R}^d$, 
\begin{align*}
    \sup_{x,x' \in K} |Q_{L}(x,x') - Q_{\infty}(x,x')| = \Theta\big(\textstyle\sum_{k \geq L} \lambda_k^2\big)\,.
\end{align*}
\end{manualprop}

\begin{proof}
Let $x,x' \in \mathbb{R}^d$. The kernel $Q_l$ is given recursively by the formula
\begin{equation*}
Q_l(x,x') =  Q_{l-1}(x,x') + \lambda_l^2 \sigma_b^2 + \frac{\sigma_w^2 \lambda_l^2}{2} \hat{f}(C_{l-1}(x,x')) \sqrt{Q_{l-1}(x,x)}\sqrt{Q_{l-1}(x',x')}\,,
\end{equation*}
where $\hat{f}(t) = 2 \mathbb{E}[\phi'(Z_1)\phi'(t Z_1 + \sqrt{1-t^2} Z_2)] = t + f(t)$ and $Z_1,Z_2$ are iid standard Gaussian variables. In particular, we have
$$
Q_l(x,x) = \lambda_l^2 \sigma_b^2 + (1 + \frac{\sigma_w^2 \lambda_l^2}{2}) Q_{l-1}(x,x)\,.
$$
which brings
$$Q_l(x,x) + \frac{2\sigma_b^2}{\sigma_w^2} = (1 + \frac{\sigma_w^2 \lambda_l^2}{2}) (Q_{l-1}(x,x) + \frac{2\sigma_b^2}{\sigma_w^2})\,,$$
Therefore, we can assume without loss of generality that $\sigma_b=0$. This yields
$$
C_l(x,x') = \frac{1}{1+\frac{\lambda_l \sigma_w^2}{2}}C_{l-1}(x,x') + \frac{\frac{\sigma_w^2 \lambda_l^2}{2}}{1+\frac{\lambda_l \sigma_w^2}{2}} \hat{f}(C^{l-1}(x,x'))\,.
$$
Letting $\alpha_l = \frac{\sigma_w^2 \lambda_l^2}{2}$ and $C_l:=C_l(x,x')$, we have that 
$$
C_l = \frac{1}{1+\alpha_l} C_{l-1} + \frac{\alpha_l}{1+\alpha_l} \hat{f}(C_{l-1})\,.
$$
Since $\hat{f}$ is non decreasing, $C^l$ is non-decreasing and has a limit $C_{\infty}(x,x')\leq1$.

Now let us prove that the convergence of $C_l$ to $C_{\infty}$ happens uniformly with a rate $\sum_{k\geq l} \lambda_l^2$. Using the recursive formula of $C_l$, and knowing that we have that 
$$
C_{\infty} - C_l = \frac{1}{1 + \alpha_l} (C_{\infty} - C_{l-1}) + \frac{\alpha_l}{1 + \alpha_l} (C_{\infty} - f(C_{l-1}))\,.
$$
Letting $\delta_l = C_{\infty} - C_l$, it is easy to see that, uniformly in $x,x' \in \mathbb{R}^d$, we have that
$$
\delta_l = \delta_{l-1} + \alpha_l + o(\alpha_l)\,.
$$
Therefore, using the fact that $C_l \leq C_\infty$, we have
$$
\sup_{(x,x') \in \mathbb{R}^d} |C_l(x,x') - C_{\infty}(x,x')| = \mathcal{O}\left(\sum_{k\geq l} \alpha_k\,\right).
$$
Moreover, we know that 
$$
Q_l(x,x) = Q_0(x,x) \prod_{k=1}^l (1 + \alpha_k)\,,
$$
so that for any compact set $K \subset \mathbb{R}^d$ 
$$
\sup_{x \in K} |Q_{l}(x,x)-Q_{\infty}(x,x)| \sim \sum_{k\geq l} \alpha_k\,.
$$
Moreover, since $C_\infty(x,x') \geq C_l(x,x')$ and $Q_\infty(x,x) \geq Q_l(x,x)$ for all $x \in \mathbb{R}^d$, we can use the fact that 
\begin{equation*}
\begin{split}
    Q_{\infty}(x,x') - Q_{l}(x,x') &= \sqrt{Q_\infty(x,x)Q_\infty(x',x')}( C_\infty(x,x')-C_l(x,x'))\\
    &+ C_l(x,x') (\sqrt{Q_\infty(x,x)Q_\infty(x',x')} - \sqrt{Q_l(x,x)Q_l(x',x')})
\end{split}
\end{equation*}
and hence conclude.
\end{proof}

\subsection{Proof of Corollary \ref{cor:universality_decreasing_scaling}}

\begin{manualcor}{\ref{cor:universality_decreasing_scaling}}
The following statements hold\\
$\bullet$~Let $K$ be a compact set of $\mathbb{R}^d$ and assume $\sigma_b>0$. Then, $Q_\infty$ is universal on $K$.\\
$\bullet$~Assume $\sigma_b=0$. Then $Q_\infty$ is universal on $\Sd$.
\end{manualcor}

\begin{proof}
Corollary \ref{cor:universality_decreasing_scaling} is a direct result of Propositions \ref{prop:universality_on_a_compact}, \ref{prop:universality_on_sphere} and \ref{prop:rkhs_hierarchy}. Indeed, for any compact $K \subset \mathbb{R}^d$, $\mathcal{H}_{Q_L}(K) \subset \mathcal{H}_{Q_\infty}(K)$ for all $L \geq 0$. Therefore, the universality of $Q_L$ for some finite $L$ is sufficient to conclude that $Q_\infty$ is universal.
\end{proof}

\section{Neural Tangent Kernel}\label{app:section:NTK}
Throughout this section, we will consider ResNets with NTK parameterization \citep{jacot}. This simply means that all the components of the biases and the weights will be initialized as iid standard normal random variables. In order to compensate this change of parameterization, the propagation through the network needs to be slightly modified. Hence \eqref{equation:scaled_Resnet} will be replaced by
\begin{align}\label{eq:NTK_param}
\begin{split}
    &y_0(x) = \tfrac{\sigma_w}{\sqrt{d}}\,W_0\,x+\sigma_b\,B_0\,;\\
    &y_l(x) = y_{l-1}(x)+\lambda_{l,L}\,\tfrac{\sigma_w}{\sqrt{N_{l-1}}}\,W_l+\sigma_b\,B_l\,.
\end{split}
\end{align}
However, it is strightforward to verify that the recurrence \eqref{rec} for the covariance kernels keeps unchanged.\\
Clearly, the dynamics of a standard ResNet with NTK parameterization can be recovered from \eqref{eq:NTK_param} by setting $\lambda_{l+1,L}=1$ for all $l,L$.

The Neural Tangent Kernel, introduced by \citep{jacot}, is defined as
\begin{align*}
    \tilde{\Theta}^{ij}_L(x,x') = \nabla_{\text{par}}\, y^i_L(x)\cdot\nabla_{\text{par}}\, y^j_L(x')\,,
\end{align*}
where $\nabla_{\text{par}}$ denotes the gradient wrt the parameters of the network. \\
The NTK of a Stable ResNet can be evaluated recursively.
We will now prove the recurrence formula \eqref{recurrence_NTK}. The following result was proven in Lemma 3 in \citep{hayou_ntk} for the case of a standard ResNet without bias. We extend it to ResNet with bias.
\begin{lemma}[Recurrence relatino for the NTK]
For a Stable ResNet, the NTK can be evaluated recursively, layer by layer, as
\begin{align}\tag{\ref{recurrence_NTK}}
\Theta_0=Q_0\,;\qquad\Theta_{l+1} = \Theta_l + \lambda_{l+1,L}^2\left(\Psi_l + \Psi'_l\,\Theta_l\right)\,,
\end{align}
where $\Psi_l(x,x') = {\sbq} +  {\swq} \E[\phi(y^l_1(x))\phi(y^l_1(x'))]$ and $\Psi'_l(x,x') =\swq \E[\phi'(y^l_1(x))\phi'(y^l_1(x'))]$.
\end{lemma}
\begin{proof}
The first result is the same as in the FFNN case \citep{jacot}, since we assume there is no residual connections between the first layer and the input. Let $x,x' \in \mathbb{R}^d$. We have 
    $$
    \Theta_0(x,x') = \sum_{j=0}^d \frac{\partial y^1_0(x)}{\partial w^{1j}_0}\frac{\partial y^1_0(x)}{\partial w^{1j}_0}   +  \frac{\partial y^1_0(x)}{\partial b^1_0}\frac{\partial y^1_0(x)}{\partial b^1_0} = \frac{\sigma_w^2}{d} x\cdot x' + \sigma_b^2\,.
    $$
We prove the second result by induction. The proof is similar to the one of ResNet in \citep{hayou_ntk}. Let $\theta_k = (W_k,B_k)$. For $l\geq 1$ and $i \in [1:N_{l+1}]$
\begin{align*}
    \partial_{\theta_{0:l}} y_{l+1}^i(x) &= \partial_{\theta_{0:l}} y_{l}^i(x)  + \lambda_{l+1,L} \frac{\sigma_w}{\sqrt{N_{l}}} \sum_{j=1}^{N_{l}} W_{l+1}^{ij} \phi'(y_{l}^{j}(x)) \partial_{\theta_{1:l}} y_{l}^j(x)\,.
\end{align*}
Therefore, we obtain
\begin{align*}
    (\partial_{\theta_{0:l}} y_{l+1}^i(x))(\partial_{\theta_{0:l}} y_{l+1}^i(x'))^t &= (\partial_{\theta_{0:l}} y_l^i(x))(\partial_{\theta_{0:l}} y_l^i(x'))^t  \\
    &+ \lambda_{l+1,L}^2\frac{\sigma_w^2}{N_l} \sum_{j, j'}^{N_{l}} W_{l+1}^{ij} W_{l+1}^{ij'}\phi'(y_l^{j}(x)) \phi'(y_l^{j'}(x'))\partial_{\theta_{0:l}} y_l^j(x) (\partial_{\theta_{0:l}} y_l^{j'}(x'))^t +I\,,
\end{align*}
where 
$$I = \lambda_{l+1,L}\frac{\sigma_w}{\sqrt{N_l}} \sum_{j=1}^{N_l} W_{l+1}^{ij} (\phi'(y_l^{j}(x)) \partial_{\theta_{0:l}} y_l^i(x) (\partial_{\theta_{0:l}} y_l^j(x'))^t +\phi'(y_l^j(x')) \partial_{\theta_{0:l}} y_l^j(x) (\partial_{\theta_{0:l}} y_l^i(x'))^t )\,.$$
We prove the result by induction. Assume the result is true for layers $1,2, ..., l$ and let us prove it for $l+1$. Using the induction hypothesis, as $N_1, N_2, ..., N_{l-1} \rightarrow \infty$ recursively, we have that 
\begin{align*}
    & (\partial_{\theta_{0:l}} y_{l+1}^i(x))(\partial_{\theta_{0:l}} y_{l+1}^i(x'))^t  + \lambda_{l+1,L}^2 \frac{\sigma_w^2}{N_l} \sum_{j, j'}^{N_{l}} W_{l+1}^{ij} W_{l+1}^{ij'}\phi'(y_l^{j}(x)) \phi'(y_l^{j'}(x'))\partial_{\theta_{0:l}} y_l^j(x) (\partial_{\theta_{0:l}} y_l^{j'}(x'))^t
    +I \\
    &\rightarrow \Theta_l(x,x')  + \lambda_{l+1,L}^2 \frac{\sigma_w^2}{N_l} \sum_{j}^{N_{l}} (W_{l+1}^{ij})^2 \phi'(y_l^{j}(x))\phi'(y_l^j(x')) \Theta_l(x,x') 
    +I'\,,
\end{align*}
where $I' = \frac{\sigma_w^2}{N_l} W_{l+1}^{ii} (\phi'(y_l^{i}(x)) + \phi'(y_l^i(x')))\, \Theta_l(x,x').$\\

 As $N_l \rightarrow \infty$,  we have that $I' \rightarrow 0$. Using the law of large numbers, as $N_l \rightarrow \infty$ 
 $$
\frac{\sigma_w^2}{N_l} \sum_{j}^{N_{l}} (W_{l+1}^{ij})^2 \phi'(y_l^{j}(x)) \phi'(y_l^j(x')) \Theta_l(x,x') \rightarrow \Psi'_l \Theta_l(x,x').
 $$
Moreover, we have that 
\begin{align*}
(&\partial_{W_{l+1}} y_{l+1}^i(x))(\partial_{W_{l+1}} y_{l+1}^i(x'))^t +  (\partial_{B_{l+1}} y_{l+1}^i(x))(\partial_{B_{l+1}} y_{l+1}^i(x'))^t\\ &= \frac{\sigma_w^2}{N_l} \sum_{j} \phi(y_l^j(x)) \phi(y_l^j(x')) + \sigma_b^2
\underset{N_l \rightarrow \infty}{\rightarrow} \sigma_w^2 \mathbb{E}[\phi(y_{l}^1(x))\phi(y_{l}^1(x'))] + \sigma_b^2 = \Psi_l\,,
\end{align*}
and so we conclude.
\end{proof}
As a corollary of the above result, using the results in \citep{daniely2016deeper} for the ReLU activation function, we can express the recursion more explicitly. We have
\begin{align*}
    \Psi_l = \sbq + \tfrac{\swq}{2}\left(1+\tfrac{f(C_l)}{C_l}\right)Q_l\,;\qquad\Psi'_l = \tfrac{\swq}{2}(1+f'(C_l))\,,
\end{align*}
where $f$ is defined in \eqref{deff} and $f':\gamma\mapsto -\tfrac{1}{\pi}\arccos\gamma$ is the first derivative of $f$. So we can write
\begin{align}\label{eq:recurrence_relu_ntk}
    \Theta_{l+1} = \Theta_l + \lambda_{l+1,L}^2\left(\sbq + \tfrac{\swq}{2}\left(1+\tfrac{f(C_l)}{C_l}\right)Q_l+\tfrac{\swq}{2}(1+f'(C_l))\,\Theta_l\right).
\end{align}
We can now easily check that the NTK is a kernel in the sense of Definition \ref{def:kernel}.
\begin{lemma}[$\Theta_l$ is a kernel]\label{NTK_is_kernel}
For all layer $l$, $\Theta_L$ is a kernel in the sense of definition \eqref{def:kernel}.
\end{lemma}
\begin{proof}
It's clear that $\Theta_0=Q_0$ is a kernel. Now fix any layer $l$. We have already proved in Lemma \ref{lemma:Ql_is_kernel} that $\left(1+\tfrac{f(C_l)}{C_l}\right)Q_l$ is a kernel. With a similar argument, noting that $1+f'$ can be expressed as a power series with only non negative coefficients on $[-1,1]$, we conclude by Lemma \ref{schur} that $1+f'(C_l)$ is a kernel. Using the usual argument that sums and product of kernels are kernels, we conclude by induction that $\Theta_l$ is a kernel. 
\end{proof}
As a final remark, note that from \eqref{eq:recurrence_resnet}, we have that $\lambda_{l,L}^2\Psi_l = Q_{l+1}-Q_l$. Hence we can rewrite \eqref{eq:recurrence_relu_ntk} as 
\begin{align}\label{deltathetadeltaq}
    \Theta_{l+1} - \Theta_l  = Q_{l+1}-Q_l+ \lambda_{l+1,L}^2\tfrac{\swq}{2}(1+f'(C_l))\,\Theta_l\,.
\end{align}
Since $1+f'$ is non negative on $[-1,1]$, it is easy to show by induction that $\Theta_l\geq Q_l$, point-wise, for all $l$. This is done explicitly in the next Lemma, which is a Corollary of Lemma \ref{lemma:exploding} and show the divergence of the NTK for a Standard ResNet.
\begin{lemma}[Exploding NTK]
Consider a ResNet of form \eqref{equation:Resnet_dynamics}. For all $x \in \mathbb{R}^d$,
\begin{align}
\Theta_{L}(x,x) \geq \left(1+\tfrac{\sigma_w^2}{2}\right)^{L} \left(Q_0(x,x)+\tfrac{2\sbq}{\swq}\right)\,.
\end{align}
\end{lemma}
\begin{proof}
By Lemma \ref{lemma:exploding}, it suffices to show that $\Theta_L(x,x)\geq Q_L(x,x)$. \\
Recall \eqref{eq:recurrence_relu_ntk}, noticing that $1+f'\geq 0$ on $[-1,1]$, $\left(1+\tfrac{f(C_l)}{C_l}\right)Q_l\geq 0$ and that $\Theta_0 = Q_0\geq 0$, by an easy induction we have that $\Theta_l\geq 0$ for all $l$. As a consequence, from \eqref{deltathetadeltaq}, we get that $\Theta_{l+1}-\Theta_l\geq Q_{l+1}-Q_l$. Hence, again with a straightforward induction we have that $\Theta_l\geq Q_l$ for all $l$ and the the whole $K^2$. In particular $\Theta_L(x,x)\geq Q_L(x,x)$ for all $x\in K$. 
\end{proof}

\begin{lemma}[Normalized NTK recursion]\label{lemma:stable_ntk_recursion}
Consider a ResNet of type \eqref{equation:Resnet_dynamics} without bias, and let $\alpha = \frac{\sigma_w^2}{2}$. The NTK recursion formula can be written in terms of normalized NTK $\kappa^l(x,x') = \Theta_l(x, x') / (1 + \alpha)^{l-1}$
$$
\kappa_l(x,x') = \left(\frac{1 + \alpha \hat{f}'(C_{l-1}(x,x'))}{1 + \alpha}\right) \kappa_{l-1}(x,x')  + \alpha \hat{f}(C_{l-1}(x,x')) \sqrt{Q_0(x,x)Q_0(x',x')}\,,
$$
where $\hat{f}$ is given by \eqref{equation:f_hat}, $\hat{f}(t) = \frac{1}{\pi} (t\,\arcsin{t} +\sqrt{1-t^2}) + \frac{1}{2}t$.
\end{lemma}

\begin{proof}
Let $x,x' \in \mathbb{R}^d$. For a ResNet of type \eqref{equation:Resnet_dynamics}, we have that 
$$
\Theta_{l} = \Theta_{l-1} + \left(\Psi_{l-1} + \Psi'_{l-1}\,\Theta_{l-1}\right)\,,
$$
where $\Psi_{l-1} = \alpha Q_{l-1}(x,x')$ and $\Psi_{l-1}' = \alpha \hat{f}'(C_{l-1})$. Using the recursive formula for the diagonal elements, we have that $\Psi_{l-1} = \alpha (1+\alpha)^{l-1} \hat{f}(C_{l-1}(x,x')) \sqrt{Q_0(x,x)Q_0(x',x')}$. We conclude by dividing both sides by $(1+\alpha)^{l-1}$.
\end{proof}

\subsection{Proof of Proposition \ref{prop:decreasing_ntk_universal}}

\begin{manualprop}{\ref{prop:decreasing_ntk_universal}}
Fix a compact $K\subset\R^d$ ($0\notin K$ if $\sigma_b=0$) and consider a Stable ResNet with decreasing scaling. Then $\Theta_L$ converges uniformly over $K^2$ to a kernel $\Theta_\infty$. Moreover $\Theta_\infty$ is universal on $K$ if $\sigma_b>0$. If $K=\Sd$, then the universality holds for $\sigma_b=0$.
\end{manualprop}

\begin{proof}
Let $K\subset\R^d$ ($0\notin K$ if $\sigma_b=0$) be a compact. From \eqref{eq:recurrence_relu_ntk}, with a decreasing scaling, we have that
\begin{equation*}
\begin{aligned}
\Theta_{l} &= \Theta_{l-1} + \lambda_{l}^2\left(\Psi_{l-1} + \Psi'_l\,\Theta_{l-1}\right)\\
&= \left(1 + \lambda_l^2\frac{\sigma_w^2}{2} f'(C_{l-1})\right)\Theta_{l-1} + \lambda_{l}^2\Psi_{l-1}\,.
\end{aligned}
\end{equation*}
Therefore, the NTK can be expressed exclusively in terms of the covariance kernels $(Q_k)_{ k\in [0:l-1]}$, more precisely we have that 
$$
\Theta_l = \prod_{k=1}^l \left(1 + \lambda_k^2 \frac{\sigma_w^2}{2} f'(C_{k-1})\right)  Q_0 + \sum_{k=1}^l \lambda_l^2 \prod_{j=k}^l \left(1 + \lambda_j^2 \frac{\sigma_w^2}{2} f'(C_{j-1})\right) \Psi_{k-1}\,.
$$
It is straightforward that $\Theta_l$ converges pointwise to a limiting kernel $\Theta_\infty$. Let us prove that the convergence is uniform over $K$. By observing that $|f'| \leq 1$, we have that for all $x,x' \in K$
\begin{equation*}
\begin{aligned}
\left| \Theta_\infty(x,x') - \Theta_l(x,x') \right| &\leq \prod_{k=1}^l \left(1 + \lambda_k^2 \frac{\sigma_w^2}{2}\right) \left| \prod_{k=l+1}^\infty \left(1 + \lambda_k^2 \frac{\sigma_w^2}{2}\right) - 1\right| Q_0(x,x') \\
&+ \sum_{k=l+1}^\infty \lambda_k^2 \prod_{j=k}^l \left(1 + \lambda_j^2 \frac{\sigma_w^2}{2}\right) \Psi_{k-1}(x,x')\\
&\leq \kappa \sum_{k=l+1}^\infty \lambda_k^2\,.
\end{aligned}
\end{equation*}
where $\kappa$ is a constant that depends on the compact $K$. This proves the uniform convergence with a rate of $\mathcal{O}\left(\sum_{k=l+1}^\infty \lambda_k^2\right)$. As a consequence, being a uniform limit of kernels, $\Theta_\infty$ is a kernel.

Proceeding as in the proof of Lemma \ref{NTK_is_kernel}, it's easy to prove by induction that for all $l$, $\Theta_l-Q_l$ is a kernel. In particular,
$$
T(\Theta_l) \succeq T(Q_l)\,,
$$
where $\succeq$ is in the operator sense, that is $T(\Theta_l)-T(Q_l)$ is non-negative definite. This yields 
$$
T(\Theta_\infty) \succeq T(Q_\infty)\,.
$$
Therefore $\Theta_\infty$ inherits the universality of $Q_\infty$ naturally by the RKHS hierarchy \citep{paulsen2016RKHS}. We conclude that $\Theta_\infty$ is universal (for both cases).
\end{proof}

For the rest of this section, let $K\subset\R^d$ by a compact set. If $\sigma_b=0$, assume that $0\notin K$.

With the uniform scaling, for arbitrary $x,x'\in K$, the continuous version of \eqref{recurrence_NTK} reads
\begin{align}\label{eq:ntk_cont}
\begin{split}
&\dot \theta_t(x,x') = \dot q_t(x,x') + \tfrac{\swq}{2} (1+f'(c_t(x,x')))\,\theta_t(x,x')\,;\\
&\theta_0 = q_0\,,
\end{split}
\end{align}
where $f':\gamma\mapsto -\tfrac{1}{\pi}\arccos\gamma$ is the first derivative of $f$, defined in \eqref{deff}.
\begin{lemma}\label{ex_NTK}
For any $x,x'$ in $K$, the solution $t\mapsto\Theta_t$ of \eqref{eq:ntk_cont} is unique and well defined for all $t\in[0,1]$. Moreover, the map $(x,x')\mapsto \Theta_t(x,x')$ is a kernel in the sense of Definition \ref{def:kernel} for all $t\in[0,1]$.\\
We have the $L^2(K)$ convergence of the discrete model to the continuous one:
\begin{align*}
    \lim_{L\to\infty}\sup_{l\in[0:L]}\left\|T(\Theta_{l|L})-T(\theta_{t=l/L})\right\|_2=0\,.
\end{align*}
\end{lemma}
\begin{proof}
The existence and the uniqueness are clear, since it is a homogeneous first order Cauchy problem, with continuous coefficients. We can write explicitly the solution as 
\begin{align}\label{eq:theta_implicit}
    \theta_t = e^{G_t}\left(q_0+\int_0^t\dot q_s\,e^{-G_s}\,\dd s\right)\,,
\end{align}
where $G_t(z) = \tfrac{\swq}{2}\,\int_0^t(1+f'(c_s(z)))\,\dd s$ for $z\in K^2$. It becomes then clear that $z\mapsto\Theta_t(z)$ is a continuous and symmetric function on $K^2$.\\
it is easy to check that the uniform convergence of $C$ and $Q$ to $c$ and $q$ implies that for all $z\in K$, $\lim_{L\to\infty}\sup_{l\in[0:L]}|\Theta_{l|L}(z)-\theta_{l/L}(z)|=0$. As consequence, by dominated convergence, 
\begin{align*}
    \lim_{L\to\infty}\sup_{l\in[0:L]}\left\|T(\Theta_{l|L})-T(\theta_{t=l/L})\right\|_2=0\,.
\end{align*}
Hence, $T(\theta_t)$ is the limit of a sequence of non-negative definite operators and hence it is non-negative definite, so that $\theta_t$ is a kernel on $K$ for all $t\in [0,1]$.
\end{proof}

\begin{manualprop}{\ref{prop:uniform_ntk_universal}}
Let $K\subset\R^d$ and fix $t\in(0,1]$. If $\sigma_b>0$, then $\theta_t$ is universal on $K$. The same holds true if $\sigma_b=0$ and $K=\Sd$.
\end{manualprop}
\begin{proof}
Fix $t\in(0,1]$. The solution of \eqref{eq:ntk_cont} can be written as $\theta_t = q_t + r_t$, where
\begin{align*}
    r_t = \tfrac{\swq}{2}\,\int_0^t (1+f'(c_s))\,\theta_s\,\dd s\,.
\end{align*}
Now, let us show that $r_t$. First, by Lemma \ref{propf} it is easy to check that $1+f'$ is analytic on $(-1,1)$ and its Taylor expansion around $0$ converges on $[-1,1]$. Moreover all the Taylor coefficients are non negative. Hence, Lemma \ref{schur} shows that $(1+f'(c_s))$ is a kernel for all $s\in[0,s)$. It follows that $(1+f'(c_s))\,\theta_s$ is a kernel.\footnote{See footnote \ref{footnote_schur}.} Now, $(1+f'(c_s)\,\theta_s$ is continuous and symmetric on $Z^2$, and it is easy to check from \eqref{eq:theta_implicit} that it is uniformly bounded for $s\in[0,t)$. It follows that $r_t$ is continuous and symmetric. Now, fix an arbitrary finite Borel measure $\mu$ on $K$. We have to show that $T_\mu(r_t)$ is non-negative definite, so that we can conclude by Lemma \ref{T_mu(Q)}. Fixed $\f\in L^2(K,\mu)$, by simple standard arguments we have
\begin{align*}
    \langle T_\mu(r_t)\,\f,\f\rangle = \int_0^t\langle T_\mu((1+f'(c_s))\,\theta_s)\,\f,\f\rangle\,\dd s\geq 0
\end{align*}
and so $r_t$ is a kernel.\\
Now, given two kernels $Q$ and $R$, it is a classical result that $Q+R$ is a kernel and its RKHS contains the RKHS of $Q$ and $R$, \citep{paulsen2016RKHS}. We conclude that the RKHS of $\theta_t$ contains the RKHS of $q_t$. Since $q_t$ is universal, $\theta_t$ is universal.
\end{proof}

\section{A PAC-Bayes Generalization result}\label{app:section_PAC-Bayes}
In this section, we study the PAC-Bayes upper bound of a GP with kernel $Q_L$. We consider a dataset $S$ with $N$ iid training examples $\{(x_i,y_i) \in X\times Y, i\in[1:N]\}$, and a hypothesis space $\mathcal{H}$ from which we want to learn an optimal hypothesis according to some bounded loss function $\ell : Y\times Y \rightarrow [0,1]$. The empirical loss of a hypothesis $h \in \mathcal{H}$ is given by
$$
r_S(h) = \frac{1}{N} \sum_{i=1}^N \ell(h(x_i), y_i)\,.
$$
Assuming that the samples are distributed as $(x,y) \sim \nu$ where $\nu$ is a probability distribution on $X \times Y$, we define the generalization (true) loss by 
$$
r(h) = \mathbb{E}_{\nu}[\ell(f(x),y)]\,.
$$
For some randomized learning algorithm $\mathcal{A}$, the empirical and generalization loss are given by
$$
r_S(\mathcal{A}) = \mathcal{E}_{h \sim \mathcal{A}}[r_s(h)]\,; \qquad r(\mathcal{A}) = \mathcal{E}_{h \sim \mathcal{A}}[r(h)]\,.
$$
The PAC-Bayes theorem gives a probabilistic upper bound on the generalization loss $r(\mathcal{A})$ of a randomized learning algorithm $\mathcal{A}$ in terms of the empirical loss $r_S(\mathcal{A})$. Fix a prior distribution $\mathcal{P}$ on the hypothesis set $\mathcal{H}$. The Kullback-Leibler divergence between $\mathcal{A}$ and $\mathcal{P}$ is defined as $KL(\mathcal{A} \| \mathcal{P}) = \int \log \frac{\mathcal{A}(h)}{P(h)} \mathcal{A}(h) dh \in [0, \infty]$. The Bernoulli KL-divergence is given by $kl(a||p) = a \log \frac{a}{p} + (1-a) \log \frac{1-a}{1-p} $  for $a,p \in [0,1]$. We define the inverse Bernoulli KL-divergence $kl^{-1}$ by 
$$
kl^{-1}(a,\varepsilon) = \sup\{ p \in [0,1] : kl(a,p) \leq \varepsilon \}\,.
$$
\begin{manualthm}{\ref{thm:pac-bayes_theorem}}[PAC-Bayesian theorem]
For any loss function $\ell$ that is $[0,1]$ valued, for any distribution $\nu$, for any $N \in \mathbb{N}$, for any prior $P$, and any $\delta \in (0,1]$, with probability at least $1-\delta$ over the sample $S$, we have 
$$
\forall \mathcal{A}, \hspace{0.25cm} r(\mathcal{A}) \leq  kl^{-1}\left(r_S(\mathcal{A}), \frac{KL(\mathcal{A}\|P) + \log \frac{2 \sqrt{N}}{\delta}}{N}\right).
$$
\end{manualthm}
The PAC-Bayesian theorem gives can also be stated as
$$ kl(r_S(\mathcal{A}), r(\mathcal{A})) \leq  \frac{KL(\mathcal{A}\|P) + \log \frac{2 \sqrt{N}}{\delta}}{N}\,.
$$
The KL-divergence term $KL(\mathcal{A}\|P)$ plays a major role as it controls the generalization gap, i.e. the difference (in terms of Bernoulli KL-divergence) between the empirical loss and the generalization loss. In our setting, we consider an ordinary GP regression with prior $P(f) = \mathcal{GP}(f | 0, Q(x,x'))$. Under the standard assumption that the outputs $y_N = (y_i)_{i \in [1:N]}$ are noisy versions of $f_N = (f(x_i))_{ i \in [1:N]}$ with $y_N | f_N \sim \mathcal{N}(y_N | f_N, \sigma^2 I)$, the Bayesian posterior $\mathcal{A}$ is also a GP and is given by 
\begin{equation}
\begin{split}
\mathcal{A}(f) = \mathcal{GP}(f | Q_N(x)(Q_{NN} + \sigma^2 I)^{-1} y_N, Q(x,x') - Q_N(x)(Q_{NN} + \sigma^2 I)^{-1})Q_N(x')^T)\,,
\end{split}
\end{equation}
where $Q_N(x) = (Q(x,x_i))_{i \in [1:N]}$ and $Q_{NN} = (Q(x_i, x_j))_{1\leq i,j \leq N}$. In this setting, we have the following result
\begin{manualprop}{\ref{prop:pac-bayes_bound}}[Stability of PAC-Bayes bound]
Let $Q_L$ be the kernel of a ResNet. Let $P_L$ be a GP with kernel $Q_L$ and $\mathcal{A}_L$ be the corresponding Bayesian posterior for some fixed noise level $\sigma>0$. Then, in a fixed setting (fixed sample size N), the following results hold:
\begin{enumerate}
    \item With a standard ResNet, we have 
    $$
    KL(\mathcal{A}_L \| P_L) \gtrsim L\,.
    $$
    \item With a Stable ResNet, we have 
    $$
    KL(\mathcal{A}_L \| P_L) = \mathcal{O}_L(1)\,.
    $$
\end{enumerate}
\end{manualprop}
\begin{proof}
The proof relies on the simple observation that $P_L(f | f_N) = \mathcal{A}_L(f | f_N)$. This yields
\begin{equation}
\begin{split}
KL(\mathcal{A}_L \| P_L) &= KL(\mathcal{A}_L(f_N) \mathcal{A}_L(f | f_N) \| P_L(f_N) P_L(f | f_N)) \\
&= KL(\mathcal{A}_L(f_N) \| P_L(f_N)) \\
&= \frac{1}{2} \log(\det (Q_{L,NN} + \sigma^2 I)) - \frac{N}{2} \log(\sigma^2) - \frac{1}{2} \Tr (Q_{L,NN}(Q_{L,NN} + \sigma^2I)^{-1})\\
&+ \frac{1}{2} y_N^T (Q_{L,NN} + \sigma^2 I)^{-1} Q_{L,NN} (Q_{L,NN} + \sigma^2 I)^{-1} y_N\,,
\end{split}
\end{equation}
where $Q_{L,NN} = (Q_{L}(x_i,x_j))_{1\leq i, j \leq N}$.

Since $Q_{L,NN}$ is symmetric and strictly positive definite, it is straightforward that the largest eigenvalue of $Q_{L,NN}(Q_{L,NN} + \sigma^2I)^{-1})$ is smaller than $1$. This yields   
$$\Tr (Q_{L,NN}(Q_{L,NN} + \sigma^2I)^{-1}) \leq N$$ and $$y_N^T (Q_{L,NN} + \sigma^2 I)^{-1} Q_{L,NN} (Q_{L,NN} + \sigma^2 I)^{-1} y_N \leq \sigma^{-2} \|y_N\|_2\,.$$
Both quantities are bounded independently from $L$ and the scaling factors $(\lambda_{k,L})_{k\in[2:L]}$. \\
Now let us analyse the first term $\frac{1}{2} \log(\det (Q_{L,NN} + \sigma^2 I))$. Let $\mu_{L,0} \geq \mu_{L,1} \geq \dots \geq \mu_{L,N}$ be the eigenvalues of $Q_{L,NN}$. For a simplification purpose, we assume the inputs belong to the unit sphere $\Sd$. The proof extends to any compact set.\\
Let us study the behaviour of the first term for both cases.

\textit{Case 1.} Assume we have a standard ResNet architecture. On the unit sphere $\mathbb{S}^{d-1}$, we have that $Q_L(x,x') \geq q_L C_L(x,x')$, where $q_L = (1 + \frac{\sigma_2}{2})^{L} \delta$ with $\delta = (\sigma_b^2 + \frac{\sigma_w^2}{d}) / (1 + \frac{\sigma_w^2}{2})$. Using Lemma \ref{lemma:infnite_depth_standard_resnet}, we know that $\lim_{L \rightarrow \infty} \hat{\mu}_{L,0} =  \hat{\mu}_{\infty,0} \in (0,\infty)$ and for all $k\geq 1$, $\lim_{L \rightarrow \infty} \hat{\mu}_{L,k} = 0$. This yields
    \begin{equation*}
    \begin{split}
        \log(\det (Q_{L,NN} + \sigma^2 I)) &\geq \sum_{k=1}^N \log(q_L \hat{\mu}_{L,k} + \sigma^2)\\
        &\geq \log(q_L \hat{\mu}_{L,0} + \sigma^2) + (N-1)\log(\sigma^2)\\
        &\gtrsim L\log(1+\frac{\sigma_w^2}{2})\,,
    \end{split}
    \end{equation*}
    where the last inequality holds for sufficiently large $L$.
    
\textit{Case 2.} In the case of Stable ResNet, we know that as $L \rightarrow \infty$, the kernel $Q_L$ converges to a strictly positive definite kernel $Q_\infty$, therefore the first term $\log(\det (Q_{L,NN} + \sigma^2 I))$ remains bounded as $L \rightarrow \infty$, which concludes the proof.
\end{proof}
\section{NNGP correlation kernel without bias as a modified NNGP kernel}\label{app:corr_as_nngp}
Unscaled ResNets suffer from the exploding variance problem, which needs to be avoided in order to isolate the disadvantages of inexpressivity in their NNGP kernel. In order to do so, we use the NNGP correlation kernel $C$ instead of NNGP covariance kernel $Q$, noting that Lemma \ref{lemma:corr_formula} provides a simple recursion formula for $C$ if $\sigma_b=0$, at depth $l\leq L$:

\begin{equation}\label{equation:corr_repeated}
C_l(x,x') = \frac{1}{1+\alpha_{l.L}} C_{l-1}(x,x') + \frac{\alpha_{l,L}}{1 + \alpha_{l,L}} \hat{f}(C_{l-1}(x,x') )\,,
\end{equation}
where $\alpha_{l,L} = \frac{\lambda_{l,L}^2 \sigma_w^2}{2}$ and $\hat{f}$ defined in \eqref{equation:f_hat}.
In order to combine this with open-source packages \citep{neuraltangents2020, jax2018github} designed for NNGP calculation, we note that \eqref{equation:corr_repeated} can be viewed as the NNGP kernel of the following modified ResNet layer, using the same notation as in \eqref{equation:scaled_Resnet}:
\begin{align}\label{equation:corr_Resnet}
\begin{split}
&y_l(x) = \sqrt{1 - \hat{\alpha}_{l,L}} y_{l-1}(x) + \sqrt{\hat{\alpha}_{l,L}}\, \mathcal{F}((W_l,B_l), y_{l-1})\,, \quad l \in [1:L]\,,
\end{split}
\end{align}
with $\hat{\alpha}_{l,L}=\frac{\alpha_{l,L}}{1+\alpha_{l,L}}$
\section{Experimental details and additional results}\label{app: experimental details}
\subsection{NNGP results}
For our Vanilla ResNet NNGP results, we preprocess all training, validation and test data by first centering the training set and then normalizing all images to lie on the pixel dimension sphere. For our Wide ResNet NNGP results we normalise all data so that the training set is centered and has channel-wise unit variance. We use Kaiming \citep{he2} initialisation throughout, with $\swq=2$ and $\sbq=0$. Vanilla ResNets have the same structure as type \eqref{equation:scaled_Resnet} in Table \ref{tab:vrn_krr} and we use the same WRN kernel architecture as \citep{lee_wide_nn_ntk} in Table \ref{table: wrn_krr} but omit the final average pooling step, which is known to improve kernel performance but dramatically increase computational costs \citep{novak2018bayesian, lee2020finite}. Throughout this work, where there are residual blocks with multiple layers, we calculate our scaling factors for uniform and decreasing scaled Stable ResNets by the number of residual connections. For example, a WRN-202 has only 99 residual connections, so we set $\lambda_{l,L}^{-1}=\sqrt{99}$ for the uniform scaling factors.
We tune the noise variance $\sigma^2$, which is akin to the regularisation parameter in kernel ridge regression. To do so, we compute validation accuracy on a validation set of size 5000, selecting the best $\sigma^2=\lambda \times \text{Trace}(Q_{NN})/N$ from a logarithmic scale of $\lambda=[0.001, 0.01, 0.1]$, where $N$ is the training set size and $Q_{NN}$ is the $N\times N$ training set Gram matrix for NNGP $Q$. 
\subsection{Trained ResNet results}
For all our trained ResNet experiments we use a similar setup to the open-source code for \citep{Wang2020Picking} in PyTorch \citep{paszke2019pytorch}. We repeat each experiment 3 times and report the best test accuracy and error intervals. 
All ResNets are initialised with Kaiming initialisation \citep{he2} and like \citep{Wang2020Picking} we adopt ResNets architectures where we double the number of filters in each convolutional layer. For experiments with BatchNorm, on CIFAR-10/100 we use batch size 64 across all depths and on TinyImageNet we used batch size 128 for depths 32 \& 50, and batch size 100 for depth 104 in order to allow the model to fit onto a single 11GB VRAM GPU. We use SGD with momentum parameter $0.9$ and weight decay parameter $10^{-4}$ throughout.

We also present results for ResNets trained without BatchNorm \citep{ioffe2015batch}. BatchNorm is a normalization layer commonly used with modern ResNets that is known to improve performance and allows deeper ResNets to be trained, though the precise reasons for this are not well understood. Several recent works \citep{de2020batch, zhang2019fixup} have studied the possibility of removing the need for BatchNorm layers, by introducing trainable uniform scalings to the residual connection to stabilise variance at initialisation \& gradients, demonstrating promising results. Note, our work additionally introduces decreasing scaling and also uses the infinite-width NNGP/NTK connection to assess the theoretical advantages of scaled Stable ResNets in the limit of infinite depth. 

Moreover, our focus is not towards the possibility of removing BatchNorm and we show in Table \ref{table: NN results} that our scalings can improve BatchNorm ResNets. However, we also present results without BatchNorm in Table \ref{table: NN results no batchnorm}, where again we see that our scaled stable ResNets improve performance compared to their unscaled counterparts: for example both Decreasing and Uniform scaling outperform the unscaled ResNet by over 3\% test accuracy on CIFAR-100 with ResNet-104.

For ResNets trained without BatchNorm, for a fair comparison we tuned the initial learning rate on a small logarithmic scale, using batch size 128.
\begin{table}[ht]
\caption{Test accurracies (\%) of trained deep ResNets \textbf{without BatchNorm} of various scalings and depths on CIFAR-10 (C-10), CIFAR-100 (C-100).}\label{table: NN results no batchnorm}
\setlength\tabcolsep{5pt}
\centering
\begin{tabular}{lllll}
\toprule
Dataset & Depth &               Scaled (D) &              Scaled (U) &              Unscaled \\
\midrule
C-10 & 32  &  ${92.64}_{\pm 0.19}$ &     ${92.78}_{\pm 0.18}$ &  ${92.11}_{\pm 0.17}$ \\
& 50  &  ${92.33}_{\pm 0.05}$ &  $\bm{92.72}_{\pm 0.12}$ &  ${92.10}_{\pm 0.17}$ \\
& 104 &  ${92.81}_{\pm 0.09}$ &  $\bm{93.28}_{\pm 0.17}$ &  ${92.70}_{\pm 0.08}$ \\
\midrule

C-100 & ResNet32  &  $\bm{67.73}_{\pm 0.42}$ &     ${67.06}_{\pm 0.38}$ &  ${65.37}_{\pm 0.32}$ \\
& ResNet50  &  $\bm{69.38}_{\pm 0.20}$ &     ${68.76}_{\pm 0.18}$ &  ${66.02}_{\pm 0.41}$ \\
& ResNet104 &     ${70.60}_{\pm 0.52}$ &  $\bm{70.95}_{\pm 0.13}$ &  ${67.41}_{\pm 0.41}$ \\
\bottomrule
\end{tabular}
\end{table}
\section{Some results on the Sphere $\Sd$}\label{app:sphere}
On the sphere $\Sd$, the kernel $Q_2$ is analytic as a result of lemma \ref{lemma:analytic_decomp_f_hat}. Moreover, the coefficient of the analytic decomposition are all positive.
\begin{lemma}[Analytic decomposition of 2 layer ReLU ResNet]\label{lemma:analytic_decom_relu_3layers}
For all $(x,x') \in \Sd$, $Q_2(x,x') = g(x \cdot x')$ where $g(z) = \sum_{i\geq 0} a_i z^i$ and $a_i>0$ for all $i\geq 0$. 
\end{lemma}
\begin{proof}
Let $x,x' \in \Sd$. We have 
$$
Q_0(x,x') = \sigma_b^2 + \frac{\sigma_w^2}{d} x \cdot x'\,.
$$
As a result, for all $x,x'$, $Q_0(x,x) = Q_0(x',x') = \sigma_b^2 + \frac{\sigma_w^2}{d}$. The diagonal term of the kernel is the same for all $x \in \Sd$. We note $\beta_l = Q_l(x,x)$ and $z = x \cdot x'$. Using this observation, we have that
\begin{align*}
Q_1(x,x') &= Q_0(x,x') + \lambda_1^2 (\sigma_b^2 + \frac{\sigma_2}{2} \hat{f}(C_0(x,x')) \beta_0)\,.
\end{align*}
It can be easily deduced from lemma \ref{lemma:analytic_decomp_f_hat} that there exist $\{b_i\}_{i \geq 0}$ such that
$$
C_1(x,x') = b_0 + b_1 z + \sum_{i \geq 0} b_{2i} z^{2i}\,,
$$
where $b_0,b_1,b_{2i} > 0$.\\
Following the same approach, we have that
\begin{align*}
Q_2(x,x') &= Q_1(x,x') + \lambda_2 (\sigma_b^2 + \frac{\sigma_w^2}{2} f(C_1(x,x')) \beta_2)
\end{align*}
and 
\begin{align*}
\hat{f}(C_1(x,x')) &= a_0 + a_1 C_1(x,x') + \sum_{i\geq 1} a_{2i} (C_1(x,x'))^{2i}\,.
\end{align*}
Having the terms of orders 0 and 1 in $C_1(x,x')$ ensures having a positive coefficient for all terms $z^i$ for $i\geq 1$, which concludes the proof.
\end{proof}

The previous result can be easily extended to general $L \geq 2$. We have that
$$
Q_L(x,x') = g_L(x\cdot x')\,,
$$
where $g_L : [-1,1] \rightarrow \mathbb{R}$ is a continuous function. Kernels that can be written in this form are known as the dot-product kernels (or zonal kernels on the unit sphere). In our setting, we have a stronger property; we prove in the next result we show a that the kernel $Q_L$ is analytic on the sphere $\mathbb{S}^{d-1}$ in the sense that the function $g_L$ is analytic on $[-1,1]$.
\begin{prop}[$Q_L$ is analytic]\label{prop:analytic_decomposition_Q_L}
Let $L\geq 2$, there exists $(\alpha_{L,i})_{i\geq 0}$ such that for all $x,x' \in \mathbb{S}^{d-1}$
$$
Q_L(x,x') = \sum_{i\geq 0} \alpha_{L,i} (x \cdot x')^i\,.
$$
Moreover, $(\alpha_{L+1, i})_{i\geq 0}$ can be expressed in terms of $(\alpha_{L, i})_{i\geq 0}$
\begin{equation}
    \alpha_{L+1,i} = \alpha_{L,i} + \lambda_{L+1,L+1} \times  \gamma_{L,i}\,,
\end{equation}
with
\begin{equation*}
    \gamma_{L,i} =
    \begin{cases*}
      \sigma_b^2 + \beta_L \frac{\sigma_w^2}{2}  \sum\limits_{m\geq0} \frac{a_m}{\beta_L^m} \alpha_{L,0}^m   & if $i = 0$\,; \\
      \beta_L \frac{\sigma_w^2}{2}  \sum\limits_{m\geq0} \frac{a_m}{\beta_L^m} \sum\limits_{k_1+...+k_m=i} \prod\limits_{j=1}^m \alpha_{L,k_j}& if $i \geq 1$\,.
    \end{cases*}
\end{equation*}
where $\beta_L = Q_L(x,x) = Q_L(x',x') = \sum_{i\geq0} \alpha_{L,i}$ and $(a_m)_{m\geq0}$ is such that $a_0,a_1 >0$ and $a_{2i}>0$ and $a_{2i+1} = 0$ for all $i\geq 1$. \\
As a result, for all $L\geq 2, i\geq0, \alpha_{L,i}>0$.
\end{prop}
\begin{proof}

The result is true for $L=2$ by lemma \ref{lemma:analytic_decom_relu_3layers}. Let us prove the result for all $L\geq3$ by induction.\\
Let $ L\geq 3$, $x,x' \in \Sd$, $z = x\cdot x'$ and $\beta_l = Q_l(x,x) =Q_l(x',x')$. Assume the result is true for $L$ and let us prove it for $L+1$. We have that
$$
Q_{L+1}(x,y) = Q_{L}(x,y) + \lambda_{L+1,L+1}^2 (\sigma_b^2 + \frac{\sigma_w^2}{2} f(C_{l}(x,y)) \beta_l)\,.
$$
Knowing that $C_l(x,y) = \frac{1}{\beta_l} Q_l(x,y)$, we have that 
\begin{align*}
f(C_l(x,y)) &= \sum_{m\geq 0} \frac{a_m}{\beta_l^m} C_l(x,y)^m \\
&= \sum_{m\geq 0} \frac{a_m}{\beta_l^m} (\sum_{i\geq 0} \alpha_{l,i} z^i)^m \\
&= \sum_{m\geq 0} \frac{a_m}{\beta_l^m} \sum_{i\geq 0} \sum_{k_1+...+k_m=i} \prod_{j=1}^m \alpha_{l,k_j} z^i\\
&= \sum_{i\geq 0} \Big[\sum_{m\geq 0}\frac{a_m}{\beta_l^m}  \sum_{k_1+...+k_m=i} \prod_{j=1}^m \alpha_{l,k_j}\Big] z^i\,,
\end{align*}
which gives the recursive formulas for the coefficients of the analytic decomposition. Observe that the coefficients are non-decreasing wrt $L$. Using lemma \ref{lemma:analytic_decom_relu_3layers} we conclude that $\alpha_{L,i } > 0$.
\end{proof}
For depth $L\geq 2$, proposition \ref{prop:analytic_decomposition_Q_L} shows that all coefficient $(\alpha_{L,i})_{i\geq0}$ are (strictly) positive. It turns out that this is a sufficient condition for the kernel $Q_L$ to be strictly positive definite. We state this in the next proposition. The result can be seen as a consequence of Lemma \ref{lem10} and Lemma \ref{density_sphere}. However we will give here a more direct proof.
\begin{prop}[$Q_L$ is strictly p.d. for $L\geq 2$]\label{prop:strictly_pd_sphere}
Let $Q$ be an analytic kernel on the unit sphere $\mathbb{S}^{d-1}$, i.e. there exist a sequence of real numbers $(\alpha_i)_{i\geq 0}$ such that for all $x,x' \in \mathbb{S}^{d-1}$
$$
Q(x,x') = \sum_{i\geq 0 } \alpha_i (x \cdot x')^i
$$
Assume $\alpha_i > 0$ for all $i \in \mathbb{N}$. Then, $Q$ is strictly positive definite.\\
As a result, for all $L \geq 2$, $T_\nu(Q_L)$ is strictly positive definite, i.e. for any non-zero function $\f \in L^2(\mathbb{S}^{d-1}l,\nu)$
$$
\langle T_\nu(Q_L) \f, \f \rangle > 0
$$
$\nu$ is the standard uniform measure on the sphere $\mathbb{S}^{d-1}$.
\end{prop}
\begin{proof}
Let $Q$ be an analytic kernel on the unit sphere $\mathbb{S}^{d-1}$, that is there exists a sequence of real numbers $(\alpha_i)_{i\geq 0}$ such that for all $x,x' \in \mathbb{S}^{d-1}$
$$
Q(x,x') = \sum_{i\geq 0 } \alpha_i (x \cdot x')^i\,,
$$
and assume $\alpha_i > 0$ for all $i \in \mathbb{N}$. The map $(x,x')\mapsto x\cdot x'$ is trivially a kernel in the sense of Definition \ref{def:kernel}. For all $i \geq 0$, $(x,x')\mapsto (x\cdot x')^i$ is a kernel as well.\footnote{See footnote \ref{footnote_schur}.} It follows that $T_\nu(Q)$ is non-negative definite, as a converging sum of non-negative operators. Let us prove that it is strictly positive definite.\\
Let $\f \in L_2(\mathbb{S}^{d-1},\nu)$ such that $\langle T_\nu(Q)\, \f, \f \rangle = 0$. Since $\alpha_i > 0$ for all $i$, we have that for all $i \geq 0$
$$
\int \int (x \cdot x')^i \f(x) \f(x') \,\dd\nu(x)\dd\nu(x') = 0\,,
$$
recalling that $\nu$ is the uniform measure on the sphere $\mathbb{S}^{d-1}$. This yields
\begin{equation}\label{equation:zero_product_with_polynomials}
\int \int P(x \cdot x') \f(x) \f(x') \,\dd\nu(x)\dd\nu(x') = 0
\end{equation}
for any polynomial function $P$. \\
Since $\f$ is a function on the sphere $\mathbb{S}^{d-1}$, it can be decomposed in the Spherical Harmonics orthonormal basis $(Y_{k,j})_{k,j}$ (see e.g. \citep{macrobert}) as
$$
\forall x \in\mathbb{S}^{d-1},\hspace{0.25cm} \f(x) = \sum_{k\geq 0} \sum_{j=1}^{N(d,k)} b_{k,j} Y_{k,j}(x)
$$
where $b_{k,j} = \int_{\mathbb{S}^{d-1}} \f(w) Y_{k,j}(w) \,\dd\nu(w)$.\\
In particular, equation \eqref{equation:zero_product_with_polynomials} is true for the Associated Legendre Polynomials $P_k$. Knowing that $N(d,k) P_k(x\cdot x') = \sum_{j=1}^{N(d,k)} Y_{k,j}(x)Y_{k,j}(x')$,  \eqref{equation:zero_product_with_polynomials} yields
$$
\int \int \sum_{j=1}^{N(d,k)} Y_{k,j}(x)Y_{k,j}(x') \f(x) \f(x')\, \dd\nu(x)\dd\nu(x') = 0
$$
for all $k \geq 0$. Therefore,
$$
\sum_{j=1}^{N(d,k)} b_{k,j}^2 = 0
$$
for all $k \geq 0$. We conclude that $\f = 0$.
\end{proof}

By Mercer's theorem \citep{paulsen2016RKHS}, the kernel $Q_L$ can be decomposed in an orthonormal basis of $L^2(\mathbb{S}^{d-1})$. It turns out that this orthonormal basis is the so-called Spherical Harmonics of $\mathbb{S}^{d-1}$. This is a corollary of the next lemma, which is a classical result \citep{yang2019finegrained}.
\begin{lemma}[Spectral decomposition on $\Sd$]\label{lemma:spherical_decomposition}
Let $Q$ be a zonal kernel on $\Sd$, that is $Q(x,x') = p(x\cdot x')$ for a continuous function $p:[-1,1]\to\mathbb{R}$. Then, there is a sequence $\{\mu_k\geq 0\}_{k\in\mathbb N}$ such that for all $x,x'\in\Sd$
\begin{align*}
Q(x,x') = \sum_{k\geq 0 } \mu_{k} \sum_{j=1}^{N(d,k)} Y_{k,j}(x) Y_{k,j}(x')\,,
\end{align*}
where $\{Y_{k,j}\}_{k\geq0, j\in [1:N(d,k)]}$ are spherical harmonics of  $\mathbb{S}^{d-1}$ and $N(d,k)$ is the number of harmonics of order $k$. With respect to the standard spherical measure $\nu$ on $\Sd$, the spherical harmonics form an orthonormal basis of $L^2(\Sd,\nu)$ and $T_\nu(Q)$ is diagonal on this basis.
\end{lemma}
\begin{proof}
We start by giving a brief review of the theory of Spherical Harmonics (\citep{macrobert}). For some $k \geq 1$, let $(Y_{k,j})_{1\leq j \leq N(d,k)}$ be the set of Spherical Harmonics of degree $k$. We have  $N(d,k) = \frac{2k + d - 2}{k} {k + d-3 \choose d-2}$.

The set of functions $(Y_{k,j})_{k\geq1, j \in [1:N(d,k)]}$ form an orthonormal basis of $L^2(\mathbb{S}^{d-1},\nu)$, where $\nu$ is the uniform measure on $\Sd$.

For some function $p$, the Hecke-Funk formula reads
$$
\int_{\mathbb{S}^{d-1}} p(\langle x, w \rangle)Y_{k,j}(w) \,\dd\nu(w) = \frac{\Omega_{d-1}}{\Omega_d} Y_{k,j}(x) \int_{-1}^1 p(t) P^d_k(t) (1-t^2)^{(d-3)/2} dt
$$
where $\Omega_d$ is the volume of the unit sphere $\mathbb{S}^{d-1}$, and $P^d_{k}$ is the multi-dimensional Legendre polynomials given explicitly by Rodrigues' formula
$$
P^d_k(t) = \left(-\frac{1}{2}\right)^{k} \frac{\Gamma(\frac{d-1}{2})}{\Gamma(k + \frac{d-1}{2})} (1-t^2)^{\frac{3-d}{2}} \frac{\dd^k}{\dd t^k} (1-t^2)^{k + \frac{d-3}{2}}\,.
$$
$(P^d_k)_{k\geq0}$ form an orthogonal basis of $L^2([-1,1], (1-t^2)^{\frac{d-3}{2}}dt)$, i.e. 
$$
\langle P^d_k, P^d_{k'} \rangle_{L^2([-1,1], (1-t^2)^{\frac{d-3}{2}}dt)} = \delta_{k,k'} \,,
$$
where $\delta_{ij}$ is the Kronecker symbol.

Using the Heck-Funk formula, we prove that $Q$ can be decomposed on the Spherical Harmonics basis. Indeed, for any $x,x' \in \mathbb{S}^{d-1}$, the decomposition on the spherical harmonics basis yields 
$$
Q(x,x') = \sum_{k\geq 0} \sum_{j=1}^{N(d,k)} \left[\int_{\mathbb{S}^{d-1}}p(\langle w, x' \rangle)Y_{k,j}(w)\dd\nu(w)\right] Y_{k,j}(x)\,.
$$
Using the Hecke-Funk formula yields
$$
Q(x,x') = \sum_{k\geq0} \sum_{j=1}^{N(d,k)}\left[ \frac{\Omega_{d-1}}{\Omega_d} \int_{-1}^1 p(t) P^d_k(t) (1-t^2)^{(d-3)/2} dt\right] Y_{k,j}(x)Y_{k,j}(x')\,.
$$
We conclude that 
$$
Q(x,x') = \sum_{k \geq 0} \mu_k \sum_{j=1}^{N(d,k)} Y_{k,j}(x)Y_{k,j}(x')\,.
$$
where $\mu_k = \frac{\Omega_{d-1}}{\Omega_d} \int_{-1}^1 p(t) P^d_k(t) (1-t^2)^{(d-3)/2} dt$. We also have that $\mu_k \geq 0$ since $Q$ is non-negative by definition.\\
The last statement, follows from the spectral theory of compact self-adjoint operators and the orthonormality of the spherical harmonics (see the appendix of \citep{yang2019finegrained} for details).
\end{proof}
\begin{corollary}[Spectral decomposition of $Q_L$]\label{cor:spectral_decomposition_Q_L}
For $L\geq1$, there exist $(\mu_{L,k})_{k\geq0}$ such that $\mu_{L,k} > 0$ for all $k\geq 0$, and for all $x,x' \in \mathbb{S}^{d-1}$ we have
$$
Q_L(x,x') = \sum_{k\geq 0 } \mu_{L,k} \sum_{j=1}^{N(d,k)} Y_{k,j}(x) Y_{k,j}(x')\,,
$$
where $(Y_{k,j})_{k\geq0, j\in [1:N(d,k)]}$ are spherical harmonics of  $\mathbb{S}^{d-1}$ and $N(d,k)$ is the number of harmonics of order $k$.
\end{corollary}
Corollary \ref{cor:spectral_decomposition_Q_L} shows that for any depth $L$, the Spherical Harmonics are the eigenfunctions of the kernel $Q_L$. The fact that $\mu_{L,k} > 0$ is a direct result of Proposition \ref{prop:strictly_pd_sphere}. Leveraging this result, we can prove a stronger result, which is the universality of the kernel $Q_L$. 
\begin{prop}[Universality on $\Sd$]\label{prop:universality_on_sphere_appendix}
For all $L \geq 2$, $Q_L$ is universal on $\Sd$ for $d\geq 2$.
\end{prop}
\begin{proof}
The result is a consequence of Lemma \ref{lemma:Sd_spd=>univ} and Proposition \ref{prop:strictly_pd_sphere}. An alternative proof is the following.\\
It is a classical result that the set Spherical Harmonics form an orthonormal basis on $L^2(\Sd,\nu)$.
Leveraging the result from corollary \ref{cor:spectral_decomposition_Q_L}, it is straightforward that any continuous function in $L^2(\Sd,\nu)$ can be approximated by a function of the form $\sum_{i} Q_L(x_i, .)$ which belongs to the RKHS of $Q_L$. Therefore, $Q_L$ is universal on $\Sd$.\\
Note that we have not made the assumption that $\sigma_b >0$.
\end{proof}

\newpage
\bibliographystyle{plain}
\bibliography{sample}
\end{document}


%

%

\onecolumn
\aistatstitle{Stable ResNet : 
Appendix}

\section{Residual Neural Networks and Gaussian processes}\label{app:section:ResNets_and_GPs}
Consider a standard ResNet architecture of depth $L\geq2$
\begin{equation}\tag{\ref{equation:Resnet_dynamics}}
\begin{aligned}
y^1(x) &= W^1x + B^1 \\
y^l(x) &= y^{l-1}(x) + \mathcal{F}((W^l,B^l), y^{l-1}), \quad \mbox{for } l \in [2:L],
\end{aligned}
\end{equation}
where $x \in \mathbb{R}^d$ is an input, $y^l(x)$ is the vector of pre-activations, $W^l$ and $B^l$ are respectively the weights and bias of the $l^{th}$ layer, and $\mathcal{F}$ is a mapping that defines the nature of the layer. In general, the mapping $\mathcal{F}$ consists of a successive applications of simple activation functions. In this work, for the sake of simplicity, we consider Fully Connected blocks with ReLU activation function
$$
\mathcal{F}((W,B), x) = W\phi(x) + B
$$
Hereafter, $N_l$ denotes the number of neurons in the $l^{th}$ layer, $\phi$ the activation function and $[m:n]:=\{m,m+1, ..., n\}$ for  $m\leq n$. The weights and bias are initialized with $W^l\overset{\text{iid}}{\sim}\mathcal{N}(0,\sigma_w^2/N_{l-1})$, and $B^l\overset{\text{iid}}{\sim}\mathcal{N}(0,\sigma_b^2)$ where $\mathcal{N}(\mu, \sigma^{2})$ denotes the normal distribution of mean $\mu$ and variance $\sigma^{2}$. 

In \cite{yang2017meanfield}, authors showed that wide deep ResNets might suffer from gradient exploding during backpropagation. The following lemma

Recent results by \cite{hayou_pruning} suggest that scaling the residual blocks with $L^{-1/2}$ might have some beneficial properties on model pruning at initialization. This is a result of the stabilization effect of scaling on the gradient.

More generally, we introduce the residual architecture
\begin{equation}\tag{\ref{equation:scaled_Resnet}}
\begin{aligned}
y^1(x) &= \mathcal{F}(W^1, x), \\
y^l(x) &= y^{l-1}(x) + \lambda_{l,L} \times \mathcal{F}(W^l, y^{l-1}), \quad \mbox{for } l \in [2:L],
\end{aligned}
\end{equation}
where $(\lambda_{k,L})_{k\in[2:L]}$ is a sequence of scaling factors. We assume hereafter that there exists $\lambda_{\max} \in (0,\infty)$ such that for all $L\geq 2$ and $k \in [2:L]$, we have that $\lambda_{k,L} \in (0,\lambda_{\max})$.\\

Recall that for all $x,x'\in \mathbb R^d$
\begin{align*}
Q_l(x,x') = Q_{l-1}(x,x') + \lambda_{l,L}^2 \Psi_{l-1}(x,x')
\end{align*}
where $\Psi_{l-1}(x,x') = \sigma^2_b + \sigma^2_w \mathbb{E}[\phi(y^{l-1}_1(x))\phi(y^{l-1}_1(x'))]$

For the ReLU activation function $\phi(x) = \max(0,x)$, the recurrence relation can be written more explicitly as in \cite{daniely2016deeper}.
Let $C_l$ be the correlation kernel, defined as 
\begin{align*}
C_l(x,x') = \frac{Q_l(x,x')}{\sqrt{Q_l(x,x)Q_l(x,x')}}
\end{align*}
and let $f:[-1,1]\to\mathbb{R}$ be given by 
\begin{align}\label{deff}
f:\gamma\mapsto \frac{1}{\pi}(\sqrt{1-\gamma^2}-\gamma\arccos \gamma)\,.
\end{align}
For the remainder of this appendix, we define the function 
\begin{equation}\label{equation:f_hat}
\hat{f}(\gamma) =\gamma +  f(\gamma) = \frac{1}{\pi} (\gamma \arcsin(\gamma) + \sqrt{1-\gamma^2}) + \frac{1}{2}\gamma
\end{equation}

\begin{lemma}[Kernel Diagonal elements]\label{lemma:diagonal_elements}
Consider a ResNet of the form \ref{equation:scaled_Resnet} and let $x \in \mathbb{R}^d$. We have that for all $l \in [2:L]$,
$$
Q_{l}(x,x) = \prod_{k=2}^{l}(1 + \frac{\sigma_w^2 \lambda_{k,L}^2}{2}) (Q_1(x,x) + \frac{2 \sigma_b^2}{\sigma_w^2}) - \frac{2 \sigma_b^2}{\sigma_w^2}
$$
\end{lemma}
\begin{proof}
We know that 
$$
Q_{l}(x,x) = Q_{l-1}(x,x) + \lambda_{l,L}^2 (\sigma_b^2 + \frac{\sigma_w^2}{2} \hat{f}(1))
$$
where $\hat{f}$ is given by \ref{equation:f_hat}. It is straightforward that $\hat{f}(1) = 1$. This yields
$$
Q_{l}(x,x) + \frac{2 \sigma_b^2}{ \sigma_w^2} =  (1+ \lambda_{l,L}^2 \frac{\sigma_w^2}{2})(Q_{l-1}(x,x) + \frac{2 \sigma_b^2}{ \sigma_w^2})
$$
we conclude by telescopic product.
\end{proof}

\begin{lemma}[Corollary of Theorem D.1. in \cite{yang_tensor3_2020}]\label{lemma:gradient_independence}
Consider a ResNet of the form \ref{equation:scaled_Resnet} with weights $W$. In the limit of infinite width, we can assume that $W^T$ used in back-propagation is independent from $W$ used for forward propagation, for the calculation of Gradient Covariance and NTK.
\end{lemma}

\begin{manualprop}{\ref{proposition:stable_gradient}}[Stable Gradient]
Consider a ResNet of type \ref{equation:scaled_Resnet}, and let $\mathcal{L}_y(x):= \ell(y^L_1(x), y)$ for some $(x,y) \in \mathbb{R}^d \times \mathbb{R}$, where $\ell : (z,y) \mapsto \ell(z,y)$ is a loss function satisfying $\sup_{K_1\times K_2}\left|\frac{\partial \ell(z,y)}{\partial z}\right| < \infty$ for any two compact subsets $K_1, K_2 \subset \mathbb{R}$. Let $(\sigma_b, \sigma_w) \in \mathbb{R}^+ \times (\mathbb{R}^+)^*$. Then, under the approximation of infinite width, for any compact subsets $K \subset \mathbb{R}^d$, $K' \subset \mathbb{R}$, there exists a constant $C>0$ such that 
$$
\sup_{(x,y) \in K\times K'}\sup\limits_{l \in [1: L]}\mathbb{E}\left[\left|\frac{\partial \mathcal{L}_y(x)}{\partial W^l_{11}}\right|^2\right] \leq C \hspace{0.05cm} \exp{\left(\frac{\sigma_w^2}{2} \sum_{k=2}^L \lambda_{k,L}^2\right)}
$$
Moreover, if there exists $\lambda_{\min}>0$ such that for all $L\geq 2$ and $k\in [2:L]$ we have $\lambda_{k,L}\geq\lambda_{\min}$, then, for all $(x,y) \in \mathbb{R}^d\{0\} \times \mathbb{R}$ such that $\left|\frac{\partial \ell(z,y)}{\partial z}\right| \neq 0$, there exists $\kappa>0$ such that , for all $l \in [2:L]$,
$$
\mathbb{E}\left[\left|\frac{\partial \mathcal{L}_y(x)}{\partial W^l_{11}}\right|^2\right] \geq \kappa (1 + \frac{\lambda^2_{\min} \sigma_w^2}{2})^L
$$
\end{manualprop}

\begin{proof}
Let $(x,y) \in \mathbb{R}^d \times \mathbb{R}$ and $\Bar{q}^l(x) = \mathbb{E}\left[\left|\frac{\partial \mathcal{L}_y(x)}{\partial y^l_{1}}\right|^2\right]$. Using lemma \ref{lemma:gradient_backprop}, we have that 
$$
\Bar{q}^l(x) = (1 + \frac{\sigma_w^2 \lambda_{l+1,L}^2}{2}) \Bar{q}^{l+1}(x)
$$
this yields
$$
\Bar{q}^l(x) = \prod_{k=l+1}^L (1 + \frac{\sigma_w^2 \lambda_{k,L}^2}{2})
 \Bar{q}^L(x)$$ 
 
Moreover, using lemma \ref{lemma:gradient_independence}, we have that $\mathbb{E}\left[\left|\frac{\partial \mathcal{L}_y(x)}{\partial W^l_{11}}\right|^2\right] = \lambda_{l,L}^2 \Bar{q}^l(x) \mathbb{E}[\phi(y^{l-1}_1(x))^2]$. We have $\mathbb{E}[\phi(y^{l-1}_1(x))^2] = \frac{1}{2} Q_{l-1}(x,x)$. From lemma \ref{lemma:diagonal_elements} we know that 
\begin{equation*}
Q_{l-1}(x,x) = \prod_{k=2}^{l-1}(1 + \frac{\sigma_w^2 \lambda_{k,L}^2}{2}) (Q_1(x,x) + \frac{2 \sigma_b^2}{\sigma_w^2}) - \frac{2 \sigma_b^2}{\sigma_w^2}\\
\leq \prod_{k=2}^{l-1}(1 + \frac{\sigma_w^2 \lambda_{k,L}^2}{2}) (Q_1(x,x) + \frac{2 \sigma_b^2}{\sigma_w^2})
\end{equation*}
This yields
$$
\mathbb{E}\left[\left|\frac{\partial \mathcal{L}_y(x)}{\partial W^l_{11}}\right|^2\right] \leq \frac{2}{\sigma_w^2} \prod_{k=2}^L(1 + \frac{\sigma_w^2 \lambda_{k,L}^2}{2}) (\frac{1}{2}Q_1(x,x) + \frac{ \sigma_b^2}{\sigma_w^2}) \Bar{q}^L(x)
$$
It is straightforward that $\prod_{k=2}^l(1 + \frac{\sigma_w^2 \lambda_{k,L}^2}{2}) \leq  \exp{\left(\frac{\sigma_w^2}{2} \sum_{k=2}^L \lambda_{k,L}^2\right)}$. 
Let $K \subset \mathbb{R}^d$, $K' \subset \mathbb{R}$ be two compact subsets. Using the condition on the loss function $\ell$, we have that
$$
\mathbb{E}\left[\left|\frac{\partial \mathcal{L}_y(x)}{\partial W^l_{11}}\right|^2\right] \leq C \exp{\left(\frac{\sigma_w^2}{2} \sum_{k=2}^L \lambda_{k,L}^2\right)}
$$
where $C = \frac{2}{\sigma_w^2} \sup_{(x,y) \in K\times K'}\Bar{q}^L(x) \times (\sup_{x\in K} Q_1(x,x) + \frac{2 \sigma_b^2}{\sigma_w^2})$. We conclude by taking the supremum.\\

Let $(x,y) \in \mathbb{R}^d\{0\} \times \mathbb{R}$ such that $\left|\frac{\partial \ell(z,y)}{\partial z}\right| \neq 0$. We have that 

\begin{equation*}
\begin{split}
\mathbb{E}\left[\left|\frac{\partial \mathcal{L}_y(x)}{\partial W^l_{11}}\right|^2\right] &\geq \frac{1}{2} \frac{\lambda_{l,L}^2}{1 + \frac{\sigma_w^2}{2}\lambda_{l,L}^2}   \prod_{k=2}^L(1 + \frac{\sigma_w^2 \lambda_{k,L}^2}{2}) Q_1(x,x) \Bar{q}^L(x)\\
& \geq \kappa (1 + \frac{\sigma_w^2 \lambda_{\min}^2}{2})^L
\end{split}
\end{equation*}
where $\kappa = \frac{1}{2} \frac{\lambda_{\min}^2}{(1 + \frac{\sigma_w^2}{2}\lambda_{\max}^2)(1 + \frac{\sigma_w^2}{2}\lambda_{\min}^2)} Q_1(x,x) \Bar{q}^L(x) > 0$.
\end{proof}

Using lemma \ref{lemma:gradient_independence}, we can derive simple recursive formulas for the second moment of the gradient as well as for the Neural Tangent Kernel (NTK). This was previously done in \cite{samuel} for feedforward neural networks, we prove a similar result for ResNet in the next lemma.
\begin{lemma}[Gradient Second moment]\label{lemma:gradient_backprop}
In the limit of infinite width, using the same notation as in proposition \ref{proposition:stable_gradient}, we have that 
$$
\Bar{q}^l(x) = (1 + \frac{\sigma_w^2 \lambda_{l+1,L}^2}{2}) \Bar{q}^{l+1}(x)
$$
\end{lemma}
\begin{proof}
It is straighforward that 
$$
\frac{\partial \mathcal{L}_y(x)}{\partial y^l_i} = \frac{\partial \mathcal{L}_y(x)}{\partial y^{l+1}_i} + \lambda_{l+1,L} \sum_{j} \frac{\partial \mathcal{L}_y(x)}{\partial y^{l+1}_j} W^{l+1}_{ji} \phi'(y^l_i)
$$
using lemma \ref{lemma:gradient_independence} and using CLT theorem, we have that 
$$
\Bar{q}^l(x) = \Bar{q}^{l+1}(x) + \lambda_{l+1,L}^2 \Bar{q}^{l+1}(x) \sigma_w^2 \mathbb{E}[\phi'(y^l_i(x))^2]
$$
we conclude using $\mathbb{E}[\phi'(y^l_i(x))^2] = \mathbb{P}(\mathcal{N}(0,1) > 0) = \frac{1}{2}$.
\end{proof}

\begin{definition}[Stable ResNet]
a ResNet of type \ref{equation:scaled_Resnet} is called a Stable ResNet if and only if $\lim\limits_{L \rightarrow \infty}\sum\limits_{k=1}^L \lambda_{k,L}^2 < \infty$.
\end{definition}

\subsection{Some general results}
\begin{lemma}[O'Donnell (2014)]
Let $F(\phi)(c) = \mathbb{E}[\phi(X)\phi(Y), (X,Y) \sim \mathcal{N}(0, \begin{pmatrix}
1 & c \\
c & 1
\end{pmatrix})]$. Then for all $\phi \in L^2(\mathcal{N}(0,1))$, there exists $a_0,a_1, ... \geq 0$ such that $F(\phi)(c) = \sum_{i\geq0} a_i c^i$ for all $c \in [-1,1]$. 
\end{lemma}

\begin{lemma}[Analytic decomposition of ReLU kernel]\label{lemma:analytic_decomp_f_hat}
let $f$ be the ReLU kernel defined by $\hat{f}(c) = 2 \mathbb{E}[\phi(Z_1) \phi(c Z_1 + \sqrt{1 - c^2} Z_2)] = c + f(c)$ where $Z_1, Z_2$ are iid standard Gaussian variables. Let $a_0,a_1, ... \geq 0$ be the coefficients of the analytic decomposition of $\hat{f}$, i.e. $\hat{f}(c) = \sum_{i\geq0} a_i c^i$. Then, we have $a_0 = 1/\pi, a_1 =1/2$, and for all $i\geq 1$, $a_{2i} >0$ and $a_{2i+1}=0$.

\end{lemma}

\begin{proof}
We have that for all $c\in [-1,1]$
$$
\hat{f}(c) = \frac{1}{\pi} c \hspace{0.1cm}\text{arcsin}(c) + \frac{1}{\pi} \sqrt{1-c^2} + \frac{1}{2}c
$$
therefore, $a_0 = \frac{1}{\pi}$. Moreover, we have that for all $c \in (-1,1)$
$$
\hat{f}'(c) = \frac{1}{\pi} \arcsin{c} + \frac{1}{2}
$$
this yields $a_1 = f'(0) = \frac{1}{2}$.
$$
\hat{f}^{(3)}(c) = \frac{c}{\pi (1 - c^2)^{3/2}}
$$
since $\hat{f}^{(3)}$ is an odd function, then for all $i \geq 1, a_{2i+1} = 0$. Now let us prove that for all $k\geq 1$, there exist $b_{k,0},b_{k,1},..., b_{k,k-1} > 0$ such that for all $c \in (-1,1)$
$$
\hat{f}^{(2k)}(c) = \frac{1}{\pi} \sum_{m=0}^{k-1} b_{k,m} c^{2m} (1 - c^2)^{-k-m+1/2}
$$
We prove this by induction. For $k=1$, we have that 
$$
\hat{f}^{(2)}(c) = \frac{1}{\pi} (1 - c^2)^{-1/2}
$$
the result is true for $k=1$. Assume the result is true for some $k\geq 1$, let us prove it for $k+1$. It is easy to see that
\begin{equation}
    b_{k+1,m} =
    \begin{cases*}
      2(2k-1) b_{k,0} + 2b_{k,1} & if $m = 0$ \\
      2(4k^2-1) b_{k,0} + 5(2k+1) b_{k,1} + 12 b_{k,2} & if $m = 1$ \\
      2(m+1)(2m+1) b_{k,m+1} + (4m+1)(2k+2m-1) b_{k,m} \\
      + (2k+2m-3)(2k+2m-1) b_{k,m-1} & if $m \in \{2,3,...,k-1\}$ \\
      (4k-3)(4k-1) b_{k,k-1} & if $m = k$ \\
    \end{cases*}
  \end{equation}

The induction is straightforward.\\

In particular, we have shown that $a_{2i} = \frac{\hat{f}^{(2i)}(0)}{(2i)!} = \frac{b_{i,0}}{(2i)!}> 0$.
\end{proof}
\subsection{Results on the Sphere $\Sd$}

\begin{lemma}[Analytic decomposition of 3 layer ReLU ResNet]\label{lemma:analytic_decom_relu_3layers}
For all $(x,x') \in \Sd$, $Q_3(x,x') = g(x \cdot x')$ where $g(z) = \sum_{i\geq 0} a_i z^i$ and $a_i>0$ for all $i\geq 0$. 
\end{lemma}
\begin{proof}
Let $x,y \in \Sd$. We have 
$$
Q_1(x,x') = \sigma_b^2 + \frac{\sigma_w^2}{d} x \cdot x'
$$
as a result, for all $Q_1(x,x) = Q_1(x',x') = \sigma_b^2 + \frac{\sigma_w^2}{d}$. The diagonal term of the kernel is the same for all $x \in \Sd$. We note $\beta_l = Q_l(x,x)$ and $z = x \cdot x'$. Using this observation, we have that
\begin{align*}
Q_2(x,x') &= Q_1(x,x') + \lambda_2^2 (\sigma_b^2 + \frac{\sigma_2}{2} \hat{f}(C_1(x,x')) \beta_1)
\end{align*}
It can be easily deduced from lemma \ref{lemma:analytic_decomp_f_hat} that there exist $(b_i)_{i \geq 0}$ such that
$$
C_2(x,x') = b_0 + b_1 z + \sum_{i \geq i} b_{2i} z^{2i}
$$
where $b_0,b_1,b_{2i} > 0$.\\
Following the same approach, we have that
\begin{align*}
Q_3(x,x') &= Q_2(x,x') + \lambda_3 (\sigma_b^2 + \frac{\sigma_w^2}{2} f(C_2(x,x')) \beta_2)
\end{align*}
and 
\begin{align*}
\hat{f}(C_2(x,x')) &= a_0 + a_1 C_2(x,x') + \sum_{i\geq 1} a_{2i} (C_2(x,x'))^{2i}
\end{align*}
Having the terms of orders 0 and 1 in $C_2(x,x')$ ensures having a positive coefficient for all terms $z^i$ for $i\geq 1$, which concludes the proof.
\end{proof}

Let $x,x' \in \mathbb{S}^{d-1} := \{x \in \mathbb{R}^d, \|x\|=1\}$. It is clear that the kernel $Q_L$ can be written as 
$$
Q_L(x,x') = g_L(x\cdot x')
$$
where $g_L : [-1,1] \rightarrow \mathbb{R}$ is a continuous function. Kernels that can be written in this form are known as the dot-product kernels (or zonal kernels on the unit sphere). In our setting, we have a stronger property; we prove in the next result we show a that the kernel $Q_L$ is analytic on the sphere $\mathbb{S}^{d-1}$ in the sense that the function $g_L$ is analytic on $[-1,1]$.
\begin{prop}[$Q_L$ is analytic]\label{prop:analytic_decomposition_Q_L}
Let $L\geq 3$, there exists $(\alpha_{L,i})_{i\geq 0}$ such that for all $x,x' \in \mathbb{S}^{d-1}$
$$
Q_L(x,x') = \sum_{i\geq 0} \alpha_{L,i} (x \cdot x')^i
$$
Moreover, $(\alpha_{L+1, i})_{i\geq 0}$ can be expressed in terms of $(\alpha_{L, i})_{i\geq 0}$
\begin{equation}
    \alpha_{L+1,i} = \alpha_{L,i} + \lambda_{L+1,L+1} \times  \gamma_{L,i}
\end{equation}
whith
\begin{equation*}
    \gamma_{L,i} =
    \begin{cases*}
      \sigma_b^2 + \beta_L \frac{\sigma_w^2}{2}  \sum\limits_{m\geq0} \frac{a_m}{\beta_L^m} \alpha_{L,0}^m   & if $i = 0$ \\
      \beta_L \frac{\sigma_w^2}{2}  \sum\limits_{m\geq0} \frac{a_m}{\beta_L^m} \sum\limits_{k_1+...+k_m=i} \prod\limits_{j=1}^m \alpha_{L,k_j}& if $i \geq 1$ \\
    \end{cases*}
\end{equation*}
where $\beta_L = Q_L(x,x) = Q_L(x',x') = \sum_{i\geq0} \alpha_{L,i}$ and $(a_m)_{m\geq0}$ is such that $a_0,a_1 >0$ and $a_{2i}>0$ and $a_{2i+1} = 0$ for all $i\geq 1$. \\
As a result, for all $L\geq 3, i\geq0, \alpha_{L,i}>0$.
\end{prop}
\begin{proof}

The result is true for $L=3$ by lemma \ref{lemma:analytic_decom_relu_3layers}. Let us prove the result for all $L\geq3$ by induction.\\
Let $ L\geq 3$, $x,x' \in \Sd$, $z = x\cdot x'$ and $\beta_l = Q_l(x,x) =Q_l(x',x')$. Assume the result is true for $L$ and let us prove it for $L+1$. We have that
$$
Q_{L+1}(x,y) = Q_{L}(x,y) + \lambda_{L+1,L+1}^2 (\sigma_b^2 + \frac{\sigma_w^2}{2} f(C_{l}(x,y)) \beta_l)
$$
Knowing that $C_l(x,y) = \frac{1}{\beta_l} Q_l(x,y)$, we have that 
\begin{align*}
f(C_l(x,y)) &= \sum_{m\geq 0} \frac{a_m}{\beta_l^m} C_l(x,y)^m \\
&= \sum_{m\geq 0} \frac{a_m}{\beta_l^m} (\sum_{i\geq 0} \alpha_{l,i} z^i)^m \\
&= \sum_{m\geq 0} \frac{a_m}{\beta_l^m} \sum_{i\geq 0} \sum_{k_1+...+k_m=i} \prod_{j=1}^m \alpha_{l,k_j} z^i\\
&= \sum_{i\geq 0} \Big[\sum_{m\geq 0}\frac{a_m}{\beta_l^m}  \sum_{k_1+...+k_m=i} \prod_{j=1}^m \alpha_{l,k_j}\Big] z^i\\
\end{align*}
which gives the recursive formulas for the coefficients of the analytic decomposition. Observe that the coefficients are non-decreasing w.r.t $L$. Using lemma \ref{lemma:analytic_decom_relu_3layers} we conclude that $\alpha_{L,i } > 0$.
\end{proof}
For depth $L\geq 3$, proposition \ref{prop:analytic_decomposition_Q_L} shows that all coefficient $(\alpha_{L,i})_{i\geq0}$ are (strictly) positive. It turns out that this is a sufficient condition for the kernel $Q_L$ to be strictly positive definite. We show this in the next proposition.

\begin{prop}[$Q_L$ is strictly p.d. for $L\geq 3$]
Let $Q$ be an analytic kernel on the unit sphere $\mathbb{S}^{d-1}$, i.e. there exist a sequence of real numbers $(\alpha_i)_{i\geq 0}$ such that for all $x,x' \in \mathbb{S}^{d-1}$
$$
Q(x,x') = \sum_{i\geq 0 } \alpha_i (x \cdot x')^i
$$
Assume $\alpha_i > 0$ for all $i \in \mathbb{N}$. Then, $Q$ is strictly positive definite.\\
As a result, for all $L \geq 3$, $Q_L$ is strictly positive definite, i.e. for any non-zero function $\f \in L^2(\mathbb{S}^{d-1})$
$$
\langle Q_L \f, \f \rangle > 0
$$
where the scalar product is defined with respect to the uniform measure on the sphere $\mathbb{S}^{d-1}$.
\end{prop}

\begin{proof}
Let $Q$ be an analytic kernel on the unit sphere $\mathbb{S}^{d-1}$,  i.e. there exist a sequence of real numbers $(\alpha_i)_{i\geq 0}$ such that for all $x,x' \in \mathbb{S}^{d-1}$
$$
Q(x,x') = \sum_{i\geq 0 } \alpha_i (x \cdot x')^i
$$
and assume $\alpha_i > 0$ for all $i \in \mathbb{N}$. The linear kernel $(x,x')\mapsto x\cdot x'$ is trivially positive semi-definite. By a standard extension of the Schur product theorem to integral operators \cite{horn2012matrix} we have that the kernel $(x,x')\mapsto (x\cdot x')^{i}$ is positive semi-definite for all $i \geq 0$. It follows that $Q$ is positive semi-definite as a sum of p.s.d kernels. Let us prove that it is strictly positive definite.\\
Let $\phi \in L_2(\mathbb{S}^{d-1})$ such that $\langle Q \f, \f \rangle = 0$. Since $\alpha_i > 0$ for all $i$, we have that for all $i \geq 0$
$$
\int \int (x \cdot x')^i \f(x) \f(x') d\nu_{d-1}(x)d\nu_{d-1}(x') = 0
$$
where $\nu_{d-1}$ is the uniform measure on the sphere $\mathbb{S}^{d-1}$. This yields
\begin{equation}\label{equation:zero_product_with_polynomials}
\int \int P(x \cdot x') \f(x) \f(x') d\nu_{d-1}(x)d\nu_{d-1}(x') = 0
\end{equation}
for any polynomial function $P$. \\
Since $\f$ is a function on the sphere $\mathbb{S}^{d-1}$, it can be decomposed in the Spherical Harmonics orthonormal basis $(Y_{k,j})_{k,j}$ \CIT as
$$
\forall x \in\mathbb{S}^{d-1},\hspace{0.25cm} \f(x) = \sum_{k\geq 0} \sum_{j=1}^{N(d,k)} b_{k,j} Y_{k,j}(x)
$$
where $b_{k,j} = \int_{\mathbb{S}^{d-1}} \f(w) Y_{k,j}(w) d\nu_{d-1}(w)$.\\
Equation \ref{equation:zero_product_with_polynomials} is particularly true for the Associated Legendre Polynomials $P_k$ \CIT. Knowing that $N(d,k) P_k(x\cdot x') = \sum_{j=1}^{N(d,k)} Y_{k,j}(x)Y_{k,j}(x')$, equation \ref{equation:zero_product_with_polynomials} yields
$$
\int \int \sum_{j=1}^{N(d,k)} Y_{k,j}(x)Y_{k,j}(x') \f(x) \f(x') d\nu_{d-1}(x)d\nu_{d-1}(x') = 0
$$
for all $k \geq 0$. Therefore,
$$
\sum_{j=1}^{N(d,k)} b_{k,j}^2 = 0
$$
for all $k \geq 0$. We conclude that $\f = 0$.
\end{proof}

The kernel $Q_L$ is symmetric, continuous and strictly positive definite. Therefore, it is a Mercer kernel, i.e. it can be decomposed in an orthonormal basis of $L^2(\mathbb{S}^{d-1})$. It turns out that this orthonormal basis is the so-called Spherical Harmonics of $\mathbb{S}^{d-1}$. This is a corollary of the next lemma, which is a classical result \cite{yang2019finegrained}.
\begin{lemma}[Spectral decomposition on $\Sd$]\label{lemma:spherical_decomposition}
Let $Q$ be an integral operator such that its integral kernel is in the form $Q(x,x') = p(x\cdot x')$, for some continuous function $p:[-1,1]\to\mathbb{R}$. Then we can find coefficients $\alpha$'s such that for all $x,x'\in\Sd$
$$
Q(x,x') = \sum_{k\geq 0 } \alpha_{k} \sum_{j=1}^{N(d,k)} Y_{k,j}(x) Y_{k,j}(x')
$$
where $(Y_{k,j})_{k\geq0, j\in [1:N(d,k)]}$ are spherical harmonics of  $\mathbb{S}^{d-1}$ and $N(d,k)$ is the number of harmonics of order $k$.
\end{lemma}
\begin{corollary}[Spectral decomposition of $Q_L$]\label{cor:spectral_decomposition_Q_L}
For $L\geq1$, there exist $(\mu_{L,k})_{k\geq0}$ such that $\mu_{L,k} > 0$ for all $k\geq 0$, and for all $x,x' \in \mathbb{S}^{d-1}$ we have
$$
Q_L(x,x') = \sum_{k\geq 0 } \mu_{L,k} \sum_{j=1}^{N(d,k)} Y_{k,j}(x) Y_{k,j}(x')
$$
where $(Y_{k,j})_{k\geq0, j\in [1:N(d,k)]}$ are spherical harmonics of  $\mathbb{S}^{d-1}$ and $N(d,k)$ is the number of harmonics of order $k$.
\end{corollary}
Corollary \ref{cor:spectral_decomposition_Q_L} shows that for any depth $L$, the Spherical Harmonics are the eigenfunctions of the kernel $Q_L$. Leveraging this result, we can prove a stronger result, which is the universality of the kernel $Q_L$. Let us first give a formal definition of universal kernels (\cite{Steinwart}).
\begin{definition}[Universal Kernel]\label{univ_kern}
Let $K$ be a compact subset of $\mathbb{R}^d$, $Q$ a strictly positive definite kernel on $\mathbb{R}^d$, and $H(K)$ the Reproducing Kernel Hilbert Space (RKHS) generated by the kernel $Q$. We say that $Q$ is universal on $K$ if for any $\epsilon>0$ and continuous function $g$ on $K$, there exists $f \in H(K)$ such that $\|f-g\|_{\infty} <\epsilon$.
\end{definition}
Universal kernels are characterized by the property that their associated RKHS is dense (w.r.t the uniform norm $\|.\|_{\infty}$) in the space of continuous functions on the compact set $K$. This is in particular crucial for Kernel regression and Gaussian Process inference \footnote{The closure of the set of functions described by the mean function of the posterior of a GP regression is exactly the RKHS of the kernel of the GP prior}. It turns out that for $L \geq 3$, $Q_L$ is universal.
\begin{prop}
For $L \geq 3$, $Q_L$ is universal on the sphere $\mathbb{S}^{d-1}$.
\end{prop}

\section{Stable ResNet with uniform scaling}\label{app:section:uniform_scaling}
Henceforth, let $K\in\R^d$ be compact. If $\sigma_b=0$, assume that $0\notin K$.\\
Fix any non-zero finite Borel measure $\mu$ on $K$. Any kernel $Q$ on $K$ induces a non-negative compact integral operator on $L^2(K,\mu)$. We will denote this induced kernel as $T_\mu(Q)$. \TODO{State the existence and properties of $T_\mu$ as a lemma.}\\
We state here a characterization of universal kernels \cite{JMLR:v12:sriperumbudur11a}. 
\begin{lemma}\label{univ=Lp_spd}
Let $Q:K^2\to\mathbb\R$ be a kernel, where $K\subset\R^d$ is compact. $Q$ is a universal kernel if and only if $T_\mu(Q)$ is strictly positive definite for all finite Borel measure $\mu$, i.e., $\langle T_\mu(Q)\,\f,\f\rangle > 0$ for all non-zero $\f\in L^2(K,\mu)$.
\end{lemma}
\subsection{Continuous formulation}
\begin{lemma}\label{ex}
For any $x,x'$ in $K$, the solution of \eqref{recc} is unique and well defined for all $t\in[0,1]$. The maps $(x,x')\mapsto q_t(x,x')$ and $(x,x')\mapsto c_t(x,x')$ are Lipschitz continuous on $K^2$ and $c_t$ takes values in $[-1,1]$. Moreover, both $q_t$ and $c_t$ are kernels in the sense of Definition \ref{def:kernel}.
\end{lemma}
\begin{proof}
For the diagonal terms $q_t(x,x)$, \eqref{recc} reduces to $\dot q_t ={\sigma_b}^2 + \frac{{\sigma_w}^2}{2}\,q_t$, whose solution is 
\begin{align*}
q_t(x,x) = e^{\frac{{\sigma_w}^2}{2}\,t}\,q_0(x,x)+\frac{2{\sigma_b}^2}{{\sigma_w}^2}\left(e^{\frac{{\sigma_w}^2}{2}\,t}-1\right)=e^{\frac{{\sigma_w}^2}{2}\,t}\,(\sbq + \swq\,\|x\|^2)+\frac{2{\sigma_b}^2}{{\sigma_w}^2}\left(e^{\frac{{\sigma_w}^2}{2}\,t}-1\right)\,.
\end{align*}
This allows to write a time dependent ODE for $c_t(x,x')$ as 
\begin{align}\label{eqc}
\begin{split}
&\dot c_t(x,x')= {\sigma_b}^2\left(\Gt(x,x') - \At(x,x')\,c_t(x,x')\right) + \frac{\swq}{2}\,f(c_t(x,x'))\,,\\
&c_0(x,x') = \frac{\sbq + \swq\,x\cdot x'}{\sqrt{(\sbq + \swq\,\|x\|^2)(\sbq + \swq\,\|x'\|^2)}}\,.
\end{split}
\end{align}
where $f$ is defined in $\eqref{deff}$ and
\begin{align*}
\At(x,x') = \frac{1}{2}\left(\frac{1}{q_t(x,x)}+\frac{1}{q_t(x',x')}\right)\,;\qquad\Gt(x,x') = \sqrt{\frac{1}{q_t(x,x)\,q_t(x',x')}}\,.
\end{align*}
Fix $z=(x,x')\in K^2$ and let $\gamma_0 = c_0(z)\in[-1,1]$. Consider $\bar f:\mathbb{R}\to \mathbb{R}$, an arbitrary Lipschitz extension of $f$ to the whole $\mathbb{R}$ and define $H:[0,\infty)\times \mathbb{R}\to \mathbb{R}$ as
\begin{align*}
H(t,\gamma) = \sbq(\Gt(z) - \At(z)\,\gamma) + \frac{\swq}{2}\,\bar f(\gamma)\,. 
\end{align*}
$H$ is Lipschitz continuous in $\gamma$ and $C^\infty$ in $t$, so there exists $\tau>0$ such that the Cauchy problem
\begin{align*}
&\dot \gamma (t) = H(t,\gamma(t))\,;\\
&\gamma(0) = \gamma_0
\end{align*}
has a unique $C^1$ solution defined for  $t\in[0,\tau)$. \\
Noticing that 
\begin{align*}
\Gt(x,x') -\At(x,x') = -\frac{1}{2}\left(\frac{1}{q_t(x,x)}-\frac{1}{q_t(x',x')}\right)^2\leq 0\,,
\end{align*}
we get that for all $t_1$ such that $\gamma(t_1)=1$ we have $\dot\gamma(t_1)\leq 0$, since $f(1)=0$, and for all $t_{-1}$ such that $\gamma(t_{-1})=-1$ we have $\dot\gamma(t_{-1}) = \sbq(\Gt(x,x')+\At(x,x'))+\tfrac{\swq}{2}>0$. As a consequence $\gamma(t)\in[-1,1]$ for all $t\in[0,\tau)$ and we can take $\tau=\infty$.\\
In particular we get that \eqref{eqc} has a unique solution $t\mapsto c_t(z)$, defined for $t\in[0,1]$ and bounded in $[-1,1]$. \\
As a consequence, \eqref{recc} has a unique and well defined solution for all $t\geq 0$. \\

Now notice that $z\mapsto c_0(z)$ is Lipschitz on $K^2$. Let's denote as $L_0$ a Lipschitz constant for $c_0$.\\
Since both $\Gt$ and $\At$ are $C^1$, we can find real constants $L_G$, $L_A$ and $M_A$ such that for all $z,z'$ elements of  $K^2$
\begin{align*}
&|\Gt(z)-\Gt(z')|\leq L_G\,\|z-z'\|\,;\\
&|\At(z)-\At(z')|\leq L_A\,\|z-z'\|\,;\\
&|\At(z)|\leq M_A\,.
\end{align*}
Let $L_f$ be a Lipschitz constant for $f$. Using the fact that $|c_t|\leq 1$, we can write
\begin{align*}
|\dot c_t(z) - \dot c_t(z')| \leq L_1\,\|z-z'\|+L_2\,|c_t(z)-c_t(z')|\,,
\end{align*}
where $L_1 = \sbq(L_G+L_A)$ and $L_2 = \sbq M_A+\frac{\swq}{2}\,L_f$.\\
Now fix $z$ and $z'$ and consider $\Delta(t) = c_t(z)-c_t(z')$. We have 
\begin{align*}
&|\dot\Delta(t)| \leq L_1\,\|z-z'\| + L_2\,|\Delta(t)|\,;\\
&|\Delta(0)| \leq L_0\,\|z-z'\| \,.
\end{align*} 
So $|\Delta(t)|\leq \left(\frac{L_1}{L_2}\,\left(e^{L_2\,t}-1\right)+L_0\,e^{L_2\,t}\right)\|z-z'\|$, meaning that $c_t$ (and so $q_t$) is Lipschitz on $L^2$.\\

$(x,x')\mapsto q_t(x,x')$ being continuous, it defines a compact integral operator $T(q_t)$ on $L^2(K)$ \cite{lang2012real}. Since $q_t$ is real and symmetric under the swap of $x$ and $x'$, the operator is self-adjoint. The same holds true for $c_t$.\\ 

The fact that  $T(q_t)$ is a non-negative operator can be seen as a corollary of Lemma \ref{convcont}. Indeed all $T(Q_{l|L})$ is a non-negative definite operator, since it's induced by a kernel. Hence, for each $t\in[0,1]$ it is enough to find a sequence $(l_n,L_n)$ such that $L_n\to\infty$ and $(l_n-1)/(L_n-1)\to t$. By Lemma \ref{convcont} $T(Q_{l_n|L_n})\to T(q_t)$ in the $L^\infty$ norm, and hence in $L^2$, as we are on a compact. Since the subspace of non-negative definite operators in $L^2$  is closed wrt the $L^2$ operator norm, we conclude.\\
Once established that $T(q_t)$ is non-negative definite, it follows immediately that $T(c_t)$ is non-negative as well. We can thus conclude that both $q_t$ and $c_t$ are kernels, in the sense of Definition \ref{def:kernel}.
\end{proof}
\begin{corollary}\label{corex}
The maps $t\mapsto T_\mu(q_t)$ and $t\mapsto T_\mu(c_t)$, defined on $[0,1]$, are continuous and twice differentiable with respect to the operator norm in $L^2(K,\mu)$. Moreover, $\tfrac{\dd}{\dd t}T(q_t) = T(\dot q_t)$, $\tfrac{\dd}{\dd t}T(c_t) = T(\dot c_t)$, $\tfrac{\dd^2}{\dd t^2}T(q_t) = T(\ddot q_t)$ and $\tfrac{\dd^2}{\dd t^2}T(c_t) = T(\ddot c_t)$.
\end{corollary}
\begin{proof}
Consider the map $(t,z)\mapsto q_t(z)$, defined on $[0,1]\times K^2$, which is continuous wrt $z$ and $C^2$ wrt $t$, as it can be easily checked. Since $K^2$ and $[0,1]$ are compact sets, it follows that for any $t$
\begin{align*}
\lim_{s\to t} \sup_{z\in I^2}\left|\frac{q_{s}(z)-q_t(z)}{s-t} -\dot q_t (z)\right|= \sup_{z\in I^2}\lim_{s\to t}\left|\frac{q_{s}(z)-q_t(z)}{s-t}-\dot q_t(z)\right| = 0\,.
\end{align*}
Hence $\lim_{s\to t}\frac{q_s-q_t}{t-s}=\dot q_t$ uniformly on $K^2$, and hence $\lim_{s\to t}\frac{T(q_s)-T(q_t)}{t-s}=T(\dot q_t)$ in the $L^2(K,\mu)$ norm for operators, since $K$ is compact. \\
The proof for the second derivative works in the same way, using the fact that $(t,z)\mapsto q_t(z)$ is continuous in $z$ and $C^1$ in $t$.\\
As a consequence of the above results, $t\mapsto T(q_t)$ is continuous and twice differentiable, with $\tfrac{\dd}{\dd t}T(q_t) = T(\dot q_t)$ and $\tfrac{\dd^2}{\dd t^2}T(q_t) = T(\ddot q_t)$.\\ The proof for $T(c_t)$ is analogous.
\end{proof}
\begin{manuallemma}{\ref{convcont}}
Denoting $Q_{l|L}$ the covariance kernel of the layer $l$ in a net of depth $L$ and letting $q_t$ being the solution of \eqref{recc}
\begin{align*}
    \lim_{L\to\infty}\sup_{l\in[1:L]}\sup_{z\in K^2}|Q_{l|L}(z)-q_{t(l,L)}(z)| = 0\,,
\end{align*}
where $t(l,L) = (l-1)/(L-1)$.
\end{manuallemma}
\begin{proof}
We'll show that the relation holds for $c_T$, and so for $q_t$. Let $F$, defined on $[0,1]\times K^2$, be such that $\dot c_t(z) = F(z,t,c_t(z))$. Define
\begin{align*}
\tau(h) = \sup_{t,z}\left| \frac{c_{t+h}(z)-c_t(z)}{h}-F(z,t,c_t(z))\right|\,.
\end{align*}
Since $t$ and $z$ takes values on compact sets, by uniform continuity, fixed $h$ we can write, for $h\to 0$
\begin{align*}
\sup_t\sup_{s\in [t,t+h]} \left| F(z,s,c_s(z)) - F(z,t,c_t(z))\right|=o(h)\,.
\end{align*}
Hence, since $\tau$ can be rewritten as $\tau(h) = \frac{1}{h}\sup_{t,z}\left| \int_t^{t+h}(F(z,s,c_s(z)) - F(z,t,c_t(z))\,\dd s\right|$, it's clear that $\tau(h)\to 0$ for $h\to 0$. \\
Now, from the explicit form of $F$, as in \eqref{eqc}, we deduce that there exists a constant $M$ such that for all $z$, all $t$ and all pairs $(\gamma,\gamma')\in[-1,1]^2$
\begin{align*}
|F(z,t,\gamma)-F(z,t,\gamma')|\leq M\|\gamma-\gamma'\|\,.
\end{align*}
Under this uniform Lipschitz condition, the following is a standard result \cite{atkinson2011numerical} \TODO{Write proof explicitly}
\begin{align}
\sup_{l\in[0:L]}\sup_{z\in K^2}|Q_{l|L}(x,x')-q_{t(l,L)}(x,x')| \leq \tau(1/L) \frac{e^M-1}{M}\,.\phantom\qedhere\tag*{$\qed$}
\end{align}
\end{proof}
\subsection*{Proof of Theorem \ref{thm:expr}}
\begin{lemma}\label{propf}
The function $f:[-1,1]\to\mathbb{R}$, defined in \eqref{deff}, is an analytic function on $(-1,1)$, whose expansion $f(\gamma) = \sum_{n\in\mathbb{N}}\al_n\,\gamma^n$
converges absolutely on $[-1,1]$. Moreover, $\al_n>0$ for all even $n\in\mathbb{N}$, $\al_1=-1/2$ and $\al_n=0$ for all odd $n\geq 3$.\\
Let $g:[-1,1]\to\mathbb{R}$ be defined as $g(\gamma) = f(\gamma)f'(\gamma)$. $g$ is analytic on $(-1,1)$ and its expansion $g(\gamma) = \sum_{n\in\mathbb{N}}\be_n\,\gamma^n$ converges absolutely on $[-1,1]$. Moreover, for all odd $n\in\mathbb N$ the coefficient $\be_n$ is strictly positive.
\end{lemma}
\begin{proof}
The claims for $f$ are not hard to prove and well known, see for instance \cite{daniely2016deeper}. As for $g$, the analyticity of $f$ implies the one of $f'$, and it's easy to check the convergence on $[-1,1]$. Moreover, since all the odd Taylor coefficients of $f'$ are striclty positive, as the even coefficients of $f$ are. It follows that $\beta_n>0$ for all odd $n$.
\end{proof}
\begin{lemma}\label{schur}
Let $C$ be a kernel on $K$, such that $|C(z)|\leq 1$ for all $z\in K$. Consider a non-negative real sequence $\{\al_n\}_{n\in\mathbb{N}}$, and assume that
\begin{align*}
g(\gamma) = \sum_{k=0}^\infty \al_k\, \gamma^k
\end{align*} 
converges uniformly on $[-1,1]$. Then $T_\mu(g(C))$ is a non-negative definite compact operator, and in particular $g(C)$ is a kernel.
\end{lemma}
\begin{proof}
Clearly $g(C)$ is continuous and symmetric (as composition of continuous and symmetric functions). Moreover, since the Taylor expansion of $g$ converges uniformly on $[-1,1]$ and $|C(z)|\leq 1$ for all $z\in K$, we have that $T_\mu(g(C)) = \sum_{k\in\mathbb N}\al_k\,T_\mu(C^k)$, the sum converging wrt the operator norm.\\
By a standard extension of the Schur product theorem to integral operators \cite{horn2012matrix}, the product of two kernels is still a kernel.\TODO{Check the reference for a generic $\mu$}\\
As a consequence, it is easy to prove by induction that $T_\mu(C^k)$ is non-negative definite for all $k$. Hence $T_\mu(g(C))$ is the converging limit of a sum of compact non-negative definite operator and we conclude.
\end{proof}
\begin{prop}\label{prop:univ_K}
Let $K\subset\R^d$ be compact. Assume $\sigma_b>0$ and let $\tilde f:\gamma\mapsto \tfrac{\gamma}{2}+f(\gamma)$ be defined on $[-1,1]$. Then the kernel $\tilde f(c_0)$, defined point-wise as $\tilde f(c_0)(x,x') = \tilde f(c_0(x,x'))$, is universal on $K$.
\end{prop}
\begin{proof}
First notice that $c_0(x,x') = \tfrac{1+\zeta\,x\cdot x'}{\sqrt{(1+\zeta\,\|x\|^2)(1+\zeta\,\|x'\|^2)}}$, where $\zeta = \swq/\sbq$. For $n\in\mathbb N$, define the kernel $p_n: (x,x')\mapsto c_0(x,x')^{2n}$, with the convention that $p_0\equiv 1$. From Lemma \ref{propf} we can write 
\begin{align*}
\tilde f(c_0) = \sum_{n\in\mathbb N}\al_n\,p_n\,,
\end{align*}
the sum converging wrt uniformly on $K^2$ with $\al_n>0$ for all $n\in\mathbb{N}$. By Lemma \ref{schur}, $\tilde f(c_0)$ is a kernel.\\
Now, for each $n$, we have
\begin{align*}
p_n(x,x') = \frac{1}{(1+\zeta\,\|x\|^2)^n(1+\zeta\,\|x'\|^2)^n}\sum_{k=0}^{2n}\omega_{k,n}\,(x\cdot x')^k \,,
\end{align*}
where the coefficients $\omega_{k,n}$'s are all strictly positive, explicitly $\omega_{k,n} = \zeta^k\binom{k}{n}$. \\
Expanding the inner product $x\cdot x'$, we can express $p_n$ in the form 
\begin{align*}
p_n(x,x') = \sum_{J\in\mathcal J_n}\be_{J,n}\,A_{J,n}(x)A_{J,n}(x')\,,
\end{align*}
where $\mathcal J_n =\{(j_1\dots j_d)\in\mathbb N^d: \sum_{i=1}^dj_i \in[0:2n]\}$, all the coefficients $\be_{J,n}$'s are strictly positive and the $A_{J_n}$'s are defined as
\begin{align*}
A_{J,n}(x) = \frac{{x_1}^{j_1}\dots {x_d}^{j_d}}{(1+\zeta\,\|x\|^2)^n}\,.
\end{align*}
Hence we can write $\tilde f(c_0)$ as
\begin{align}\label{feature_map}
\tilde f(c_0)(x,x') = \sum_{n\in\mathbb N}\sum_{J\in\mathcal J_n} \al_n\be_{J,n} A_{J,n}(x) A_{J,n}(x')\,.
\end{align}
For any $n,n'\in\mathbb N$, $J\in\mathcal J_n$, $J'\in\mathcal J_{n'}$, it's clear that $A_{J,n}A_{J',n'} = A_{J'',n+n'}$, where $J''$ is some element in $\mathcal J_{n+n'}$. As a consequence, the linear span of the family $\{A_{J,n}\}_{n\in\mathbb N, J\in \mathcal J_n}$ is an algebra $\mathcal A$ (which is actually a subalgebra of $C(K)$ since all the $A_{J,n}$'s are continuous). Moreover $A_{(0\dots0),0}\equiv 1$, so that $\mathcal A$ contains a constant, and it is straightforward to check that $\mathcal A$ separates points. Then, from Stone-Weierstrass theorem, $\mathcal A$ is dense in $C(K)$. \\
Define a bijection $\iota:\mathbb{N}\to\{(n,J):n\in\mathbb N, J\in\mathcal J_n\}$ and let $\Phi_n = \sqrt{\al_n\be_{\iota(n)}}A_{\iota(n)}$. For all $x\in K$, we have that $\Phi(x) = \{\Phi_n(x)\}_{n\in\mathbb{N}}\in \ell^2$, since $p_n(x,x)<\infty$. We conclude that $\Phi$ is a feature map for $\tilde f(c_0)$ and the density of the linear span of $\{\Phi_n\}_{n\in\mathbb N}$ allows to conclude that the kernel is universal on $K$, in the sense of Definition \ref{univ_kern}, \cite{micchelli}.
\end{proof}
\begin{prop}\label{prop:t_phi}
Fix any finite Borel measure $\mu$ and assume that $\sigma_b>0$. Given any non-zero $\f\in L^2(K,\mu)$, there exists a $t_\f \in (0,1]$ such that $\langle T_\mu(q_t)\, \f,\f\rangle > 0$, for all $t\in (0,t_\f)$.
\end{prop}
\begin{proof}
From Corollary \ref{corex}, we can expand $T_\mu(q_t)$ around $t=0$ as 
\begin{align*}
    T_\mu(q_t) = T_\mu(q_0) +  t\, T_\mu(\dot q_0) + o(t) = t \,T_\mu\left(\sbq + \frac{\swq}{2} q_0\right) + T_\mu((c_0 + t f(c_0))R_0) + o(t) \,,
\end{align*}
the $o(t)$ being wrt the operator norm, where we've defined the kernel $R_0$ via $R_0(x,x') =\tfrac{\swq}{2}\sqrt{(1+\zeta \|x\|^2)(1+\zeta \|x'\|^2)}$.\\
Since $T_\mu(q_0)$ is non-negative, for any $\f\in L^2(I)$, we have 
\begin{align*}
    \langle T_\mu(q_t)\,\f,\f\rangle \geq \langle T_\mu((c_0 + t f(c_0))R_0)\,\f,\f\rangle + o(t) =\left(1-\frac{t}{2}\right)\langle T_\mu(c_0)\,\psi,\psi\rangle + t\,\langle T_\mu(f(c_0)))\,\psi,\psi\rangle + o(t)\,,
\end{align*}
where $\psi(x) = \sigma_w\sqrt{(1+\zeta \|x\|^2)/2}\,\f(x)$. We conclude by the strictly positivity of $\tilde f(c_0)$ on $L^2(K,\mu)$, thanks to Proposition \ref{prop:univ_K} and Lemma \ref{univ=Lp_spd}. 
\end{proof}
\begin{prop}\label{prop:dotq>=0}
For any finite Borel measure $\mu$, for any $t\in[0,1]$, the operator $T_\mu(\dot q_t)$ on $L^2(K,\mu)$ is non-negative definite. In particular, for all $\f\in L^2(K,\mu)$ we have
\begin{align*}
    \tfrac{\dd}{\dd t}\langle T_\mu(q_t)\,\f,\f\rangle \geq 0\,.
\end{align*}
\end{prop}
\begin{proof}

Fix $\mu$ and $\f\in L^2(K,\mu)$. From \eqref{recc} we can write
\begin{align*}
T_\mu(\dot q_t) = T_\mu\left(\sbq +\frac{\swq}{2}\,q_t +\frac{\swq}{2}\,\frac{f(c_t)}{c_t} q_t\right)\,.
\end{align*}

By Lemma \ref{ex}, $T_\mu(q_t)$ is non-negative definite, so we can write
\begin{align*}
\langle T_\mu(\dot q_t)\,\f,\f\rangle &= \sbq |\langle 1,\f\rangle|^2 + \frac{\swq}{2}\left\langle T_\mu\left(\frac{c_t + f(c_t)}{c_t} \,q_t\right)\f,\f\right\rangle\\
&\geq\frac{\swq}{2}\left\langle T_\mu\left(\tilde f(c_t)\,\frac{q_t}{c_t}\right)\f,\f\right\rangle\\
&=\frac{\sbq}{2}\langle T_\mu(\tilde f(c_t))\,\psi,\psi\rangle\,,
\end{align*}
where $\tilde f:\gamma\mapsto \tfrac{\gamma}{2}+f(\gamma)$, for $\gamma\in [-1,1]$, and $\psi(x) = \sqrt{q_t(x,x)}\,\f(x)$. By Lemma \ref{propf}, the Taylor expansion of $\tilde f$ around $0$ converges uniformly on $[-1,1]$, and all its coefficients are non-negative. We conclude by Lemma \ref{schur} that $T_\mu(\dot q_t)$ is non-negative definite.\\
Finally, to prove the inequality, it is enough to recall that $\tfrac{\dd}{\dd t}T_\mu(q_t) = T_\mu(\dot q_t)$ by Corollary \ref{corex}, the derivative $\tfrac{\dd}{\dd t}$ being wrt the operator norm on $L^2(K,\mu)$. 
\end{proof}
\begin{manualthm}{\ref{thm:expr}}
Let $K$ a compact set of $\R^d$ and assume $\sigma_b>0$. For any $t\in (0,1]$, the kernel $q_t$ on $K$, solution of \eqref{recc}, is universal.
\end{manualthm}
\begin{proof}
By Lemma \ref{univ=Lp_spd}, it suffices to show that for any finite Borel measure $\mu$ on $K$, $T_\mu(q_t)$ is strictly positive definite for all $t\in(0,1]$. Fix any nonzero $\f\in L^2(K,\mu)$, define the map $F$ on $[0,1]$ by $F(t) = \langle T_\mu(q_t)\,\f,\f\rangle$. For any fixed $t\in(0,1]$, by Proposition \ref{prop:t_phi} we can find $s\in(0,t)$ such that $F(s)>0$. Since $F$ is non decreasing by Proposition \ref{prop:dotq>=0}, we get that $F_t>0$. Hence $T_\mu(q_t)$ is strictly positive definite.
\end{proof}
\subsection*{Proof of Proposition \ref{prop:uniform_Sd}}
\begin{lemma}\label{lem10}
Let $\{A_n\}_{n\in\mathbb{N}}$ be a family of compact non-negative operators on a separable Hilbert space $\mathcal{H}$. Let $R_n$ be the range of $A_n$ and assume that $V = \Span(\bigcup_{n\in\mathbb N} R_n)$ is dense in $\mathcal{H}$. Let $\{\al_n\}_{n\in\mathbb{N}}$ be a strictly positive sequence such that the sum
\begin{align*}
A = \sum_{n\in\mathbb{N}} \alpha_n\,A_n
\end{align*}
converges in the operator norm. Then $A$ is a compact strictly positive definite operator. 
\end{lemma}
\begin{proof}
$A$ is the convergent limit of a sum of compact self-adjoint operators and hence it is compact and self-adjoint. Now, fix an arbitrary nonzero $h\in\mathcal{H}$. To show that $A$ is strictly positive it is enough to prove that $\langle A\,h,h\rangle>0$.\\
Denote by $V_N$ the linear span of $\bigcup_{n\in[0:N]}R_n$. Since $V_{N}\subseteq V_{N+1}$ for all $N$, and  
$\bigcup_{N\in\mathbb{N}}V_N= V$ is dense in $H$, there exists a sequence $\{h_N\}_{N\in\mathbb N}$ converging to $h$ and such that $h_N\in V_N$ for all $N$. \\
Now let's show that there must exist $n^\star\in\mathbb N$ such that $A_{n^\star}\,h\neq 0$. Since $\lim_{N\to\infty}\langle h,h_N\rangle = \langle h,h\rangle>0$, there must be a $N^\star$ such that $\langle h,h_{N^\star}\rangle>0$ and so there exists $n^\star\in[0:N^\star]$ and $h_{n^\star}\in V_{n^\star}$ such that $\langle h, h_{n^\star}\rangle\neq 0$. In particular, $h$ is not orthogonal to $R_{n^\star}$ and can not lie in the nullspace of $A_{n^\star}$, using the fact that $A_{n^\star}$ is compact and self-adjoint and so its range and its nullspace are orthogonal \cite{lang2012real}. \\
Using the spectral decomposition of non-negative compact operators, it's straightforward that $A_{n^\star}\,h \neq 0$ implies that $\langle A_{n^\star}\,h,h\rangle > 0$. Now, since $A_n$ is non-negative and $\al_n>0$ for all $n$, we have
\begin{align*}
    \langle A\,h,h\rangle = \sum_{n\in\mathbb N}\al_n\langle A_n\,h,h\rangle \geq \al_{n^\star} \langle A_{n^\star}\,h,h\rangle > 0\,,
\end{align*}
and so we conclude.
\end{proof}
\begin{lemma}\label{density_sphere}
For all $n\in\mathbb{N}$, consider the kernel $p_n$ on $\Sd$, defined by $p_n(x,x') = (x\cdot x')^n$, and let $T(p_n)$ be the induced integral operator on $L^2(\Sd)$. Denoting as $R_n$ the range of $T(p_n)$, the subspace $V = \Span\left(\bigcup_{n\in\mathbb N}R_n\right)$ is dense in $L^2(\Sd)$.\\
Moreover, letting $V' = \Span\left(\bigcup_{n\in\mathbb N}R_{2n}\right)$ and $V'' = \Span\left(\bigcup_{n\in\mathbb N}R_{2n+1}\right)$, we have $L^2(\Sd) = \overline{V'}\oplus\overline{V''}$, the overline denoting the closure in $L^2(\Sd)$. 
\end{lemma}
\begin{proof}
To prove that $V$, first notice that for each spherical harmonic $Y$, we can find an operator in the form $T(P(x\cdot x'))$ which have $Y$ in its range. Since the range of such an operator is trivially contained in $V$, it follows that $V$ contains all the spherical harmonic, and so it's dense in $L^2(\Sd)$. \\
Now, note that for any even $n$ and odd $n'$ we have 
\begin{align*}
    \int_{\Sd}(x\cdot z)^n(z\cdot x')^{n'}\dd z = 0\,,
\end{align*}
by an elementary symmetry argument, since it is the integral on the sphere of a homogeneous polynomial of odd degree $n+n'$ in the components $z_i$'s of $z$.\\
It follows that $V'$ and $V''$ are orthogonal. Since their union $V$ is dense, we conclude that $L^2(\Sd)=\overline{V'}\oplus\overline{V''}$.
\end{proof}
\begin{corollary}\label{op_dec}
With the notations of Lemma \ref{density_sphere}, assume that a sequence $\{\al_{n\in\mathbb N}\}$ is such that 
$A = \sum_{n\in\mathbb{N}}\al_n\,T(p_n)$
converges wrt the operator norm on $L^2(\Sd)$. Then $A = A'+A''$, where $A':\overline{V'}\to\overline{V'}$ and $A'':\overline{V''}\to\overline{V''}$. Such a decomposition is unique and 
\begin{align*}
    &A' = \sum_{n\in\mathbb{N}}\al_{2n}\,T(p_{2n})\,;
    &A'' = \sum_{n\in\mathbb N}\al_{2n+1}\,T(p_{2n+1})\,,
\end{align*}
both sums converging wrt the operator norm.
\end{corollary}
\begin{proof}
It's clear that $A=A'+A''$, when both $A'$ and $A''$ are defind on the whole $L^2(\Sd)$.\\
Consider any $\f\in L^2(\Sd)$. We have $A'\f \in \overline{V'}$, since $T(p_{2n})\f\in\overline{V'}$ for all $n$. Analogously, we can show that $A''\f\in\overline{V''}$. To conclude that we can consider the restrictions of $A'$ and $A''$ to $\overline{V'}$ and $\overline{V''}$ respectively, it is enough to recall that for compact self adjoint operators the nullspace is the orthogonal of the closure of the range \cite{lang2012real}, so that the nullspace of $A'$ contains $\overline{V''}$ and the nullspace of $A''$ contains $\overline{V'}$. 
\end{proof}
\begin{prop}\label{prop:t_phi_Sd}
Given any non-zero $\f\in L^2(\Sd)$, there exists a $t_\f \in (0,1]$ such that $\langle T(q_t)\, \f,\f\rangle > 0$, for all $t\in (0,t_\f)$. 
\end{prop}
\begin{proof}
The case $\sigma_b>0$ has been already established in Proposition \ref{prop:t_phi}, hence supposed that $\sigma_b=0$.\\
First recall \eqref{eqc}
\begin{align}\label{cdt}
     \dot c_t = \frac{\swq}{2}\,f(c_t)\,.
\end{align}
Deriving once more we have
\begin{align}\label{cddt}
    \ddot c_t = g(c_t)\,,
\end{align}
where $g=ff'$ as in Lemma \ref{propf}.\\
Define the kernels $p_n$'s, and the subspaces $V'$ and $V''$ of $L^2(\Sd)$, as in Lemma \ref{density_sphere}. 
By \eqref{cdt} and \eqref{cddt} we can write
\begin{align*}
    c_t = c_0 + t\,\dot c_0 + \tfrac{t^2}{2}\,\ddot c_0 + o(t^2) = c_0 + t\,f(c_0) + \tfrac{t^2}{2}\,g(c_0) + o(t^2)\,.
\end{align*}
Since $\sigma_b = 0$, we have that $c_0(x,x') = x\cdot x'$, so that $c_0=p_1$. \\
From Lemma \ref{propf}, $T(\dot c_0)=\sum_{n\in\mathbb N}\al_n\,T(p_n)$ and $T(\ddot c_0)=\sum_{n\in\mathbb N}\be_n\,T(p_n)$, both sums converging in the operator norm. Moreover, $\al_n>0$ for all even $n$ and $\al_n=0$ for all odd $n\geq 3$, whilst $\be_n>0$ for all odd $n$.\\
In particular, by Corollary \ref{op_dec} and Lemma \ref{lem10}, we deduce that the restriction of $T(\dot c_0)|_{\overline{V'}}:\overline{V'}\to\overline{V'}$ is well defined and strictly positive, and the same holds true for the restriction $T(\ddot c_0)|_{\overline{V''}}:\overline{V''}\to\overline{V''}$. \\
Now fix a non-zero $\f\in L^2(\Sd)$. By Lemma \ref{density_sphere}, we can write $\f = \f'+\f''$, with $\f'\in \overline{V'}$, $\f''\in \overline{V''}$ uniquely determined. \\
First, suppose that $\f'\neq 0$. Using Corollary \ref{corex} and recalling that $c_0=p_1$, we get
\begin{align*}
    \langle T(c_t)\,\f,\f\rangle = t\langle T(\dot c_0)|_{\overline{V'}}\,\f',\f'\rangle+
    \langle (1+t\,\al_1)T(p_1)\,\f'',\f''\rangle+o(t)>0\,,
\end{align*}
for $t$ small enough.\\
On the other hand, for $\f'=0$, we have $\f=\f''$ and so
\begin{align*}
    \langle T(c_t)\,\f,\f\rangle = \langle (1+t\,\al_1)T(p_1)\,\f'',\f''\rangle + \tfrac{t^2}{2}\,\langle T(\ddot c_0)|_{\overline{V''}}\,\f'',\f''\rangle+o(t^2)>0
\end{align*}
for $t$ small enough.\\
So there is a $t_\f$ such that, for $t\in(0,t_\f)$, $\langle T(c_t)\,\f,\f\rangle>0$. It follows immediately that the same property is true for $T(q_t)$.
\end{proof}
\begin{lemma}\label{lemma:Sd_spd=>univ}
Let $Q$ be a kernel on $\Sd$. Then $Q$ is universal on $\Sd$ if and only if $T(Q)$ is strictly positive definite on $L^2(\Sd)$.
\end{lemma}
\begin{proof}
If $Q$ is universal, $T(Q)$ is strictly positive definite by Lemma \ref{univ=Lp_spd}. On the other hand, if $T(Q)$ is strictly positive definite, by Proposition \ref{proposition:spectral_decomposition_on_sphere} its range contains all the spherical harmonics. Since the RKHS generated by $Q$ lies in the range of $T(Q)$ \cite{paulsen2016RKHS}, it contains the linear span of the spherical harmonics, which is dense in $C(\Sd)$ \cite{kounchev2001multivariate}. Hence $Q$ is universal.
\end{proof}
\begin{manualprop}{\ref{prop:uniform_Sd}}
For any $t\in (0,1]$, the integral operator $q_t$ on $L^2(K)$, solution of \eqref{recc} with $\sigma_b=0$, is universal on $\Sd$ (with $d\geq 2$).
\end{manualprop}
\begin{proof}
Proceeding as in the proof of Theorem \ref{thm:expr}, using Proposition \ref{prop:dotq>=0} and Proposition \ref{prop:t_phi_Sd} we can show that $T(q_t)$ is strictly posititive definite on $L^2(\Sd)$ for all $t\in(0,1]$. We conclude by Lemma \ref{lemma:Sd_spd=>univ} that $q_t$ is universal on $\Sd$.
\end{proof}

\section{Stable Neural Tangent Kernel}
As before, let $K\subset\R^d$ by a compact set.\\
With the uniform scaling, for arbitrary $x,x'\in K$, the continuous version of \eqref{recurrence_NTK} reads
\begin{align}\label{eq:ntk_cont}
&\dot \Theta_t(x,x') = \dot q_t(x,x') + \tfrac{\swq}{2} (1+f'(c_t(x,x')))\,\Theta_t(x,x')\,;\\
&\Theta_0 = q_0\,,
\end{align}
where $f':\gamma\mapsto -\tfrac{1}{\pi}\arccos\gamma$ is the first derivative of $f$, defined in \eqref{deff}.
\begin{lemma}\label{ex_NTK}
For any $x,x'$ in $K$, the solution $t\mapsto\Theta_t$ of \eqref{eq:ntk_cont} is unique and well defined for all $t\in[0,1]$. Moreover, the map $(x,x')\mapsto \Theta_t(x,x')$ is a kernel in the sense of Definition \ref{def:kernel} for all $t\in[0,1]$.
\end{lemma}
\begin{proof}
The existence and the solution is clear, since it is a homogeneous first order Cauchy problem, with continuous coefficients. We can write explicictly the solution as 
\begin{align*}
    \Theta_t = 
\end{align*}
\end{proof}

\section{Stable ResNet with decreasing scaling}\label{app:section:decreasing_scaling}

\begin{prop}[Uniform Convergence of the Kernel]
Consider a Stable ResNet with a decreasing scaling, i.e. the sequence $(\lambda_l)_{l\geq1}$ is such that $\sum_l \lambda_l^2 < \infty$. Then for all $(\sigma_b,\sigma_w)\in \mathbb{R}^+ \times (\mathbb{R}^+)^*$, there exists a kernel $Q_{\infty}$ on $\mathbb{R}^d$ such that for any compact set $C \subset \mathbb{R}^d$ we have 
$$
\sup_{x,x' \in C^2} |Q_{L}(x,x') - Q_{\infty}(x,x')| = \Theta\big(\sum_{k \geq L} \lambda_k^2\big)
$$
\end{prop}

\begin{proof}
Let $x,x' \in \mathbb{R}^d$. The kernel $Q_l$ is given recursively by the formula
\begin{equation*}
Q_l(x,x') =  Q_{l-1}(x,x') + \lambda_l^2 \sigma_b^2 + \frac{\sigma_w^2 \lambda_l^2}{2} \hat{f}(C_{l-1}(x,x')) \sqrt{Q_{l-1}(x,x)}\sqrt{Q_{l-1}(x',x')}
\end{equation*}
where $\hat{f}(t) = 2 \mathbb{E}[\phi'(Z_1)\phi'(t Z_1 + \sqrt{1-t^2} Z_2)] = t + f(t)$ and $Z_1,Z_2$ are iid standard Gaussian variables. In particular, we have
$$
Q_l(x,x) = \lambda_l^2 \sigma_b^2 + (1 + \frac{\sigma_w^2 \lambda_l^2}{2}) Q_{l-1}(x,x)
$$
observe that 
$$Q_l(x,x) + \frac{2\sigma_b^2}{\sigma_w^2} = (1 + \frac{\sigma_w^2 \lambda_l^2}{2}) (Q_{l-1}(x,x) + \frac{2\sigma_b^2}{\sigma_w^2})$$
therefore, we can assume without loss of generality that $\sigma_b=0$. This yields
$$
C_l(x,x') = \frac{1}{1+\frac{\lambda_l \sigma_w^2}{2}}C_{l-1}(x,x') + \frac{\frac{\sigma_w^2 \lambda_l^2}{2}}{1+\frac{\lambda_l \sigma_w^2}{2}} f(C^{l-1}(x,x'))
$$

Letting $\alpha_l = \frac{\sigma_w^2 \lambda_l^2}{2}$ and $C_l:=C_l(x,x')$, we have that 
$$
C_l = \frac{1}{1+\alpha_l} C_{l-1} + \frac{\alpha_l}{1+\alpha_l} f(C_{l-1})
$$
Since $f$ is non decreasing, $C^l$ is non-decreasing and has a limit $C_{\infty}(x,x')\leq1$. \\

Now let us prove that the convergence of $C_l$ to $C_{\infty}$ happens uniformly with a rate $\sum_{k\geq l} \lambda_l^2$. Using the recursive formula of $C_l$, and knowing that we have that 
$$
C_{\infty} - C_l = \frac{1}{1 + \alpha_l} (C_{\infty} - C_{l-1}) + \frac{\alpha_l}{1 + \alpha_l} (C_{\infty} - f(C_{l-1}))
$$
letting $\delta_l = C_{\infty} - C_l$, it is easy to see that, uniformly in $x,x' \in \mathbb{R}^d$, we have that
$$
\delta_l = \delta_{l-1} + \alpha_l + o(\alpha_l)
$$
Therefore, using the fact that $C_l \leq C_\infty$, we have
$$
\sup_{(x,x') \in \mathbb{R}^d} |C_l(x,x') - C_{\infty}(x,x')| \sim \sum_{k\geq l} \alpha_k
$$
Moreover, we know that 
$$
Q_l(x,x) = Q_1(x,x) \prod_{k=1}^l (1 + \alpha_k)
$$
So that for any compact set $C \subset \mathbb{R}^d$ 
$$
\sup_{x \in C} |Q_{l}(x,x)-Q_{\infty}(x,x)| \sim \sum_{k\geq l} \alpha_k
$$
Moreover, since $C_\infty(x,x') \geq C_l(x,x')$ and $Q_\infty(x,x) \geq Q_l(x,x)$ for all $x \in \mathbb{R}^d$, we conclude using the fact that 
\begin{equation*}
\begin{split}
    Q_{\infty}(x,x') - Q_{l}(x,x') &= \sqrt{Q_\infty(x,x)Q_\infty(x',x')}( C_\infty(x,x')-C_l(x,x'))\\
    &+ C_l(x,x') (\sqrt{Q_\infty(x,x)Q_\infty(x',x')} - \sqrt{Q_l(x,x)Q_l(x',x')})
\end{split}
\end{equation*}
and that 
\begin{equation*}
Q_{\infty}(x,x') - Q_{l}(x,x') \geq \sqrt{Q_l(x,x)Q_l(x',x')}( C_\infty(x,x')-C_l(x,x'))
\end{equation*}

\end{proof}

\section{Stable NTK}\label{app:section:NTK}

\section{A PAC-Bayes Generalization result}\label{app:section_PAC-Bayes}
In this section, we study the PAC-Bayes upper bound of a GP with kernel $Q_L$. We consider a dataset $S$ with $N$ iid training examples $\{(x_i,y_i) \in X\times Y, i\in[1:N]\}$, and a hypothesis space $\mathcal{H}$ from which we want to learn an optimal hypothesis according to some bounded loss function $\ell : Y\times Y \rightarrow [0,1]$. The empirical loss of a hypothesis $h \in \mathcal{H}$ is given by
$$
r_S(h) = \frac{1}{N} \sum_{i=1}^N \ell(h(x_i), y_i)
$$
Assuming that the samples are distributed as $(x,y) \sim \nu$ where $\nu$ is a probability distribution on $X \times Y$, we define the generalization (true) loss by 
$$
r(h) = \mathbb{E}_{\nu}[\ell(f(x),y)]
$$
For some randomized learning algorithm $\mathcal{A}$, the empirical and generalization loss are given by
$$
r_S(\mathcal{A}) = \mathcal{E}_{h \sim \mathcal{A}}[r_s(h)], \quad r(\mathcal{A}) = \mathcal{E}_{h \sim \mathcal{A}}[r(h)]
$$
The PAC-Bayes theorem gives a probabilistic upper bound on the generalization loss $r(\mathcal{A})$ of a randomized learning algorithm $\mathcal{A}$ in terms of the empirical loss $r_S(\mathcal{A})$. Fix a prior distribution $\mathcal{P}$ on the hypothesis set $\mathcal{H}$. The Kullback-Leibler divergence between $\mathcal{A}$ and $\mathcal{P}$ is defined as $KL(\mathcal{A} \| \mathcal{P}) = \int \log \frac{\mathcal{A}(h)}{P(h)} \mathcal{A}(h) dh \in [0, \infty]$. The Bernoulli KL-divergence is given by $kl(a||p) = a \log \frac{a}{p} + (1-a) \log \frac{1-a}{1-p} $  for $a,p \in [0,1]$. We define the inverse Bernoulli KL-divergence $kl^{-1}$ by 
$$
kl^{-1}(a,\epsilon) = \sup\{ p \in [0,1] : kl(a,p) \leq \epsilon \}
$$
\begin{thm}[PAC-Bayesian theorem \CIT]
For any loss function $\ell$ that is $[0,1]$ valued, for any distribution $\nu$, for any $N \in \mathbb{N}$, for any prior $P$, and any $\delta \in (0,1]$, with probability at least $1-\delta$ over the sample $S$, we have 
$$
\forall \mathcal{A}, \hspace{0.25cm} r(\mathcal{A}) \leq  kl^{-1}\left(r_S(\mathcal{A}), \frac{KL(\mathcal{A}\|P) + \log \frac{2 \sqrt{N}}{\delta}}{N}\right)
$$
\end{thm}
The PAC-Bayesian theorem gives can also be stated as
$$ kl(r_S(\mathcal{A}), r(\mathcal{A})) \leq  \frac{KL(\mathcal{A}\|P) + \log \frac{2 \sqrt{N}}{\delta}}{N}
$$
The KL-divergence term $KL(\mathcal{A}\|P)$ plays a major role as it controls the generalization gap, i.e. the difference (in terms of Bernoulli KL-divergence) between the empirical loss and the generalization loss. In our setting, we consider an ordinary GP regression with prior $P(f) = \mathcal{GP}(f | 0, Q(x,x'))$. Under the standard assumption that the outputs $y_N = (y_i)_{i \in [1:N]}$ are noisy versions of $f_N = (f(x_i))_{ i \in [1:N]}$ with $y_N | f_N \sim \mathcal{N}(y_N | f_N, \sigma^2 I)$, the Bayesian posterior $\mathcal{A}$ is also a GP and is given by 
\begin{equation}
\begin{split}
\mathcal{A}(f) = \mathcal{GP}(f | Q_N(x)(Q_{NN} + \sigma^2 I)^{-1} y_N, Q(x,x') - Q_N(x)(Q_{NN} + \sigma^2 I)^{-1})Q_N(x')^T)
\end{split}
\end{equation}
where $Q_N(x) = (Q(x,x_i))_{i \in [1:N]}$ and $Q_{NN} = (Q(x_i, x_j))_{1\leq i,j \leq N}$. In this setting, we have the following result
\begin{prop}[Stability of PAC-Bayes bound]
Let $Q_L$ be the kernel of a ResNet. Let $P_L$ be a GP with kernel $Q_L$ and $\mathcal{A}_L$ be the corresponding Bayesian posterior for some fixed noise level $\sigma>0$. Then, in a fixed setting (fixed sample size N), the following results hold
\begin{enumerate}
    \item With a standard ResNet, we have 
    $$
    KL(\mathcal{A}_L \| P_L) \gtrsim L
    $$
    \item With a Stable ResNet, we have 
    $$
    KL(\mathcal{A}_L \| P_L) = \mathcal{O}_L(1)
    $$
\end{enumerate}
\end{prop}
\begin{proof}
The proof relies on the simple observation that $P_L(f | f_N) = \mathcal{A}_L(f | f_N)$. This yields
\begin{equation}
\begin{split}
KL(\mathcal{A}_L \| P_L) &= KL(\mathcal{A}_L(f_N) \mathcal{A}_L(f | f_N) \| P_L(f_N) P_L(f | f_N)) \\
&= KL(\mathcal{A}_L(f_N) \| P_L(f_N)) \\
&= \frac{1}{2} \log(\det (Q_{L,NN} + \sigma^2 I)) - \frac{N}{2} \log(\sigma^2) - \frac{1}{2} \Tr (Q_{L,NN}(Q_{L,NN} + \sigma^2I)^{-1})\\
&+ \frac{1}{2} y_N^T (Q_{L,NN} + \sigma^2 I)^{-1} Q_{L,NN} (Q_{L,NN} + \sigma^2 I)^{-1} y_N
\end{split}
\end{equation}
where $Q_{L,NN} = (Q_{L}(x_i,x_j))_{1\leq i, j \leq N}$.\\

Since $Q_{L,NN}$ is symmetric and strictly positive definite, it is straightforward that the largest eigenvalue of $Q_{L,NN}(Q_{L,NN} + \sigma^2I)^{-1})$ is smaller than $1$. This yields   
$$\Tr (Q_{L,NN}(Q_{L,NN} + \sigma^2I)^{-1}) \leq N$$ and $$y_N^T (Q_{L,NN} + \sigma^2 I)^{-1} Q_{L,NN} (Q_{L,NN} + \sigma^2 I)^{-1} y_N \leq \sigma^{-2} \|y_N\|_2$$
Both quantities are bounded independently from $L$ and the scaling factors $(\lambda_{k,L})_{k\in[2:L]}$. \\
Now let us analyse the first term $\frac{1}{2} \log(\det (Q_{L,NN} + \sigma^2 I))$. Let $\mu_{L,0} \geq \mu_{L,1} \geq \dots \geq \mu_{L,N}$ be the eigenvalues of $Q_{L,NN}$. Now, we study the behaviour of the first term depending on the cases
\begin{enumerate}
    \item Assume we have a standard ResNet architecture. On the unit sphere $\mathbb{S}^{d-1}$, we have that $Q_L(x,x') \geq q_L C_L(x,x')$, where $q_L = (1 + \frac{\sigma_2}{2})^{L} \delta$ with $\delta = (\sigma_b^2 + \frac{\sigma_w^2}{d}) / (1 + \frac{\sigma_w^2}{2})$. Using \CIT, we know that $\lim_{L \rightarrow \infty} \hat{\mu}_{L,0} =  \hat{\mu}_{\infty,0} \in (0,\infty)$ and for all $k\geq 1$, $\lim_{L \rightarrow \infty} \hat{\mu}_{L,k} = 0$. This yields
    \begin{equation*}
    \begin{split}
        \log(\det (Q_{L,NN} + \sigma^2 I)) &\geq \sum_{k=1}^N \log(q_L \hat{\mu}_{L,k} + \sigma^2)\\
        &\geq \log(q_L \hat{\mu}_{L,0} + \sigma^2) + (N-1)\log(\sigma^2)\\
        &\gtrsim L\log(1+\frac{\sigma_w^2}{2})
    \end{split}
    \end{equation*}
    where the last inequality holds for sufficiently large $L$.
    \item In the case of Stable ResNet, we know that as $L \rightarrow \infty$, the kernel $Q_L$ converges to a strictly positive definite kernel $Q_\infty$, therefore the first term $\log(\det (Q_{L,NN} + \sigma^2 I))$ remains bounded as $L \rightarrow \infty$, which concludes the proof.
\end{enumerate}
\end{proof}

\section{Experimental details and additional results}\label{app:section:experimental details}
\subsection{NNGP/NTK results}
For our NNGP/NTK results, we preprocess all training, validation and test data using the same affine transform such that the training data has zero mean and channel-wise unit variance. We use the same WRN kernel architecture as \cite{lee_wide_nn_ntk} but omit the final average pooling step, which is known improve performance but dramatically increase computational costs \cite{novak2018bayesian, lee2020finite, anonymous2021exploring}. As ResNets always live on the EOC \TODO{sort this out}, we keep the initialisation weight variance at 2 and the initialisation bias variance at 0.01, but we tune the noise variance $\sigma^2$, which is akin to the regularisation parameter in kernel ridge regression. To do so, we compute validation accuracy on a validation set of size 5000, selecting the best $\sigma^2=\lambda \times \text{Trace}(K_{NN})/N$ from a logarithmic scale of $\lambda=[0.001, 0.01, 0.1]$, where $N$ is the training set size and $K_{NN}$ is the $N\times N$ training set Gram matrix for NNGP or NTK kernel $K$. 
\subsection{Trained ResNet results}
For all our trained ResNet experiments we use a similar setup to \cite{Wang2020Picking}
All ResNets are initialised with Kaiming initialisation \cite{he2} and we adopt ResNets architectures where we double the number of filters in each convolutional layer. For CIFAR-10/100 we use batch size 64 across all depths. For TinyImageNet we used batch size 128 for depths 32 \& 50, and batch size 100 for depth 104 in order to allow the model to fit onto a single 11GB VRAM GPU.

For ResNets trained without BatchNorm, for a fair comparison we tuned the initial learning rate on a logarithmic scale.

\newpage
\bibliographystyle{plain}
\bibliography{sample}